\documentclass{article}

\usepackage[preprint]{neurips_2026}

\usepackage{comment}
\usepackage[utf8]{inputenc} 
\usepackage[T1]{fontenc}    
\usepackage{hyperref}       
\usepackage{url}            
\usepackage{booktabs}       
\usepackage{amsfonts}       
\usepackage{nicefrac}       
\usepackage{microtype}      
\usepackage{xcolor}         
\usepackage{multirow}

\hypersetup{
    colorlinks=true,       
    linkcolor=red!70!black,       
    citecolor=magenta,       
    filecolor=magenta,     
    urlcolor=cyan         
}

\usepackage{xcolor}
\usepackage{tikz}
\usetikzlibrary{fit}
\usepackage{subcaption}

\usepackage{xfrac}
\usetikzlibrary{shapes.geometric, arrows.meta, positioning, automata, decorations.pathreplacing, calligraphy}

\usepackage{longtable} 

\newif\ifCOMMENTS
\COMMENTStrue

\newcommand{\PP}{\mathbb{P}}

\newcommand{\probm}{\PP}

\newcommand{\nat}{\mathbb{N}}

\newcommand{\dist}{\Delta}
\newcommand{\dirac}[1]{\operatorname{Dirac}({#1})}
\newcommand{\conv}{\operatorname{conv}}

\newcommand{\rmdp}{\mathcal{M}}
\newcommand{\sg}{\mathcal{G}}

\newcommand{\rp}{\mathfrak{R}}
\newcommand{\rsg}{\rp\!\left(\sg\right)}

\newcommand{\pone}{\mathsf{max}}
\newcommand{\ptwo}{\mathsf{min}}
\newcommand{\pa}{\mathsf{a}}
\newcommand{\pe}{\mathsf{e}}

\newcommand{\sone}{S_{\pone}}
\newcommand{\stwo}{S_{\ptwo}}

\newcommand{\aone}{A_{\pone}}
\newcommand{\atwo}{A_{\ptwo}}

\newcommand{\tran}{\delta}

\newcommand{\Plays}{\mathsf{Plays}}
\newcommand{\Playssg}{\mathsf{Plays}_{\sg}}
\newcommand{\Playsr}{\mathsf{Plays}_{\rmdp}}

\newcommand{\Last}{\operatorname{last}}

\newcommand{\Cone}{\operatorname{Cyl}}
\newcommand{\cyl}[1]{\Cone\!\left(#1\right)}

\newcommand{\parity}[2]{\operatorname{Parity}_{#1}\!\left(#2\right)}

\newcommand{\Inf}{\operatorname{Inf}}

\newcommand{\val}{\operatorname{Val}}
\newcommand{\Val}{\val} 

\newcommand{\prfn}{p}

\newcommand{\obs}{\mathcal{O}}

\newcommand{\obfna}{\mathcal{Z}_\pa}
\newcommand{\obfne}{\mathcal{Z}_\pe}
\newcommand{\obsa}{\mathcal{O}_\pa}
\newcommand{\obse}{\mathcal{O}_\pe}

\newcommand{\obsone}{\mathcal{O}_\pone}
\newcommand{\obstwo}{\mathcal{O}_{\ptwo}}
\newcommand{\obsstar}{\mathcal{O}_{\star}}
\newcommand{\obfnone}{\mathcal{Z}_\pone}
\newcommand{\obfntwo}{\mathcal{Z}_{\ptwo}}
\newcommand{\obfnstar}{\mathcal{Z}_\star}

\newcommand{\obsnewone}{o^{\dagger}}
\newcommand{\obsnewtwo}{o^{\#}}
\newcommand{\Obsnewone}{\mathcal{O}^{\dagger}}
\newcommand{\Obsnewtwo}{\mathcal{O}^{\#}}
\newcommand{\OH}{\mathsf{OH}}

\newcommand{\stratpone}{\Sigma}
\newcommand{\stratptwo}{\Pi}

\newcommand{\uncert}{\mathcal{P}}
\newcommand{\uncertsa}{P}

\newcommand{\GammaOpone}{\Gamma_{\obs}^{\pone}}
\newcommand{\GammaOptwo}{\Gamma_{\obs}^{\ptwo}}
\newcommand{\GammaOstar}{\Gamma_{\obs}^{\star}}

\newcommand{\LambdaOpone}{\Lambda_{\obs}^{\pone}}
\newcommand{\LambdaOptwo}{\Lambda_{\obs}^{\ptwo}}

\newcommand{\mdpsga}{\mathfrak{C}}
\newcommand{\mdpsge}{\mathfrak{D}}

\newcommand{\abot}{{\top}}
\newcommand{\anew}{\abot}

\newcommand{\nc}[1]{\begingroup\small\textcolor{red!70!black}{#1}\endgroup}
\newcommand{\nb}[1]{\begingroup\small\textcolor{blue!70!black}{#1}\endgroup}

\definecolor{darkgreen}{rgb}{0,0.6,0}
\definecolor{lightblue}{rgb}{0.5,0.6,1.0}
\definecolor{lightgray}{rgb}{0.98,0.98,0.98}
\definecolor{mauve}{rgb}{0.58,0,0.82}
\definecolor{sienna}{rgb}{0.6,0.18,0.09}
\colorlet{darkblue}{blue!60!black}
\colorlet{darkred}{red!50!black}
\colorlet{safecellcolor}{yellow!5}
\colorlet{goodcellcolor}{green!10}
\colorlet{badcellcolor}{blue!10}

\definecolor{maxcolor}{RGB}{200,100,200}
\definecolor{mincolor}{RGB}{20,200,100}

\usepackage{hyperref}
\usepackage{url}

\title{Solving Robust POMDPs with $\omega$-regular \\
Objectives via Partially Observable \\ Stochastic Games}

\author{Durgam Latha, Dion Reji, S.~Akshay \\
Indian Institute of Technology Bombay \\
Mumbai, India \\
\And
\DJ or\dj e \v{Z}ikeli\'c \\
Nanyang Technological University \\
Singapore, Singapore \\
\AND
Shankaranarayanan Krishna \\
Indian Institute of Technology Bombay \\
Mumbai, India \\
}

\usepackage{comment}
\usepackage[utf8]{inputenc} 
\usepackage[T1]{fontenc}    
\usepackage{amsfonts}       
\usepackage{nicefrac}       
\usepackage{microtype}      
\usepackage{xcolor}         
\usepackage{multirow}

\usepackage{tikz}
\usetikzlibrary{fit}
\usepackage{subcaption}

\usepackage{xfrac}
\usetikzlibrary{shapes.geometric, arrows.meta, positioning, automata, decorations.pathreplacing, calligraphy}
\usepackage{longtable}

\usepackage{amsmath, amssymb, amsthm, mathtools, bm}
\usepackage{mathrsfs}     
\usepackage{stmaryrd}     
\usepackage{dsfont}      
\usepackage{cancel}
\usepackage[nameinlink, capitalise]{cleveref}
\usepackage{changepage}

\usepackage{paralist}
\usepackage{lineno}

\theoremstyle{definition}
\newtheorem{definition}{Definition}[section]
\theoremstyle{plain}
\newtheorem{theorem}[definition]{Theorem}
\newtheorem{lemma}[definition]{Lemma}
\newtheorem{corollary}[definition]{Corollary}
\newtheorem{remark}[definition]{Remark}

\newtheorem{example}{Example}
\newtheorem{observation}[definition]{Observation}

\begin{document}

\maketitle

\begin{abstract}
Robust POMDPs (RPOMDPs) generalize classical POMDPs to the setting where exact transition probabilities are not known -- rather, they are only known to belong to some uncertainty set of values.
In this work, we study the problem of solving RPOMDPs with general $\omega$-regular objectives, which subsume a broad class of objectives such as reachability, safety, and linear temporal logic (LTL) objectives. We show that, for $(s,a)$-rectangular RPOMDPs with polytopic uncertainty sets, the problem of solving RPOMDPs under $\omega$-regular objectives can be reduced to solving partially observable stochastic games (POSGs) under $\omega$-regular objectives. Moreover, we show for the first time that reductions can be constructed in both directions, establishing the semantic equivalence between $(s,a)$-rectangular RPOMDPs with polytopic uncertainty sets and POSGs. This allows us to derive a range of new computational complexity results, including both upper and lower complexity bounds, on solving RPOMDPs with different $\omega$-regular objectives. As a corollary, we also derive new computational complexity results for RMDPs. 
\end{abstract}

\section{Introduction}

Partially observable Markov decision processes (POMDPs) are a standard model for sequential decision making under uncertainty and imperfect state observability~\citep{KaelblingLC98}. The problem of solving a POMDP is concerned with computing a strategy (or policy) that maximizes the expected payoff or the probability of satisfying some objective. However, classical algorithms assume that transition probabilities are exactly known, which is not always a realistic assumption. In practice, transition probabilities are usually inferred from data and as such their estimates come with a level of uncertainty~\citep{DBLP:conf/ijcai/2024}. 
In the fully observable setting, this challenge is addressed through robust MDPs (RMDPs)~\citep{NilimG03,iyengar2005robust}, which generalize classical MDPs by assuming that transition probabilities belong to some {\em uncertainty sets} of possible values, rather than being exactly given. Robust POMDPs (RPOMDPs)~\citep{Osogami15} provide a robust generalization of POMDPs.

Recent years have seen significant advances in algorithmic approaches to solving RMDPs and RPOMDPs. Existing works largely focus on reward objectives, such as maximizing the expected discounted-sum~\citep{NilimG03,iyengar2005robust,WiesemannKR13,KaufmanS13,TamarMX14,GhavamzadehPC16,HoPW21} or the long-run average reward~\citep{TewariB07,Grand-ClementP23,GoyalG23,WangVAPZ23,DBLP:conf/ijcai/2024,MeggendorferWW25,AsadiCGKMP26arxiv} in the R(PO)MDP, which optimize for performance-related properties in the model. However, reward objectives do not consider {\em logical correctness} properties that are particularly important in safety-critical applications. For instance, in self-driving cars, the reach-avoid correctness property is concerned with reaching some target region while avoiding the unsafe obstacles~\citep{SummersL10}. The stability correctness property requires the agent to eventually reach and stay within some target region ad infinitum~\citep{khalil2002nonlinear}. These are both classical examples of logical correctness properties in control and robotics applications, which are not defined in terms of reward objectives. In formal methods, logical correctness properties are formally defined via logics such as linear temporal logic (LTL)~\citep{Pnueli77}, or more generally $\omega$-regular specifications~\citep{ClarkeHVB18}, which subsumes most logical correctness properties appearing in control and robotics applications~\citep{Kress-GazitFP09}. While recent years have seen significant advances in solving RMDPs and RPOMDPs with reward objectives, the literature on logical correctness objectives remains limited.  
Even the literature on RMDPs with $\omega$-regular objectives either makes significant assumptions on uncertainty sets such as interval MDPs~\citep{ChatterjeeSH08,SenVA06}, or is restricted to probability~$1$ or limit-sure satisfaction of $\omega$-regular specifications~\citep{AsadiCGKS26}.

\noindent{\bf Our Contributions.} In this work, we consider the problem of solving RMDPs and RPOMDPs with $\omega$-regular objectives, with a focus on computational complexity aspects. As common in the literature on solving R(PO)MDPs, we consider RPOMDPs with $(s,a)$-rectangular and polytopic uncertainty sets. These assumptions are quite general and subsume types of uncertainty sets considered in the existing literature, such as intervals~\citep{GivanLD00} or $L^1$-uncertainty sets~\citep{HoPW21}. 
Our approach is based on the following fundamental novel result: we formally prove that the problem of solving RPOMDPs (defined in section \ref{sec:prelim}) with $\omega$-regular objectives is {\em equivalent} to the problem of solving partially observable stochastic games (POSGs) (defined in section \ref{sec:posg}) with $\omega$-regular objectives. That is, we show that the two problems can be reduced to each other in polynomial-time and thus any known algorithm and computational complexity result for solving one problem readily applies to the other problem. The key novelty of our result lies in establishing the equivalence of the two models. Previous work has shown that, under reward objectives, the problem of solving R(PO)MDPs can be reduced to the problem of solving (PO)SGs~\citep{DBLP:conf/ijcai/2024,Grand-ClementPV23,BovySJ024}. Here, we consider $\omega$-regular objectives rather than reward objectives. Moreover, we also present a reduction in the opposite direction from POSGs to RPOMDPs, which is technically more challenging. 
Since solving POSGs under $\omega$-regular objectives has received more attention~\citep{DBLP:journals/fmsd/Chatterjee0H13} compared to RPOMDPs, our equivalence result allows us to derive a plethora of new complexity results on solving R(PO)MDPs. 
Our contributions can be summarized as follows:
\begin{compactenum} 
    \item {\bf Equivalence of RPOMDPs and POSGs.} We formally prove, for the first time, that the problems of solving RPOMDPs and POSGs under $\omega$-regular objectives are {\em equivalent}, i.e.~polynomial-time reducible to each other. Our equivalence result holds for Sure Winning, Almost Sure Winning, Limit Sure Winning, and Quantitative Winning $\omega$-regular objectives.
    \item {\bf New complexity results for R(PO)MDPs.} We derive new computational complexity results for RPOMDPs under $\omega$-regular objectives, see Table~\ref{tab:combined_complexity}. As a corollary, we also derive new complexity results for RMDPs and one-sided partial observability setting (1s RPOMDP).
\end{compactenum}
Thus, we significantly advance the state of the art in solving R(PO)MDPs under $\omega$-regular objectives and provides a complete picture of the complexity landscape for different variants of the problem.

\begin{table*}[t]
\centering
\begin{tabular}{|l|c|c|c|c|c|c|}
\hline
\multirow{2}{*}{} & \multicolumn{3}{c|}{\textbf{Sure Winning} (Table 1)} & \multicolumn{3}{c|}{\textbf{Almost Sure Winning} (Table 3)} \\
\cline{2-7}
 & \textbf{RMDP} & \textbf{1s RPOMDP} & \textbf{RPOMDP} & \textbf{RMDP} & \textbf{1s RPOMDP} & \textbf{RPOMDP} \\
\hline
{\small Safe}       & \nc{Linear-time}    & \nc{EXPTIME-c} & \nc{EXPTIME-c} & \nc{Linear-time}    & \nc{EXPTIME-c} & \nc{EXPTIME-c} \\
{\small Reach} & \nc{Linear-time}    & \nc{EXPTIME-c} & \nc{EXPTIME-c} &\nc{\small Linear-time}    & \nc{EXPTIME-c} & \nb{Open}\\
{\small Büchi}        & \nc{Quad-time} & \nc{EXPTIME-c} & \nc{EXPTIME-c} & \nc{\small Quad-time} & \nc{EXPTIME-c} & \nb{Open} \\
{\small co-Büchi}      & \nc{Quad-time} & \nc{EXPTIME-c} & \nc{EXPTIME-c} & \nc{\small Quad-time} & \nc{Undecidable}   & \nc{Undecidable}   \\
{\small $\omega$-regular}       & \nc{NP $\cap$ coNP} & \nc{EXPTIME-c} & \nc{EXPTIME-c} & \nc{NP $\cap$ coNP} & \nc{Undecidable}   & \nc{Undecidable}   \\
\hline
\hline
& \multicolumn{3}{c|}{\textbf{Limit Sure Winning} (Table 4)} & \multicolumn{3}{c|}{\textbf{Quantitative Winning} (Table 5)} \\
\cline{2-7}
& \textbf{RMDP} & \textbf{1s RPOMDP} & \textbf{RPOMDP} & \textbf{RMDP} & \textbf{1s RPOMDP} & \textbf{RPOMDP} \\
\hline
Safe      & \nc{Linear-time}    & \nc{EXPTIME-c} & \nc{EXPTIME-c} & \nc{NP $\cap$ coNP} & \nc{Undecidable} & \nc{Undecidable} \\
Reach & \nc{Linear-time}    & \nc{Undecidable} & \nc{Undecidable} & \nc{NP $\cap$ coNP} & \nc{Undecidable} & \nc{Undecidable} \\
Büchi        & \nc{Quad-time} & \nc{Undecidable} & \nc{Undecidable} & \nc{NP $\cap$ coNP} & \nc{Undecidable} & \nc{Undecidable} \\
co-Büchi      & \nc{Quad-time} & \nc{Undecidable} & \nc{Undecidable} & \nc{NP $\cap$ coNP} & \nc{Undecidable} & \nc{Undecidable} \\
$\omega$-regular       & \nc{NP $\cap$ coNP} & \nc{Undecidable} & \nc{Undecidable} & \nc{NP $\cap$ coNP} & \nc{Undecidable} & \nc{Undecidable} \\
\hline
\end{tabular}
\smallskip
\caption{\small Summary of complexity results that follow from our equivalence of RPOMDPs and POSGs under $\omega$-regular objectives. We show results both for general $\omega$-regular objectives as well as for special instances of objectives (reachability, safety, B\"uchi, co-B\"uchi) for which more efficient algorithms are known. The '-c' suffix denotes completeness, otherwise shown results denote upper complexity bounds. The \nc{red colored} entries are our novel results that were previously unknown and that follow from our equivalence result with POSGs~\citep{DBLP:journals/fmsd/Chatterjee0H13}. The table considers RMDPs (full observability for both the agent and the environment), 1s RPOMDPs (partial observability for the agent only), and two-sided RPOMDPs (partial observability for both). The \nb{blue colored} entries remain open problems since the decidability and complexity of solving POSGs under almost sure winning reachability and B\"uchi objectives is a well-known open problem~\citep{DBLP:journals/fmsd/Chatterjee0H13}. The table numbers correspond to the tables in \citep{DBLP:journals/fmsd/Chatterjee0H13} from which the respective complexity results on POSGs were obtained. Turn-based quantitative stochastic parity games complexity results, from which we infer complexity results for quantitative analysis of RMDPs, were also derived in~\citep{ChatterjeeJH04}.}
\label{tab:combined_complexity}
\end{table*}

\noindent {\bf Related Work.} Recent years have seen significant advances in algorithmic approaches to solving RMDPs under reward objectives, with efficient algorithms being proposed for expected discounted-sum reward  
and long-run average reward,  
as discussed above. These works assume knowledge of the RMDP model and provide guarantees on the correctness of their results. In addition, reinforcement learning algorithms for solving RMDPs in a model-free setting but without guarantees on the correctness of results have also been proposed~\citep{RoyXP17,TesslerEM19,WangZ21,WangZ22c,YuDXGH23,WangVAPZ24}. Algorithmic approaches for solving RPOMDPs under reward objectives were considered in~\citep{Osogami15,ChamieM18,Suilen0CT20,NakaoJS21,Cubuktepe0JMST21,GaleslootA0J025,GaleslootSS0ST025,krale2025on}. See~\citep{SuilenBB0025} for a recent survey. We study R(PO)MDPs under $\omega$-regular objectives, which have received much less attention. To the best of our knowledge, no existing work considers solving RPOMDPs with $\omega$-regular objectives. Literature on RMDPs with $\omega$-regular objectives either makes significant assumptions on the structure of uncertainty sets such as interval MDPs~\citep{ChatterjeeSH08,SenVA06}, or is restricted to probability~$1$ or limit-sure satisfaction of $\omega$-regular specifications~\citep{AsadiCGKS26} where the two objectives are shown to have the same value, but without exact complexity bounds that we establish. 

The existence of a reduction from R(PO)MDPs to (PO)SGs under reward objectives has been observed in ~\citep{NilimG03,TesslerEM19,GoyalG23,Grand-ClementPV23}. The work~\citep{DBLP:conf/ijcai/2024} formalizes the reduction from RMDPs to turn-based stochastic games under both expected discounted-sum and long-run average objectives. The work~\citep{BovySJ024} formalizes the reduction from RPOMDPs to POSGs under expected cumulative reward and discounted-sum objectives. However, none of these consider $\omega$-regular~objectives. 

\section{Model Definition and Preliminaries}
\label{sec:prelim}

For a set \( X \), we denote by \( \dist(X) \) the set of all probability distributions over \( X \) and by $2^X$ the set of all subsets of $X$. 
For $x\in X$, $\dirac{x}\in\dist(X)$ is the Dirac distribution assigning probability~$1$ to $x$.
For a sequence $b = b_0,b_1,b_2,\cdots,b_n$, we write $\Last(b) = b_n$.
\noindent{\bf RPOMDPs.}

 A {\em Robust POMDP (RPOMDP)} is a tuple \( (S, A, \uncert, \obsa,\obse,\obfna,\obfne,  s_0) \), where $S$ and $A$ are finite sets of states and actions, $s_0 \in S$ is the initial state, and \( \uncert \subseteq (S \times A \rightarrow \dist(S)) \), called the \emph{uncertainty set}, is a subset 
of all functions from $S \times A$ to $\Delta(S)$.
 \(\obsa,\obse\) respectively are finite sets of \emph{observations} for the \emph{agent} and \emph{environment}, \(\obfna:S\rightarrow \obsa, \obfne:S\rightarrow\obse\) are \emph{observation functions} mapping states to observations.  A \emph{robust MDP (RMDP)} arises when both the agent and the environment
have full observability: for $\star \in \{a,e\}$, $\obs_\star = \{o^\star_s \mid s \in S\}$
with $\obfna(s) = o^a_s$ and $\obfne(s) = o^e_s$. A \emph{one-sided RPOMDP} has partial observability for the agent but full for the environment, i.e.~$\obse = \{o_s \mid s \in S\}$ with $\obfne(s) = o_s$. 

\noindent{\bf Assumptions: $(s,a)$-rectangularity and Polytopic Uncertainty Sets.} As is common in the algorithmic study of RMDPs and RPOMDPs~\citep{NilimG03,iyengar2005robust}, we restrict our attention to RPOMDPs with {\em $(s,a)$-rectangular uncertainty sets}. That is, we assume that the uncertainty set is of the form $\uncert = \prod_{(s,a)\in S \times A} \uncertsa{(s,a)}$ with each $\uncertsa{(s,a)} \subseteq \dist(S)$, meaning that the environment can choose transition probabilities independently across state-action pairs. Moreover, we assume that each \(\uncertsa(s,a)\) is a {\em polytope} with a finite vertex set \(V_{s,a}\). 
$\uncertsa(s,a) = \{\sum_{i=1}^{|V_{s,a}|} \alpha_i v_i \;|\; \alpha_i \in [0,1], \sum_i \alpha_i = 1, v_i \in V_{s,a}\},$ where each vertex \(v \in V_{s,a}\) has the form \((x_1,\dots,x_{|S|})\) with  \(x_i \geq 0 \text{ and }\sum_i x_i = 1\). 
We take every $V_{s,a}$ to be given explicitly as part of the input. Accordingly, all model sizes and linear-time claims are measured with respect to these explicit vertex lists. 
A point in $\uncertsa(s,a)$ is a distribution over $S$, see Figure \ref{fig:rpomdp_red-main}(a). 
Standard examples include $L^1$-uncertainty sets and interval (PO)MDPs~\citep{HoPW21,GivanLD00}.
In the rest of the paper, RPOMDPs refer to $(s,a)$-rectangular polytopic RPOMDPs.

\noindent{\bf Semantics of RPOMDPs.} Informally, the semantics of RPOMDPs are defined as follows. The process begins in state $s_0$ with observations $(\obfna(s_0), \obfne(s_0))$. At each step $t$, an agent chooses an action $a_t$; the agent does not know the exact state $s_t$, but only the agent observation $o^a_t=\obfna(s_t)$.
In response, using its (the environment) observation history and $a_t$, the adversarial environment selects a complete transition assignment, namely, a $|S| \times |A|$ vector $P_t\in\uncert$; the transition from the actual state uses its component $P_t(s_t,a_t)\in\uncert(s_t,a_t)$ ($\uncert(s,a)$ denotes the $(s,a)$th component of $\uncert$).
The environment observes only $o^e_t=\obfne(s_t)$. The next state $s_{t+1}$ is sampled from $P_t(s_t,a_t)$, the agent observes $\obfna(s_{t+1})$ and the process continues. 
This process is formalized via agent and environment strategies. Following the terminology of \citep{BovySJ024}, our semantics model {\em dynamic uncertainty} where the environment may select a new transition assignment at each time step.

A \emph{history} is a finite sequence  
$
h_t = s_0,(o^a_0,o^e_0), s_1, \dots, s_t,(o^a_t,o^e_t) \in (S \times \obsa \times \obse)^\ast
$.  An \emph{agent} (resp., \emph{environment}) \emph{action invisible observation history} for $h_t$ is a sequence  
$
oh^a_t = o^a_0,\dots,o^a_t \in (\obsa)^* \times \obsa
$
(resp., 
$
oh^e_t = o^e_0,\dots,o^e_t \in (\obse)^* \times \obse
$).
Let \(\OH_a\) and \(\OH_e\) denote the sets of agent and environment observation histories. 
An \emph{agent strategy} and an \emph{environment strategy} are mappings
$
\sigma: \OH_a \to \dist(A),
\pi: \OH_e \times A \to \uncert,$
respectively. Thus $\sigma(oh)(a)$ is the probability of choosing action $a$ after observation history $oh$, independent of the hidden current state. The environment still chooses a complete transition assignment in response to the action played: $\pi(oh,a)$ is a tuple with one component $\pi(oh,a)(s,a)\in \uncertsa_{s,a}$ for every state $s$ compatible with $oh$, since the local uncertainty sets differ from state to state. 

Given a RPOMDP $\rmdp$, let \(\Sigma_{\rmdp}\) and \(\Pi_{\rmdp}\) denote all agent and environment strategies.
A strategy is \emph{finite-memory} if it can be implemented by a finite-state transducer that reads the player's observation history and outputs the required action distribution or transition assignment. It is \emph{memoryless} if one memory state suffices; strategies not restricted to finite memory are called \emph{infinite-memory} strategies.
A pair of strategies \((\sigma,\pi) \in \Sigma_{\rmdp} \times \Pi_{\rmdp} \) induces a \emph{play}  
$
\rho = s_0,a_0,s_1,a_1,\dots \in (S \times A)^\omega
$
starting from \(s_0\). At time \(t\), the agent selects $\alpha_t=\sigma(oh^a_t)\in\dist(A)$ and samples $a_t$ from $\alpha_t$. The environment then selects $P_t=\pi(oh^e_t,a_t)\in\uncert$. The next state satisfies $\Pr(s_{t+1}\mid h_t,a_t)=P_t(s_t,a_t)(s_{t+1})$. Evaluating the current-state component of $P_t$ does not reveal that state to the environment.
Let \(\Playsr\) denote the set of all plays in $\rmdp$. Given an initial state $s_0$, each pair of strategies \((\sigma,\pi) \in \Sigma_{\rmdp} \times \Pi_{\rmdp} \) induces a probability measure $\probm_{s_0}^{\sigma,\pi}(\cdot)$ over the set \(\Playsr\)~\citep{bill}. 
 $\Plays_{\rmdp}^{\sigma,\pi}$ denotes the set of \emph{supported plays}: those plays for which every finite prefix has positive probability under $(\sigma,\pi)$.

\noindent{\bf $\omega$-regular Objectives.} A logical objective in an RPOMDP $\rmdp$ is a measurable set $W \subseteq \Playsr$. 
We consider $\omega$-regular objectives ~\citep{ClarkeHVB18}, specified via a parity automaton. Standard examples are reachability (visit a target state at least once), safety (avoid a bad state forever), B\"uchi (visit a target state infinitely often), and co-B\"uchi (visit a bad state only finitely often). It is a classical result in formal methods that every $\omega$-regular objective can be represented via a parity automaton~\citep{thomas}.  
Parity objectives in RPOMDP are realised by  equipping them with a priority function $p:S\rightarrow \{0,1,\dots ,d\}$, $d\in \nat$ assigning a priority to each state. The $\omega$-regular objective is then the set of plays in which the minimum priority that is visited infinitely often is even, i.e.
    $
    \parity{\rmdp}{\prfn} = \left\{ \rho \in \Plays_\rmdp \middle| \min\{\prfn(s) : s \in \Inf(\rho)\}  \text{ is even}\right\}
    $.
  Here $\Inf(\rho)$ is the set of states that occur infinitely often in $\rho$.

Throughout, we make the standard assumption in RPOMDPs (and later in ac-POSGs introduced in Section~\ref{sec:posg}) that every action in a player's action alphabet is enabled at every state owned by that player. An action is 
 \emph{effective} at a state if the construction explicitly specifies its transition there, and an \emph{effective edge} is a positive-probability transition under an effective action. In figures, we only display effective actions and edges; every omitted state-action pair is implicitly completed by an absorbing outcome losing for its owner. This completion is without loss of generality.

\noindent{\bf Problem.} Given an RPOMDP $\rmdp$ with an $\omega$-regular objective $W$, its {\em value} and its value under a fixed agent strategy $\sigma$ are defined as $\val_{\rmdp}(W) = \sup_{\sigma \in \Sigma_\rmdp} \inf_{\pi \in \Pi_\rmdp} \probm_{s_0}^{\sigma, \pi}(W)$ and $\val^\sigma_{\rmdp}(W) = \inf_{\pi \in \Pi_\rmdp} \probm_{s_0}^{\sigma, \pi}(W)$, respectively. 
We consider the four classical $\omega$-regular analysis problems:
\begin{compactitem}
\item \emph{Almost-sure} analysis asks whether there exists an agent strategy $\sigma$ such that $\val^\sigma_{\rmdp}(W)=1$.
\item \emph{Limit-sure} analysis asks whether, for every $\epsilon>0$, there exists an agent strategy $\sigma_\epsilon$ such that $\val^{\sigma_\epsilon}_{\rmdp}(W)\geq1-\epsilon$, equivalently, whether $\val_{\rmdp}(W)=1$.

\item \emph{Sure} analysis asks whether there exists an agent strategy~$\sigma$ such that $\Plays^{\sigma, \pi}_{\rmdp} \subseteq W$ for every environment strategy~$\pi$.

\item \emph{Quantitative} analysis asks, for a given probability threshold $p\in[0,1]$, if $\val_{\rmdp}(W)\geq p$. 
\end{compactitem}

\section{Equivalence with Partially Observable Stochastic Games}

This section presents the main result of the paper. We provide reductions between RPOMDPs and 
alternating control POSGs for each problem above. Both reductions are linear in the model size.

\begin{figure*}[t]
\centering
\begin{subfigure}{0.4\textwidth}
\centering
    \begin{tikzpicture}[
    >={stealth[scale=1.3]},
    state/.style={circle, draw=maxcolor, thick, minimum size=0.75cm},
    stochNode/.style={circle, fill=mincolor, inner sep=0pt, minimum size=1.5mm},
    edge/.style={thick, ->, >=stealth,draw=black!80},
    scale = 0.5
]
    \node[state] (s1) at (3,3) {$s_1$};
    \node[state] (s2) at (0,0) {$s_2$};
    \node[state] (s3) at (6,0) {$s_3$};

    \node[stochNode] (j1) at (3, 5)  {};   
    \node[stochNode] (j2) at (1.75, 0)  {};   
    \node[stochNode] (j3) at (4.25, 0) {};   

    \draw[edge] (1.5,3) -- (s1) ;

    \draw[edge] (s1) -- (j1)  node[near start, left] {\tiny $a$};
    \draw[edge] (j1) to[bend right = 30] node[pos=0.1, above] {\tiny $p_1$} (s2);
    \draw[edge] (j1) to[bend left=30] node[pos=0.1, right] {\tiny $r_1$} (s3);

    \draw[edge] (s2) -- (j2)  node[near start, above] {\tiny $a$};
    \draw[edge] (j2) to[bend right=-25] node[pos=0.1,right] {\tiny $p_2$} (s1);
    \draw[edge] (j2) to[bend right=-50] node[pos=0.3,below] {\tiny $q_2$} (s2);
    \draw[edge] (j2) to[bend right=35] node[pos=0.1,right] {\tiny $r_2$} (s3);

    \draw[edge] (s3) -- (j3)  node[midway, below] {\tiny $a$};
    \draw[edge] (j3) to[bend right=-40] node[pos=0.5, above] {\tiny $r_3$} (s3);
    \draw[edge] (j3) to[bend right=30] node[pos=0.5, right] {\tiny $p_3$} (s1);

    \draw[edge] (s1) edge[loop below] node {\tiny $b,1$} (s1);
    \draw[edge] (s2) edge[loop below] node {\tiny $b,1$} (s2);
    \draw[edge] (s3) edge[loop below] node {\tiny $b,1$} (s3);

\end{tikzpicture}
\label{fig:sarrpomdp}
\end{subfigure}
\hfill
\begin{subfigure}{0.5\textwidth}
\centering
\begin{tikzpicture}[
    >={stealth[scale=1.2]},
    maxstate/.style={circle, draw=maxcolor, thick, minimum size=0.75cm},
    minstate/.style={rectangle, draw=mincolor, thick, minimum size=0.2cm},
    dn/.style={circle, fill=black, inner sep=1pt, minimum size=0.75mm},
    edge/.style={->, thick},
    scale = 0.5
]

\node[maxstate] (s1) at (0,0) {$s_1$};
\node[maxstate] (s2) at (3,0) {$s_2$};
\node[maxstate] (s3) at (-3,0) {$s_3$};

\node[minstate] (s1a) at (0,3) {\tiny $(s_1,a)$};
\node[minstate] (s2a) at (3,-3) {\tiny $(s_2,a)$};
\node[minstate] (s3a) at (-3,-3) {\tiny $(s_3,a)$};

\node[dn] (s1av1) at (-1,4) {};
\node[dn] (s1av2) at (1,4) {};
\node[dn] (s2av1) at (5,-2) {};
\node[dn] (s2av2) at (5,-3) {};
\node[dn] (s2av3) at (5,-4) {};
\node[dn] (s3av1) at (-5,-2) {};
\node[dn] (s3av2) at (-5,-4) {};

\node[minstate] (s1b) at (0,-3) {\tiny $(s_1,b)$};
\node[minstate] (s2b) at (3,3) {\tiny $(s_2,b)$};
\node[minstate] (s3b) at (-3,3) {\tiny $(s_3,b)$};

\draw[->,thick] (-1.5,0) -- (s1);

\draw[->,thick] (s1) -- node[left] {$a$} (s1a);
\draw[->,thick] (s1) to[bend right=30] node[midway, left] {$b$}  (s1b);

\draw[->,thick] (s2) -- node[pos = 0.4, right] {$a$} (s2a);
\draw[->,thick] (s2) to[bend left=30] node[midway, left, inner sep=0pt] {$b$}  (s2b);

\draw[->,thick] (s3) -- node[left] {$a$} (s3a);
\draw[->,thick] (s3) to[bend right=30] node[midway, right,inner sep=0pt] {$b$}  (s3b);

\draw[->,thick] (s1a) -- node[pos=0.8,left] {\tiny$v_{11}$} (s1av1);
\draw[->,thick] (s1a) -- node[pos=0.8,right] {\tiny$v_{12}$} (s1av2);
\draw[->,thick] (s2a.east) -- node[pos=1,right] {\tiny$v_{21}$} (s2av1);
\draw[->,thick] (s2a.east) -- node[pos=1,right] {\tiny$v_{22}$} (s2av2);
\draw[->,thick] (s2a.east) -- node[pos=1,right] {\tiny$v_{23}$} (s2av3);
\draw[->,thick] (s3a.west) -- node[pos=1,left] {\tiny$v_{31}$} (s3av1);
\draw[->,thick] (s3a.west) -- node[pos=1,left] {\tiny$v_{32}$} (s3av2);

\draw[->,dashed,thick] (s1av1) .. controls (2,6) and (6,5) .. (s2.east);
\draw[->,dashed,thick] (s1av1) -- (s3.east);
\draw[->,dashed,thick] (s1av2) -- (s2.west);
\draw[->,dashed,thick] (s1av2) .. controls (-2,6) and (-6,5) .. (s3.west);

\draw[->,thick] (s1b) to[bend right=30] node[midway, right] {\tiny $(1,0,0)$} (s1);
\draw[->,thick] (s2b) to[bend left=30] node[midway, right] {\tiny $(0,1,0)$} (s2);
\draw[->,thick] (s3b) to[bend right=30] node[midway, left] {\tiny $(0,0,1)$} (s3);

\end{tikzpicture}
\end{subfigure}
 \caption{\small (a) 
  RPOMDP $\rmdp$ with states $s_1, s_2, s_3$ and actions $a,b$. Each transition arrow consists of two parts, the first denotes an action ($a$ or $b$) whereas the second denotes the probability of moving to a state upon taking that action (e.g.~$p_1$ is the probability of moving to $s_2$ upon taking action $a$ in $s_1$). The uncertainty set of the RPOMDP is defined by specifying the vertex set of the uncertainty polytope for each state-action pair. We let $V_{s_1,a} {=} \{v_{11}, v_{12}\}$ and $(0,p_1,r_1) \in \uncertsa(s_1,a)$. 
 $V_{s_2,a} {=} \{v_{21}$,$v_{22}$,$v_{23}\}$ and 
 $(p_2, q_2, r_2) \in \uncertsa(s_2,a)$.  $V_{s_3,a} {=} \{v_{31}, v_{32}\}$
  and $(p_3,0,r_3) \in \uncertsa(s_3,a)$. Note that on $b$, all transitions are deterministic, so  $V_{s_1,b}=\{(1,0,0)\}$, $V_{s_2,b}=\{(0,1,0)\}$ and 
  $V_{s_3,b}=\{(0,0,1)\}$. 
 (b) The ac-POSG constructed from the RPOMDP, circular and rectangular states represent max-player and min-player states, respectively.
  The action set of the max-player is the same as the action set in the RPOMDP on the left, namely $\{a,b\}$. 
  Min-player states are all combinations of state-action pairs in the RPOMDP on the left. The effective actions at each such state are the vertices of the corresponding uncertainty polytope. For instance, at $(s_1,a)$ these are $v_{11}$ and $v_{12}$ from $V_{s_1,a}$; randomizing with probabilities $(\alpha_1,\alpha_2)$ mimics the point $\alpha_1v_{11}+\alpha_2v_{12}$ in $\uncertsa(s_1,a)$. We omit the remaining completed actions and the effective edges from $(s_2,a)$ and $(s_3,a)$ for clarity.}

\label{fig:rpomdp_red-main}
\end{figure*}
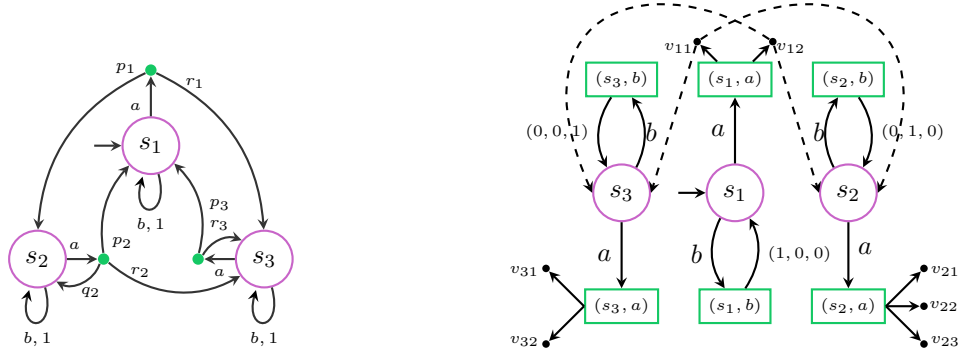

\noindent{\bf Partially Observable Stochastic Games (POSGs).}
\label{sec:posg}
A turn-based, two-player ($\pone,\ptwo$) \emph{POSG}~\citep{DBLP:journals/fmsd/Chatterjee0H13} is a tuple $
\sg = (\sone,\stwo,\aone,\atwo,\delta,s_0,\obsone,\obstwo, \obfnone,\obfntwo),
$
where \(\sone\) and \(\stwo\) are the sets of \(\pone\)- and \(\ptwo\)-states, respectively. Write $S^\sg=\sone\uplus\stwo$, with initial state $s_0\in S^\sg$. The action sets of the two players are $\aone$ and $\atwo$,  
and we write $A^\sg=\aone\cup\atwo$ and $\aone(s)$ (resp. $\atwo(s)$) denote the set of effective actions at $s \in \sone$ (resp. $s \in \stwo$). The transition function is
$
\delta : (\sone \times \aone) \cup (\stwo \times \atwo) \to \dist(S^\sg).
$
Observations are given by sets \(\obsone,\obstwo\) and functions \(\obfnone:S^\sg\to\obsone\), \(\obfntwo:S^\sg\to\obstwo\). A POSG is \emph{alternating control (ac-POSG)} if $\delta(s,a)\in\dist(\stwo)$ for every $(s,a)\in\sone\times\aone$ and $\delta(s,a)\in\dist(\sone)$ for every $(s,a)\in\stwo\times\atwo$. Throughout, ac-POSGs start in a max-state, $s_0\in\sone$; otherwise, a fresh deterministic initial max-state can be prepended.

{A \emph{history} of \(\sg\) is a sequence
$
h_t = s_0,(o^1_0,o^2_0),s_1,\dots,s_t,(o^1_t,o^2_t)
$
in \((S^\sg \times \obsone \times \obstwo)^*\) such that for all \(0\le i \le t\),
$o^1_i=\obfnone(s_i), o^2_i=\obfntwo(s_i)$. Moreover, for every $0\leq i<t$, there is an
action $a\in\aone$ if $s_i\in\sone$, or $a\in\atwo$ if $s_i\in\stwo$, such that
$\delta(s_i,a)(s_{i+1})>0$. For $\star\in\{\pone,\ptwo\}$, if $h_t$ ends in an
$S_\star$-state, its \emph{$\star$-(action invisible) observation history} is the sequence of that player's
observations $oh^\star=o^\star_0,\ldots,o^\star_t$. We write $\OH_\star^\sg$ for the set of
all such histories.
For $o\in\obstwo$, define the \emph{observation class} of $o$ as
$\stwo(o)=\{ s \in\stwo:\obfntwo(s)=o\}$.
A max-player strategy $\sigma$ maps each observation history to a single action
distribution. The min-player uses \emph{factored prescriptions} $\pi$: for an
observation $o$, a prescription chooses one local action distribution for every min-state in $\stwo(o)$.} Thus, we define max-player strategy
$
\sigma:\OH_\pone^\sg\to\dist(\aone)$ and min-player strategy $
\pi:\OH_\ptwo^\sg\to
\bigcup_{o\in\obstwo}
\prod_{s\in\stwo(o)}\dist(\atwo),
$
where $\pi(oh)$ belongs to the product for $o=\Last(oh)$ and each component is supported
on the effective min-actions at its state (Section~\ref{sec:prelim}).

Plays induced by $(\sigma,\pi)$ are defined analogously to RPOMDP plays, and $\Playssg$
denotes the set of all such plays. Probability measures over plays, priority functions,
objectives $W\subseteq\Playssg$, the value $\val_\sg(W)$, and the four analysis problems
are all defined exactly as for RPOMDPs. Let $\Sigma_\sg$ and $\Pi_\sg$ denote the sets of
all $\pone$- and $\ptwo$-strategies, respectively.
The computational analysis of ordinary observation-based POSGs was studied by
\cite{DBLP:journals/fmsd/Chatterjee0H13}, whose min-strategies
$\widehat\pi:\OH_\ptwo^\sg\to\dist(\atwo)$ assign a single distribution to each observation
history. Our factored prescriptions generalize this notion: Lemmas~\ref{lem:prescription-compilation}
and~\ref{lem:ordinary-as-prescription} compare these notions, and the fact that this {\em does not affect our complexity results} 
is established at the end of \cref{sec:sg2rmdp}.

\subsection{Reduction from RPOMDPs to ac-POSGs}
\label{sec:forward}

Our reduction from RPOMDPs to ac-POSGs draws insight from and generalizes the reduction from RMDPs to stochastic games that was presented in~\citep{DBLP:conf/ijcai/2024}. In this work, we generalize this reduction to the setting with partial observability. In what follows, given a polytopic RPOMDP $\rmdp = (S, A, \uncert, \obsa, \obse, \obfna, \obfne, s_0)$ with any of the objectives defined in Section~\ref{sec:prelim}, 
we outline our construction of the equivalent ac-POSG~$\sg = (\sone^{\sg},\stwo^{\sg},\aone^{\sg},\atwo^{\sg},\delta^{\sg},s_0^{\sg},\obsone^{\sg},\obstwo^{\sg},\obfnone^{\sg},\obfntwo^{\sg})$ and the corresponding objective. The formal construction is deferred to Appendix~\ref{app:rpomdp-posg}. 

\noindent{\bf POSG Construction.} The POSG $\sg$ is defined as follows. Intuitively, the ac-POSG models the interaction between the agent and the environment in the original RPOMDP, where the max-player corresponds to the agent, the min-player to the environment, and the two players alternate in moves. The constructed POSG is {\em linear} in the size of the RPOMDP and the number of vertices of the uncertainty polytopes. Figure~\ref{fig:rpomdp_red-main}(b) shows an example of our construction.

\emph{States:} max-player states $\sone^{\sg} = S$ are the states of
the original RPOMDP, min-player states $\stwo^{\sg}$ are all state--action
pairs $(s,a) \in S\times A$ with $a$ effective at $s$, and the initial state is $s_0^{\sg} =
s_0 \in \sone^{\sg}$. \emph{Actions:} The max-player's actions are
those of the RPOMDP, $\aone^{\sg}=A$, with the same effective support at $s$. At min-state $(s,a)$ the effective actions are $V_{s,a}$, and the global min-action alphabet is $\atwo^{\sg}=\bigcup_{(s,a)\in\stwo^{\sg}}V_{s,a}$. Every min-state is supported on its local vertex set. \emph{Observations:}
$\obsone^{\sg} = \obsa$ and $\obstwo^{\sg} = \obse$, with
$\obfnone^{\sg}(s) = \obfnone^{\sg}((s,a)) = \obfna(s)$ and $\obfntwo^{\sg}(s) = \obfntwo^{\sg}((s,a)) =
\obfne(s)$. \emph{Transitions:} from max-state $s$ under action $a$,
the POSG moves deterministically to the min-state $(s,a)$, i.e.\
$\delta^{\sg}(s,a) = \dirac{(s,a)}$; from min-state $s' = (s,a)$
under an effective action $v \in V_{s,a}$, $\delta^{\sg}(s', v) = v$ is the distribution over $\sone^{\sg}$ specified by the vertex $v$ of the uncertainty polytope $\uncertsa(s,a)$.
Randomizing over min-actions
therefore produces every distribution in $\uncertsa(s,a)$. 

\noindent{\bf POSG Objectives.} 
 The $\omega$-regular objective in the ac-POSG $\sg$ is specified by the priority function $p' : S^{\sg} \to \{0,1,\ldots,d\}$ defined via $p'(s') = p(s)$ 
for any $s' \in S^{\sg}$, where $s$ is the state component of $s'$ ($s'=s$ or $s'=(s,a)$) and $p$ is the priority function of $\rmdp$.  

\begin{theorem}[Proof in Appendix \ref{app:thm1}]
Let $\rmdp$ be a RPOMDP and $\sg$ the induced ac-POSG. 
Then $\val_\rmdp(W) {=} \val_\sg(W')$
for any $\omega$-regular objective $W$ in $\rmdp$ and corresponding objective 
$W'$ in $\sg$.\\
Consequently, sure, almost-sure, limit-sure, and quantitative
analysis of any objective above on $\rmdp$ is inter-reducible in linear time, and therefore in polynomial time, with the corresponding analysis on $\sg$, under the explicit vertex representation fixed in Section~\ref{sec:prelim}.
\label{rpomdp-posg}
\end{theorem}

\subsection{Reduction from ac-POSGs to RPOMDPs}
\label{sec:ac2rpomdp}
\noindent{\bf Key Challenge and Our Solution.} Reduction from ac-POSGs to RPOMDPs turns out to be much more challenging due to the following subtle reason. Given an ac-POSG, if the max-player in a state $s$ chooses action~$a$, the min-player gets to see the observation $\obfntwo(s')$ where $\delta(s,a)(s')>0$ before choosing its own action.
On the contrary, in an RPOMDP,  after the agent chooses an action $a$ at a state $s$, the environment at $(s,a)$ commits to a polytope ~$P(s,a)$ immediately, with nothing observed in between.  Thus, at a first glance, it seems that the min-player in ac-POSGs has more power compared to the environment in RPOMDPs. Somewhat surprisingly, we show that this is not the case and that it is still possible to construct a reduction from ac-POSGs to RPOMDPs. This observation motivates us to construct a two-step reduction. The first step of the reduction reduces ac-POSGs to an intermediate model which we call {\em pre-min-transformed POSG}, which resolves this semantic discrepancy. The second step then reduces a pre-min-transformed POSG to an RPOMDP.

\subsubsection{Step 1: Reduction from ac-POSGs to Pre-min-transformed POSGs}
\label{sec:step1}
The pre-min transformation inserts,
before every $\ptwo$-state $s \in \stwo$, a fresh \emph{dummy}
max-state $d_s$ whose sole effective action leads
deterministically to~$s$, and whose observations to both players
are copies of the observations of~$s$. Neither player gains
strategic power: the max-player is deterministic at $d_s$
(every other globally enabled action takes the max-losing default transition), while the
min-player sees the same information at $s$ as in $\sg$, the
observation at $d_s$ is determined by the observation at $s$ and so
reveals nothing new. The following construction formalizes this intuition.

\noindent{\bf Pre-min-transformed POSG.} Let $\sg = (\sone,\stwo,\aone,\atwo,\delta,s_0,\obsone,\obstwo,\obfnone,\obfntwo)$ be an ac-POSG, $D=\{d_s \mid s \in \stwo\}$ be copies of all min-player's states, and $\Obsnewone = \{\obsnewone_x \mid x \in \obsone\}$ and $\Obsnewtwo = \{\obsnewtwo_x \mid x \in \obstwo\}$ be copies of observations for both players. The corresponding pre-min- transformed POSG $\sg' = (\sone', \stwo', \aone', \atwo', \tran', s_0, \obsone', \obstwo', \obfnone', \obfntwo')$ is defined via:
\begin{compactitem}
    \item {\em States}.  $\sone'= \sone \cup D$ and $\stwo'=\stwo$, with initial state $s_0$.
    \item {\em Actions.} $\aone'=\aone\cup\{\anew\}$ and $\atwo'=\atwo$. Original states retain their effective alphabets $\aone(s)$ and $\atwo(s)$, while dummy states have sole effective action $\anew$. 
    \item {\em Transition function.} (i) For $s_1 \in \sone$, $a_1 \in \aone(s_1)$ and $s_2 \in \stwo$, $\tran'(s_1,a_1)(d_{s_2})=\tran(s_1,a_1)(s_2)$. (ii) For $s_2\in\stwo$ and $a_2\in\atwo(s_2)$, set $\tran'(s_2,a_2)=\tran(s_2,a_2)$. (iii) At a dummy state set $\tran'(d_s,\anew)=\dirac{s}$.
    \item {\em Observations.} $\obsone' =\ \obsone \cup {\Obsnewone}$ and $\obstwo' = \obstwo \cup {\Obsnewtwo}$, with $\obfnone'(s) = \obfnone(s)$ and $\obfntwo'(s) = \obfntwo(s)$ for all $s \in \sone \cup \stwo$, while $\obfnone'(d_s) = \obsnewone_{\obfnone(s)}$ and $\obfntwo'(d_s) = \obsnewtwo_{\obfntwo(s)}$ for all $d_s \in D$ are the newly introduced observation copies.
\end{compactitem}

\noindent{\bf Pre-min-transformed POSG objective.} 
Plays of $\sg$ and $\sg'$ are related by a bijection $\Gamma :
\Plays_\sg \to \Plays_{\sg'}$. Since $\sg$ is alternating control,
every $h = s_0, a_0, s_1, a_1, \dots \in \Plays_\sg$ has ${s_{2k}}{\in}
{\sone}$ and $s_{{2k+1}} {\in} {\stwo}$, and $\Gamma$ inserts the pair
$(d_{s_{2k+1}}{,}\anew)$ between every $\sone$-step $s_{2k}$ and the
$\stwo$-step $s_{2k+1}$. 
Observation histories lift the same way via $\GammaOpone,
\GammaOptwo$. Strategies in $\sg, \sg'$ are functions of observation
histories. We lift them between the two games as follows: for
$\star \in \{\pone, \ptwo\}$, $\Phi_\star$ takes a $\sg$-strategy $\sigma$ to the
$\sg'$-strategy that, given a $\sg'$-observation history $oh'$, strips the inserted
observations and queries $\sigma$ at the underlying $\sg$-history:
$\Phi_\star(\sigma)(oh') = \sigma((\Gamma_O^\star)^{-1}(oh'))$;
at dummy states $\Phi_\pone(\sigma)$ plays $\anew$ deterministically.
For max this equality copies one action distribution; for min it copies the complete prescription.
For $\sg'$-strategy $\sigma'$, reverse $\Psi_\star$  goes
the other way: $\Psi_\star(\sigma')(oh) =
\sigma'(\Gamma_O^\star(oh))$. The pairs are mutual inverses,
$\Phi_\star = (\Psi_\star)^{-1}$, which helps us lift objectives
across to $\sg'$.

  \noindent  $\bullet$ \textbf{Parity objective.}  Let $\prfn: \sone \cup \stwo \to \{0,1,\dots,d\}$ be
          the priority function for $\sg$.  Define $\prfn': \sone' \cup \stwo' \to \{0,1,\dots,d\}$
          by $\prfn'(s) = \prfn(s)$ for $s \in \sone \cup \stwo$ and
          $\prfn'(d_s) = \prfn(s)$ for $d_s \in D$.

\noindent {\bf{Size and memory preservation under Pre-$\ptwo$}}. It is easy to see that  $|\sg'| = \mathcal{O}(|\sg|)$ and the
        construction runs in linear time. Moreover, $|\Gamma_O^\star(oh)|$ and $|(\Gamma_O^\star)^{-1}(oh')|$ are linear in $|oh|$ and $|oh'|$, respectively.
 Finally, the  lifting maps
        $\Phi_\star, \Psi_\star$ preserve memory class exactly: 
        memoryless, finite-memory
         and infinite-memory regimes correspond exactly
        between $\sg$ and $\sg'$. See details in Appendix~\ref{app:premin-size-memory}.

\begin{lemma}[Equivalence between $\sg$ and $\sg'$, Appendix \ref{app:thm2}]
Let $\sg$ be an  ac-POSG and $\sg'$ its Pre-$\ptwo$
transformation. The play map
$\Gamma$  and the strategy maps
$\Phi_\star, \Psi_\star$ ($\star \in \{\pone,\ptwo\}$)  satisfy the following: for every
 priority function
$\prfn$ on $\sg$, objective $W$ and its corresponding lifted objective $W'$ in $\sg'$,
and every $\rho \in \Plays_{\sg}, \sigma \in \stratpone_\sg, \pi \in \stratptwo_\sg$: \\
\textup{(I, obj.)} $\rho \in W \!\iff\! \Gamma(\rho) \in W'; \qquad \qquad\qquad\,\,\,$
\textup{(II, plays)} $\Gamma(\Plays_\sg^{\sigma,\pi}) = \Plays_{\sg'}^{\Phi_\pone(\sigma),\Phi_\ptwo(\pi)};$\\
\textup{(III, prob.)} $\probm_\sg^{\sigma,\pi}(W) = \probm_{\sg'}^{\Phi_\pone(\sigma),\Phi_\ptwo(\pi)}(W'); \qquad\!$\;
\textup{(IV, val.)} $\val_\sg(W) = \val_{\sg'}(W')$. \\

Consequently, sure, almost-sure, limit-sure, and quantitative
analysis of any objective above on $\sg$ is inter-reducible
in polynomial
time with the corresponding analysis on $\sg'$. 
\label{thm:prenorm-main}
\end{lemma}
    
\begin{figure*}[t]
\centering
\begin{adjustwidth}{0pt}{0pt}
\begin{subfigure}[t]{0.3\textwidth}
    \begin{tikzpicture}[
    maxNode/.style={circle, draw=maxcolor, thick, minimum size=0.25cm},
    minNode/.style={rectangle, draw=mincolor, thick, minimum size=0.25cm},
    stochNode/.style={circle, fill=black, inner sep=0pt, minimum size=1mm},
    edge/.style={thick, ->, >=stealth,draw=black!80},
    scale = 0.4,
]
    \node[maxNode] (s1) at (0, 0)   {\tiny $s_1$};
    \node[maxNode] (s2) at (8, 0)  {\tiny $s_2$};
    \node[minNode] (t1) at (3.5, 2.5) {\tiny $t_1$};
    \node[minNode] (t2) at (3.5,-2.5) {\tiny $t_2$};
    \node[stochNode] (ja)  at (1.5,  1.25) {};
    \node[stochNode] (jb)  at (1.5, -1.25) {};
    \node[stochNode] (2ja) at (6,  0)    {};
    \node[stochNode] (2jy) at (6, -2.50) {};
    \node[stochNode] (2jx) at (1.5, -3) {};
    \node[stochNode] (1jx) at (6,  2.50) {};
    \draw[edge] (-1.5,0) -- (s1);
    \draw[edge] (s1) -- (ja)  node[pos=0.1,above]  {\tiny $a$};
    \draw[edge] (s1) -- (jb)  node[pos=0.5,above]  {\tiny $b$};
    \draw[edge] (s2) -- (2ja) node[midway,above]       {\tiny $a$};
    \draw[edge] (t2) -- (2jx) node[pos=0.4,above,inner sep=0.5pt] {\tiny $x$};
    \draw[edge] (t2) -- (2jy) node[pos=0.3,below]       {\tiny $y$};
    \draw[edge] (t1) -- (1jx) node[pos=0.2,above]       {\tiny $x$};
    \draw[edge] (ja) -- (t1); \draw[edge] (ja) -- (t2);
    \draw[edge] (jb) -- (t1); \draw[edge] (jb) -- (t2);
    \draw[edge] (2ja) -- (t1); \draw[edge] (2ja) -- (t2);
    \draw[edge] (2jx) -- (s1);
    \draw[edge] (2jx) to[in = 270, out = 315] (s2);
    \draw[edge] (2jy) -- (s2);
    \draw[edge] (1jx) -- (s2);
    \draw[edge] (1jx) to[in = 100, out = 75] (s1);
\end{tikzpicture}
\label{fig:posg}
\end{subfigure}
\begin{subfigure}[t]{0.3\textwidth}
\begin{tikzpicture}[
    maxNode/.style={circle, draw=maxcolor, thick, minimum size=0.25cm},
    minNode/.style={rectangle, draw=mincolor, thick, minimum size=0.25cm},
    botNode/.style={circle, draw=maxcolor, dashed, thick, minimum size=0.2cm},
    stochNode/.style={circle, fill=black, inner sep=0pt, minimum size=1mm},
    edge/.style={->, thick, draw=black!80},
    botEdge/.style={->, thick, dashed, draw=black!60},
    scale=0.4,
]
    \node[maxNode] (s1) at (0,0)    {\tiny $s_1$};
    \node[maxNode] (s2) at (10,0)   {\tiny $s_2$};
    \node[botNode] (dt1) at (4, 2)  {\tiny $d_{t_1}$};
    \node[botNode] (dt2) at (4,-2)  {\tiny $d_{t_2}$};
    \node[minNode] (t1) at (7, 2.5)  {\tiny $t_1$};
    \node[minNode] (t2) at (7,-2.5)  {\tiny $t_2$};
    \node[stochNode] (ja)  at (2.5, 1.25) {};
    \node[stochNode] (jb)  at (2.5,-1.25) {};
    \node[stochNode] (2ja) at (7.5, 0)    {};
    \node[stochNode] (2jy) at (8.5,-2.5)  {};
    \node[stochNode] (2jx) at (3,-4)  {};
    \node[stochNode] (1jx) at (8.5, 2.5)  {};
    \draw[edge] (-1.5,0) -- (s1);
    \draw[edge] (s1) -- (ja)  node[pos=0.4,above left, inner sep=0.5pt]  {\tiny $a$};
    \draw[edge] (s1) -- (jb)  node[pos=0.4,below left, inner sep=0.5pt]  {\tiny $b$};
    \draw[edge] (s2) -- (2ja) node[midway,above]       {\tiny $a$};
    \draw[edge] (ja)  -- (dt1); \draw[edge] (ja)  -- (dt2);
    \draw[edge] (jb)  -- (dt1); \draw[edge] (jb)  -- (dt2);
    \draw[edge] (2ja) -- (dt1); \draw[edge] (2ja) -- (dt2);
    \draw[botEdge] (dt1) -- (t1) node[midway,above, inner sep = 0.5pt] {\tiny $\anew$};
    \draw[botEdge] (dt2) -- (t2) node[midway,above, inner sep = 0.5pt] {\tiny $\anew$};
    \draw[edge] (t1) -- (1jx) node[midway,above] {\tiny $x$};
    \draw[edge] (t2) to[in=5, out = 225] node[pos=0.1,below right,inner sep=0] {\tiny $x$} (2jx);
    \draw[edge] (t2) -- (2jy) node[midway,below] {\tiny $y$};
    \draw[edge] (1jx) -- (s2);
    \draw[edge] (1jx) to[in = 80, out = 90] (s1);
    \draw[edge] (2jx) -- (s1);
    \draw[edge] (2jx) to[in = 290, out = 290] (s2);
    \draw[edge] (2jy) -- (s2);
\end{tikzpicture}
\label{fig:gprime}
\end{subfigure}
\begin{subfigure}[t]{0.55\textwidth}
\centering
\begin{tikzpicture}[,
    maxNode/.style={circle, draw=maxcolor, thick, minimum size=0.25cm},
    botNode/.style={circle, draw=maxcolor, dashed, thick, minimum size=0.25cm},
    uncSet/.style={rounded corners=3.5pt, dotted, draw=mincolor, thick, inner sep=3pt},
    stochNode/.style={circle, fill=mincolor, inner sep=0pt, minimum size=1.5mm},
    edge/.style={->, thick, draw=black!80},
    botEdge/.style={->, thick, dashed, draw=black!60},
    scale=0.4,
]
    \pgfdeclarelayer{background}
    \pgfsetlayers{background,main}
    \node[maxNode] (s1x)  at (0, 0)    {\tiny $s_1$};
    \node[maxNode] (s2x)  at (10, 0)  {\tiny $s_2$};
    \node[botNode] (dt1) at (5, 2)  {\tiny $d_{t_1}$};
    \node[botNode] (dt2) at (5,-2)  {\tiny $d_{t_2}$};
    \node[stochNode] (ja)  at (2.5, 1.25) {};
    \node[stochNode] (jb)  at (2.5,-1.25) {};
    \node[stochNode] (2ja) at (7.5, 0)    {};
    \node[stochNode] (1jx) at (7.5, 2.5)  {};
    \node[stochNode] (2jx) at (5,-5) {};
    \node[stochNode] (2jy) at (5,-5) {};
    \begin{pgfonlayer}{background}
      \node[uncSet,fit=(ja)]   (uja)  {};
      \node[uncSet,fit=(jb)]   (ujb)  {};
      \node[uncSet,fit=(2ja)]  (u2ja) {};
      \node[uncSet,fit=(1jx)]  (u1jx) {};
      \node[uncSet,fit=(2jx)(2jy)] (u2jxy) {};
      \node[right=0.5pt of u2jxy.south,align=center]{\tiny$\mathcal{P}(d_{t_2},\anew)$};
      \node[right=0.5pt of u1jx.east, align=right]{\tiny$\mathcal{P}(d_{t_1},\anew)$};
      \node[above=2pt of ja.north, align=center]{\tiny$\mathcal{P}(s_1,a)$};
      \node[below=2pt of jb.south, align=center]{\tiny$\mathcal{P}(s_1,b)$};
      \node[left=1pt of 2ja.west, align=center]{\tiny$\mathcal{P}(s_2,a)$};
    \end{pgfonlayer}
    \draw[edge] (-1.8,0) -- (s1x);
    \draw[edge] (s1x) -- (uja)  node[pos=0.5,above left, inner sep=0.5pt]  {\tiny $a$};
    \draw[edge] (s1x) -- (ujb)  node[pos=0.5,below left, inner sep=0.5pt]  {\tiny $b$};
    \draw[edge] (s2x) -- (u2ja) node[midway,above]       {\tiny $a$};
    \draw[botEdge] (dt1) -- (u1jx)  node[midway,below] {\tiny $\anew$};
    \draw[botEdge] (dt2) -- (u2jxy) node[midway,left]  {\tiny $\anew$};
    \draw[edge] (ja)  -- (dt1); \draw[edge] (ja)  -- (dt2);
    \draw[edge] (jb)  -- (dt1); \draw[edge] (jb)  -- (dt2);
    \draw[edge] (2ja) -- (dt1); \draw[edge] (2ja) -- (dt2);
    \draw[edge] (1jx)  -- (s2x)  node[midway,above right] {};
    \draw[edge] (1jx) to[in = 100, out = 75] (s1);
    \draw[edge] (2jx) to[out=180,in=-90] node[midway,below left] {} (s1x);
    \draw[edge] (2jy) to[out=0,in=-90]   node[midway,below right]{} (s2x);
\end{tikzpicture}
\end{subfigure}
\end{adjustwidth}
\caption{\small (a) A POSG $\sg$.  $\pone$-states $s_1,s_2$ are circles, while  $\ptwo$-states $t_1,t_2$ are  rectangles. (b) The Pre-$\ptwo$ transformed game $\mathcal{G}'$ with dashed circles representing new dummy states- $d_{t_1},d_{t_2}$ and dashed arrows represent the $\anew$-action. (c) The reduced RPOMDP $\mathfrak{R}(\mathcal{G})$ with red dotted circles representing uncertainty sets $P(s,a)$. $P(s_1,a)$ is a singleton consisting of the distribution $\delta(s_1,a)$, likewise for $P(s_1,b)$ and $P(s_2,a)$. $P(d_{t_1}, \top)$ is also a singleton $\{\delta'({t_1},x)\}$.  $P(d_{t_2}, \top)$ is the polytope with vertex set $\{\delta'({t_2},x), \delta'({t_2},y)\}$. }
\label{fig:converse-main}
\end{figure*}
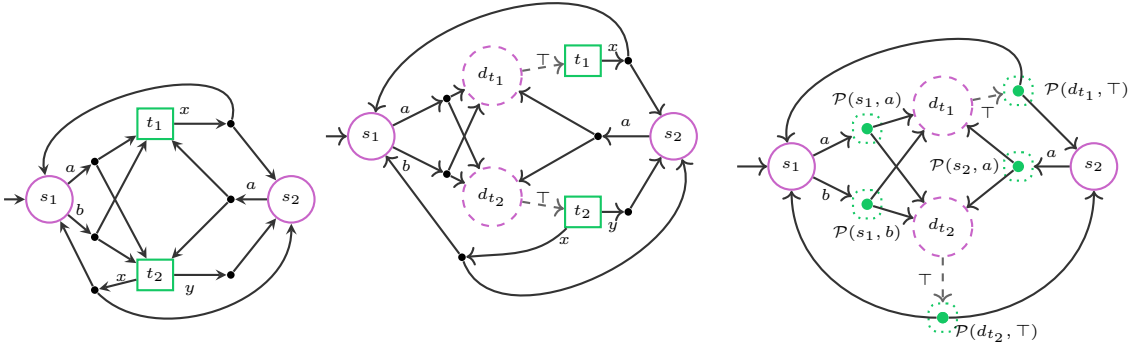

\subsubsection{Step 2: Reduction from Pre-min-transformed POSGs to RPOMDPs}
\label{sec:sg2rmdp}
We translate a
Pre-min-transformed POSG $\sg'$ into a value-equivalent RPOMDP $\rsg$
by {\em eliminating the $\ptwo$-player's states} and folding the
$\ptwo$-player's choices into polytope choices for the
RPOMDP environment. In $\sg'$, every play has the periodic structure
$
  s_1 \xrightarrow{a_1} d_{s_2} \xrightarrow{\anew}
  s_2 \xrightarrow{a_2} s_3 \xrightarrow{a_3}
  d_{s_4} \xrightarrow{\anew} s_4 \cdots,
$ 
in which each round has the order $\sone,D,\stwo$. The
key observation is that the $\ptwo$-player's only role is to pick an
action $a \in \atwo(s)$ at each $\stwo$-state $s$, and this choice
deterministically fixes a transition distribution $\tran'(s,a) \in
\dist(\sone)$. This is precisely the role of the environment in an
RPOMDP: given a state and an action, it selects a distribution from
the uncertainty set. We exploit this correspondence by collapsing
each $\stwo$-state $s$ into the dummy state $d_s$ that immediately
precedes it: $\uncertsa(d_s,\anew)$ is the convex hull of the distributions $\tran'(s,a)$ for $a\in\atwo(s)$. Picking a point in this polytope is a convex
combination $\sum_a \alpha_a\, \tran'(s,a)$, where $\alpha_a$ is the distribution value induced by $\ptwo-$action $a$ at $s$. 
The RPOMDP environment at $(d_s, \anew)$
therefore has exactly the choices available to the $\ptwo$-player at
$s$ in $\sg'$. The following construction formalizes this intuition.

\noindent{\bf RPOMDP construction.} Given a Pre-min-transformed POSG
$\sg' = (\sone \cup D, \stwo, \aone \cup \{\anew\}, \atwo, \tran',
s_0, \obsone \cup \Obsnewone, \obstwo \cup \Obsnewtwo, \obfnone',
\obfntwo')$, the corresponding RPOMDP $\rsg = (S_\rp, A_\rp,
\uncert_\rp, \obsa^\rp, \obse^\rp, \obfna^\rp, \obfne^\rp,
s_0^\rp)$ is defined as follows:
\begin{compactitem}
 \item  {\em States.} $S_\rp = \sone \cup D$, with initial state
        $s_0^\rp = s_0$.
  \item   {\em Actions.} $A_\rp=\aone\cup\{\anew\}$, with effective alphabets $A_\rp(s)=\aone(s)$ at original max-states and $A_\rp(d_s)=\{\anew\}$.
  \item {\em Uncertainty set.} $\uncert_\rp = \prod_{(s,a) \in S_\rp
        \times A_\rp} \uncertsa(s,a)$. The effective polytopes are generated by $V_{s,a}=\{\tran'(s,a)\}$ for $s\in\sone$, $a\in\aone(s)$, and $V_{d_{s'},\anew}=\{\tran'(s',a'):a'\in\atwo(s')\}$. 

  \item {\em Observations.} $\obsa^\rp = \obsone \cup \Obsnewone$
        and $\obse^\rp = \obstwo \cup \Obsnewtwo$, with
        $\obfna^\rp(s') = \obfnone'(s')$ and $\obfne^\rp(s') =
        \obfntwo'(s')$ for all $s' \in S_\rp$.
\end{compactitem}

\noindent{\bf RPOMDP objectives.} We first relate plays of $\sg'$ and
$\rsg$. For $\rho'{\in} \Plays_{\sg'}$, $\Lambda : \Plays_{\sg'} {\to} \Plays_{\rsg}$ 
removes every $\stwo$-state and every
$\atwo$-action from $\rho'$.

  Unlike $\Gamma$ in
\Cref{sec:step1}, the map $\Lambda$ is \emph{many-to-one}:
$\sg'$-plays that agree on $\sone \cup D$ states and $\aone \cup
\{\anew\}$ actions but differ in the $\atwo$-action chosen at
$\stwo$-states map to the same $\rsg$-play. Its set-valued inverse
$\Lambda^{-1}(\rho^\rp)$ recovers the pre-image by re-introducing an
arbitrary $\atwo$-action consistent with the realised next state.
Observation histories lift via $\LambdaOpone, \LambdaOptwo$:
agent map $\LambdaOpone$ is a bijection between $\OH_\pone^{\sg'}$
and $\OH_\pa^{\rsg}$, while environment map $\LambdaOptwo$ is a
bijection onto 
environment histories ending at a dummy state.

Strategies in $\sg'$ and $\rsg$ are functions of observation
histories. We lift strategies between the two models as follows.
$\Omega_\pone : \stratpone_{\sg'} \to \Sigma_{\rsg}$ converts a
$\sg'$-$\pone$-strategy $\sigma'$ into an agent strategy. Given an
agent observation history $oh^\rp$, the new strategy first lifts
$oh^\rp$ to its $\sg'$-counterpart $(\LambdaOpone)^{-1}(oh^\rp)$ and
queries $\sigma'$ there:
$\Omega_\pone(\sigma')(oh^\rp)=
\sigma'\!\left((\LambdaOpone)^{-1}(oh^\rp)\right)$.
$\Omega_\ptwo : \stratptwo_{\sg'} \to \Pi_{\rsg}$ converts a
$\sg'$-$\ptwo$-strategy $\pi'$ into an environment strategy. At a
dummy state $d_{s'}$ under $\anew$, with environment history
$oh^\rp$ ending in $\Obsnewtwo$, write
$\pi'\!\left((\LambdaOptwo)^{-1}(oh^\rp)\right)(s')(a) = \alpha_a$
for the randomised $\ptwo$-action $a$ at $s'$; then $\Omega_\ptwo(\pi')$
selects the convex combination $\sum_{a \in \atwo(s')} \alpha_a \cdot
\tran'(s',a) \in \uncertsa(d_{s'}, \anew)$. At non-dummy
state-action pairs, the polytope is a singleton, and
$\Omega_\ptwo(\pi')$ picks its unique element. The reverse maps
$\Xi_\pone : \Sigma_{\rsg} \to \stratpone_{\sg'}$ and $\Xi_\ptwo :
\Pi_{\rsg} \to \stratptwo_{\sg'}$ are defined analogously:
$\Xi_\pone(\sigma^\rp)(oh') = \sigma^\rp(\LambdaOpone(oh'))$; the
environment map uses a fixed choice of convex coefficients over the effective min-actions. For $\star \in \{\pone,\ptwo\}$, $\Omega_{\star},\Xi_{\star}$ are inverses of each other by the observation-history correspondence. The formal maps and proofs appear in Appendix~\ref{app:map3}.

  \noindent $\bullet$ \textbf{Parity objective.} Let $\prfn'$ be the
  priority function for $\sg'$. Define  the priority function for $\rsg$, $\prfn^\rp : S_\rp \to
  \{0, 1, \dots, d\}$ by $\prfn^\rp(s') = \prfn'(s')$ for all $s'
  \in S_\rp$.

\noindent{\bf{Size and memory preservation under RPOMDP construction}}. It is easy to see that  $|\rsg| = \mathcal{O}(|\sg'|)$ and the
        construction runs in linear time. Moreover, observation
        histories on the two sides are linear in each other. The lifting maps
         $\Omega_\star, \Xi_\star$ preserve memory class exactly: 
        memoryless, finite-memory, and infinite-memory regimes correspond exactly
        between $\sg'$ and $\rsg$ (Appendix~\ref{app:last-size-memory}).

\begin{lemma}[Equivalence between $\sg'$ and $\rsg$, Proof in Appendix \ref{app:thm3}]
Let $\sg'$ be a Pre-min-transformed POSG and $\rsg$ its RPOMDP
reduction. The play map $\Lambda$
and the strategy maps $\Omega_\star, \Xi_\star$
($\star \in \{\pone, \ptwo\}$) 
satisfy the following: for every priority function $\prfn'$ on $\sg'$, every $\omega$-regular objective
$W'$ in $\sg'$ with the lifted objective  $W^\rp$ in $\rsg$ 
and for every $\rho' \in \Plays_{\sg'}$, 
$\sigma' \in \stratpone_{\sg'}$  and $\pi' \in \stratptwo_{\sg'}$: \\
\textup{(I, obj.)} $\rho' \in W' \!\iff\! \Lambda(\rho') \in W^\rp;$ $\quad\quad\quad\quad\,\,\,\,\,$ 
\textup{(II, plays)} $\Lambda(\Plays_{\sg'}^{\sigma',\pi'}) = \Plays_{\rsg}^{\Omega_\pone(\sigma'),\Omega_\ptwo(\pi')}$;\\
\textup{(III, prob.)} $\probm_{\sg'}^{\sigma',\pi'}(W') = \probm_{\rsg}^{\Omega_\pone(\sigma'),\Omega_\ptwo(\pi')}(W^\rp)$; 
\textup{(IV, val.)} $\val_{\sg'}(W') = \val_{\rsg}(W^\rp)$\\
Consequently, sure, almost-sure, limit-sure, and quantitative
analysis of any objective defined above on $\sg'$ is inter-reducible in
linear time with the corresponding analysis on $\rsg$.
\label{thm:sg2rmdp-main} 
\end{lemma}
\noindent\textbf{Proof outline.}
(I) $\Lambda$ leaves the underlying $\sone\cup D$-state sequence,  and hence the positions of target-state visits unchanged, and $\prfn^\rp$ is just $\prfn'$ restricted to $S_\rp$; neither change affects the $\liminf$ the parity condition depends on.
(II) Following $\rho'\in\Plays_{\sg'}^{\sigma',\pi'}$ step by step, $\Lambda(\rho')$ is producible by $(\Omega_\pone(\sigma'),\Omega_\ptwo(\pi'))$ in $\rsg$: at each $\sone$-state the singleton uncertainty set forces the same transition, and at each dummy state the convex combination $\sum_a\alpha_a\tran'(s,a)$ chosen by $\Omega_\ptwo(\pi')$ is exactly the distribution $\pi'$ induces at $s$. Conversely, every play in $\Plays_{\rsg}^{\Omega_\pone(\sigma'),\Omega_\ptwo(\pi')}$ has a $\Lambda$-preimage in $\sg'$, obtained by drawing a $\ptwo$-action at each dummy step according to the coefficients $\alpha_a$.
(III) A cylinder-set induction on $\sg'$-history, splitting each step into a max-action reaching a dummy state and the forced dummy step followed by the min-action reaching the next max-state; both cases reduce the $\rsg$-extension probability to the matching $\sg'$-probability. The converse follows the same way by induction on $\rsg$-history.
(IV) applies (III) in both directions to equate the inner infima for each fixed agent strategy, then takes the outer supremum.
\emph{Consequently:} quantitative and limit-sure analysis follow from~(IV); almost-sure follows from~(III), taking the infimum over $\pi'$ on each side and matching witnesses through $\Omega,\Xi$; sure analysis follows from~(II) and~(I), since a strategy keeping all $\sg'$-plays inside $W'$ corresponds via $\Omega$ to one keeping all $\rsg$-plays inside $W^\rp$, and conversely via $\Xi$.  From Lemmas \ref{thm:prenorm-main} and \ref{thm:sg2rmdp-main}, we conclude:

\begin{theorem} For an ac-POSG $\sg$ and its RPOMDP construction $\rsg$ (Section~\ref{sec:ac2rpomdp}), the sure, almost-sure, limit-sure, and quantitative analyses of any $\omega$-regular objective are inter-reducible in polynomial time between $\sg$ and $\rsg$.
\end{theorem}

\noindent{\bf Complexity Results.} 
The reductions in Sections~\ref{sec:forward} and~\ref{sec:ac2rpomdp} establish equivalence between $(s,a)$ rectangular polytopic RPOMDPs and ac-POSGs as defined in sections~\ref{sec:prelim} and~\ref{sec:posg}. Lemma~\ref{lem:prescription-compilation} compiles the ac-POSGs  in \cref{sec:posg}
 into ordinary ac-POSG \cite{DBLP:journals/fmsd/Chatterjee0H13}  with values and all four objectives preserved with polynomial overhead; together with \cref{sec:forward},  upper bounds in~\cite{DBLP:journals/fmsd/Chatterjee0H13} apply to our RPOMDPs.  
Lemma \ref{lem:ordinary-as-prescription} converts ordinary ac-POSGs into the class-constant subclass (the same distribution assigned to every min-state of an
observation class) of our ac-POSGs  at no cost with values and all four objectives preserved; together 
with \cref{sec:ac2rpomdp}, the lower bounds from 
 \cite{DBLP:journals/fmsd/Chatterjee0H13} apply to our RPOMDPs. Allowing full observability to  
 the  $\ptwo$-player, factored prescriptions coincide with ordinary min strategies of \cite{DBLP:journals/fmsd/Chatterjee0H13}. 
 In this case,  the reductions of  \ref{sec:forward} and~\ref{sec:ac2rpomdp}, 
 show the equivalence of 1s RPOMDPs and one sided ordinary ac-POSGs. Thanks to our unified treatment across observability regimes,  when both players have full observation, 
 we have the  equivalence between stochastic games and RMDPs, allowing  to lift complexity results 
 of \cite{DBLP:journals/fmsd/Chatterjee0H13} to 1s RPOMDPs and RMDPs. 
 
\noindent{\bf Concluding Remarks.} We studied RPOMDPs under $\omega$-regular objectives, and showed that solving them is equivalent to solving POSGs under $\omega$-regular objectives. We do this by establishing polynomial-time reductions
in both directions. Exploiting this equivalence, we derived several new complexity
results for both RPOMDPs and RMDPs, and provided a complete picture of the complexity landscape
of the problem. Interesting future work would be to consider lifting rectangularity and polytopic uncertainty
set assumptions, and studying complexity under these more general settings.

\bibliography{references}
\bibliographystyle{abbrv}

\clearpage
\appendix

\centerline{\Large{\bf Appendix}}
\smallskip 

\medskip

\textbf{Table of Contents} 
\begin{itemize}
\item Section \ref{app:prescription-compilation} Prescription compilation for ac-POSGs
\item Section \ref{app:rpomdp-posg} Conversion of RPOMDP $\rmdp$  to POSG $\sg$ 
\begin{itemize}
  \item Section \ref{app:map} Mapping of plays, observation histories and strategies from $\rmdp$ to $\sg$ 
 \item Section \ref{app:lift} Lifting objectives from $\rmdp$ to $\sg$
 \item Section \ref{app:value} Value Preservation between $\rmdp$ to $\sg$
 \item Section \ref{app:thm1} Proof of Theorem \ref{rpomdp-posg}
  \end{itemize}
    \item Section \ref{app:reverse} ac-POSG $\sg$  to  Pre-min transformed POSG $\sg'$
  \begin{itemize}
     \item Section \ref{app:map2} Mapping of plays, observation histories and strategies from $\sg$ to $\sg'$
    \item Section \ref{app:lift2} Lifting objectives from $\sg$ to $\sg'$
 \item Section \ref{app:value2} Value Preservation between $\sg$ to $\sg'$
\item Section \ref{app:thm2} Proof of Lemma \ref{thm:prenorm-main}
  \end{itemize}
       
\item Section \ref{app:sg2rmdp} Constructing an RPOMDP from a Pre-$\ptwo$-transformed POSG
\begin{itemize}
    \item  Section \ref{app:map3} Mapping of plays, observation histories and strategies from $\sg'$ to $\rsg$
    \item Section \ref{app:lift3} Lifting objectives from $\sg'$ to $\rsg$
 \item Section \ref{app:value3} Value Preservation between $\sg'$ to $\rsg$
\item Section \ref{app:thm3} Proof of Lemma \ref{thm:sg2rmdp-main}
  \end{itemize}
\end{itemize}

\clearpage 
As in the main text, all actions are formally enabled at all states. When a construction displays only effective actions, omitted actions are implicitly completed by an absorbing outcome losing for their owner. We suppress these dominated transitions and carry out the mappings on the displayed effective actions. For the max-player, assigning positive probability to an omitted action can only decrease the probability of satisfying the objective, so that probability can be reassigned to an effective action. Dually, an omitted min-action reaches a max-winning outcome and therefore cannot improve the min-player's strategy. Thus both players may be restricted to effective actions without changing the value or the sure, almost-sure, limit-sure, and quantitative winning sets. This argument applies because all objectives considered here are represented as parity objectives and the completion outcomes are absorbing wins or losses; an arbitrary action-sensitive objective would first need to be encoded in the state space. 

Henceforth, all the strategies in the below proofs are {\em effective strategies}: they are defined only on effective actions.

\section{Prescription Compilation for ac-POSGs}
\label{app:prescription-compilation}

For $o\in\obstwo$, let observation class $\stwo(o)=\{t\in\stwo:\obfntwo(t)=o\}$ and fix an enumeration $\stwo(o)=\{t^o_1,\dots,t^o_{k_o}\}$. Let $k=\max_o k_o\geq1$ and write $j(t)$ for the index of $t$ in its observation class i.e., $t = t^{\obfntwo(t)}_{j(t)} \in \stwo(\obfntwo(t))$. The nonempty effective actions are denoted by $\aone(s)$ and $\atwo(t)$ for $s \in \sone, t\in \stwo$; they equal the global alphabet when no restriction is specified.

To compile factored prescriptions into ordinary ones, we introduce a gadget which can be plugged in at each min state. 
These gadgets have $k$ selection stages, even when $|\stwo(o)| < k$. The common length prevents the max-player from learning the min-player's observation class from the duration of a gadget.

Given an ac-POSG $\sg$ in which the min-player uses factored prescriptions rather than ordinary strategies, we compile $\sg$ into an ac-POSG $\sg^c$ whose min-player strategies are ordinary, in the sense \cite{DBLP:journals/fmsd/Chatterjee0H13}. Intuitively, $\sg^c$ simulates a single prescription choice at a min-state $t\in\stwo$ by unfolding it into $k$ sequential selection stages, one per min-state, so an ordinary strategy can supply the $k$ local actions one at a time sequentially, instead of all at once, as a prescription does. Between consecutive selection stages, a forced max move (with a fresh action $\anew$) preserves strict alternation between the two players without making any real decision. Once all $k$ stages have been visited, the accumulated choices are handed off through an exit state $D(t,\beta)$, via a fixed reference action $b_0$, into the transition $\delta(t,\beta)$ of the original state $t$ that the assembled prescription selected. The construction below makes this precise.

\begin{definition}[Compiled ac-POSG]
\label{def:compiled-posg}
Given an ac-POSG $\sg=(\sone,\stwo,\aone,\atwo,\delta,s_0,\obsone,\obstwo,\obfnone,\obfntwo)$
with priority function $p^\sg$, define the ac-POSG
$\sg^c=(\sone^c,\stwo^c,\aone^c,\atwo^c,\delta^c,s_0,\obsone^c,\obstwo^c,\obfnone^c,\obfntwo^c)$
as follows.
\begin{itemize}
    \item {\em States.}Define
    \[
    U_{t,i}=\begin{cases}\{\bot\}&i\leq j(t),\\\atwo(t)&i>j(t),\end{cases}
    \qquad
    V_{t,i}=\begin{cases}\{\bot\}&i<j(t),\\\atwo(t)&i\geq j(t).\end{cases}
    \]
    where $\bot$ is a fresh symbol. The selection states are $\stwo^{\mathrm{sel}}=\{(t,i,\beta):t\in\stwo,\ 1\leq i\leq k,\ \beta\in U_{t,i}\}$, and the exit states are $E=\{D(t,\beta):t\in\stwo,\ \beta\in\atwo(t)\}$. Set
    $\stwo^c=\stwo^{\mathrm{sel}}\cup E$ and
    $\sone^c=\sone\cup M$ where $M = \{F(t,i,\beta):t\in\stwo,\ 1\leq i\leq k,\ \beta\in V_{t,i}\}$. $\beta=\bot$ before stage $j(t)$ and contains an effective action afterwards.
    \item {\em Actions.} Choose distinct fresh actions $b_0,\anew\notin\aone\cup\atwo$ and set
    \[
    \atwo^c=\atwo\cup\{b_0\},\qquad \aone^c=\aone\cup\{\anew\}.
    \]
    $\sone$ retain $\aone(s)$, and max-states in $M$ have sole effective action $\anew$. At a selection state of class $o$, the effective actions are given by
    \[
    B_{o,i}=\begin{cases}\atwo(t^o_i)&1\leq i\leq k_o,\\\{b_0\}&k_o<i\leq k.\end{cases}
    \]
    Every exit state has sole effective action $b_0$. The fresh action never belongs to a source min-prescription and is never recorded in the register $\beta$.
    \item {\em Transitions.} Write $t=t^o_j$ where $o=\obfntwo(t)$ and $1\le j\le k_o$. For
    $s\in\sone,a\in\aone(s)$: $\delta^c(s,a)((t,1,\bot)):=\delta(s,a)(t)$. For
    $(t,i,\beta)\in\stwo^{\mathrm{sel}}$ and $b\in B_{o,i}$,
    \[
    \delta^c\big((t,i,\beta),b\big) := \dirac{F(t,i,\beta')}, \qquad
    \beta' = \begin{cases} b & \text{if } i=j \\ \beta & \text{if } i\ne j.\end{cases}
    \]
    For $F(t,i,\beta)\in\sone^c$ with $i<k$,
    $\delta^c\big(F(t,i,\beta),\anew\big) := \dirac{(t,i+1,\beta)}$; for the 
    state $F(t,k,\beta)$ (where $\beta\in\atwo(t)$),
    \[
    \delta^c\big(F(t,k,\beta),\anew\big) := \dirac{D(t,\beta)}.
    \]
    For the exit state $D(t,\beta)$, the fresh action $b_0$ is the only
    effective action, and $\delta^c\big(D(t,\beta),b_0\big) := \delta(t,\beta)$.
    \item {\em Observations.} For $\star\in\{\pone,\ptwo\}$:
    \[
    \obsstar^c=\obsstar\uplus\{(o,i):o\in O_\star,\ 1\leq i\leq k+1\}.
    \]
    Define $\obfnstar^c(s)=\obfnstar(s)$ for $s\in\sone$ and
    \[
    \obfnstar^c((t,i,\beta))=\obfnstar^c(F(t,i,\beta))=(\obfnstar(t),i),\qquad
    \obfnstar^c(D(t,\beta))=(\obfnstar(t),k+1).
    \]
    Thus each player retains its original observation of $t$ and learns only the public stage index, not $t$, $\beta$, or the other player's observation.
    \item {\em Priorities.} $p^c(s)=p^\sg(s)$ for $s\in\sone$, and
    $p^c\big((t,i,\beta)\big)=p^c\big(F(t,i,\beta)\big)=p^c\big(D(t,\beta)\big):=p^\sg(t)$.
\end{itemize}
Only the transitions listed above are effective. Every effective transition alternates owners. 
\end{definition}

Every min-state $t$ is unfolded into a chain of $k$ min steps,
alternating with max-states in $M$ carrying the dummy action $\anew$.
At stage $i$, the min-player declares what it would play if the true state were $t^o_i$;
only the declaration made at the stage matching the true state is recorded (in $\beta$).
After the $k_o$ component choices, $k-k_o$ forced padding choices use $b_0$. The exit $D(t,\beta)$ then executes $\delta(t,\beta)$ on $b_0$. The fresh min-action thus carries no additional decision. Both compiled players use standard observation-based action distributions, with the effective-action support convention stated below.
Figure~\ref{fig:prescription-compilation} illustrates the gadget for a class of size two, padded to $k=3$.

\begin{figure}[h!]
\centering
\begin{tikzpicture}[
    >={Stealth[scale=1.1]},
    maxstate/.style={circle, draw=maxcolor, thick, minimum size=0.85cm},
    minstate/.style={rectangle, draw=mincolor, thick, minimum size=0.6cm, inner sep=3pt},
    forced/.style={circle, draw=maxcolor, thick, densely dotted, minimum size=0.85cm},
    dummy/.style={rectangle, draw=mincolor, thick, densely dotted, minimum size=0.6cm, inner sep=3pt},
    edge/.style={thick,->,>=Stealth},
    every node/.style={font=\small}
]
\node at (0,5.6) {\textbf{Before: prescription over $o$}};
\node[maxstate] (w) at (0,4.4) {$w$};

\node[minstate] (t1) at (-1.8,2.8) {$t^o_1$};
\node[maxstate] (p) at (-3.0,1.2) {$p$};
\node[maxstate] (q) at (-0.6,1.2) {$q$};
\draw[edge] (t1) -- node[left]{$a$} (p);
\draw[edge] (t1) -- node[right]{$a'$} (q);

\node[minstate] (t2) at (1.8,2.8) {$t^o_2=t$};
\node[maxstate] (u1) at (0.6,1.2) {$u$};
\node[maxstate] (v1) at (3.0,1.2) {$v$};
\draw[edge] (t2) -- node[left]{$L$} (u1);
\draw[edge] (t2) -- node[right]{$R$} (v1);

\draw[edge] (w) -- node[left,xshift=-1pt]{$c_1$} (t1);
\draw[edge] (w) -- node[right,xshift=1pt]{$c_2$} (t2);

\node[align=center, font=\footnotesize] at (0,-0.2) {$\pi(oh)=\big(\pi(oh)(t^o_1),\,\pi(oh)(t^o_2)\big)\in\dist(\{a,a'\})\times\dist(\{L,R\})$};

\begin{scope}[xshift=8.5cm]
\node at (0,9.4) {\textbf{After: compiled gadget for $t$ ($k=3$)}};
\node[maxstate] (wr) at (0,8.2) {$w$};
\node[minstate, draw=mincolor!50] (t1stub) at (-2.6,6.6) {$(t^o_1,1,\bot)$};
\node[font=\scriptsize] at (-2.6,5.9) {$\cdots$};
\draw[edge] (wr) -- node[left]{$c_1$} (t1stub);

\node[minstate] (s1) at (0,6.6) {$(t,1,\bot)$};
\draw[edge] (wr) -- node[right]{$c_2$} (s1);
\node[forced] (f1) at (0,5.2) {$F_1$};
\node[minstate] (s2) at (0,3.8) {$(t,2,\bot)$};
\node[forced,fill=maxcolor!12] (f2l) at (-2.4,2.2) {$F_{2,L}$};
\node[forced,fill=maxcolor!12] (f2r) at (2.4,2.2) {$F_{2,R}$};
\node[minstate,fill=mincolor!12] (s3l) at (-2.4,0.6) {$(t,3,L)$};
\node[minstate,fill=mincolor!12] (s3r) at (2.4,0.6) {$(t,3,R)$};
\node[forced,fill=maxcolor!12] (f3l) at (-2.4,-1.0) {$F_{3,L}$};
\node[forced,fill=maxcolor!12] (f3r) at (2.4,-1.0) {$F_{3,R}$};
\node[dummy] (dl) at (-2.4,-2.6) {$D_L$};
\node[dummy] (dr) at (2.4,-2.6) {$D_R$};
\node[maxstate] (u2) at (-2.4,-4.4) {$u$};
\node[maxstate] (v2) at (2.4,-4.4) {$v$};
\draw[edge] (s1) -- node[right,xshift=1pt]{\shortstack[l]{$b\in\{a,a'\}$\\[-1pt]\tiny(discarded)}} (f1);
\draw[edge] (f1) -- node[right]{$\anew$} (s2);
\draw[edge] (s2) -- node[above,sloped]{$L$} (f2l);
\draw[edge] (s2) -- node[above,sloped]{$R$} (f2r);
\draw[edge] (f2l) -- node[left]{$\anew$} (s3l);
\draw[edge] (f2r) -- node[right]{$\anew$} (s3r);
\draw[edge] (s3l) -- node[left,xshift=-1pt]{\shortstack[l]{$b_0$\\[-1pt]\tiny(forced)}} (f3l);
\draw[edge] (s3r) -- node[right,xshift=1pt]{\shortstack[l]{$b_0$\\[-1pt]\tiny(forced)}} (f3r);
\draw[edge] (f3l) -- node[left]{$\anew$} (dl);
\draw[edge] (f3r) -- node[right]{$\anew$} (dr);
\draw[edge] (dl) -- node[left]{$b_0$} (u2);
\draw[edge] (dr) -- node[right]{$b_0$} (v2);
\end{scope}
\end{tikzpicture}
\caption{The compiled gadget for $t=t^o_2$ when $\stwo(o)=\{t^o_1,t^o_2\}$, so $k_o=2$, drawn with $k=3$: elsewhere in the same game, a different observation class $o'$ has $k_{o'}=3$ states, and since every gadget is padded to the global $k=\max_{o''}k_{o''}$, this class's gadget must also run for $3$ stages even though it has only $2$ states of its own. On the left, both states of $o$ are reachable from a common max-state $w$, via distinct actions $c_1$ and $c_2$ each leading deterministically to $t^o_1$ and $t^o_2$ respectively. On the right, $w$'s transitions are unchanged: $c_1$ and $c_2$ still lead to $t^o_1$ and $t^o_2$, but each now enters their respective gadgets rather than the original state directly -- $(t^o_1,1,\bot)$ is shown only as a stub. Following the $c_2$ branch, the three stages are consumed one at a time: Stage~1 (state $F_1=F(t,1,\bot)$) chooses $b\in\{a,a'\}$, discarded on the displayed branch since the play never visits $t^o_1$; Stage~2 records $b\in\{L,R\}$, reaching $F_{2,L}=F(t,2,L)$ or $F_{2,R}=F(t,2,R)$; Stage~3 has no state of $o$ left to represent, so it is a forced padding stage -- the only effective action there is $b_0$ -- reaching $F_{3,L}=F(t,3,L)$ or $F_{3,R}=F(t,3,R)$, and finally the exits $D_L=D(t,L)$ or $D_R=D(t,R)$. The fresh min-action $b_0\notin\{L,R\}$ is forced at both the padding stage and either exit. States differing only in the stored action have identical observations.}
\label{fig:prescription-compilation}
\end{figure}
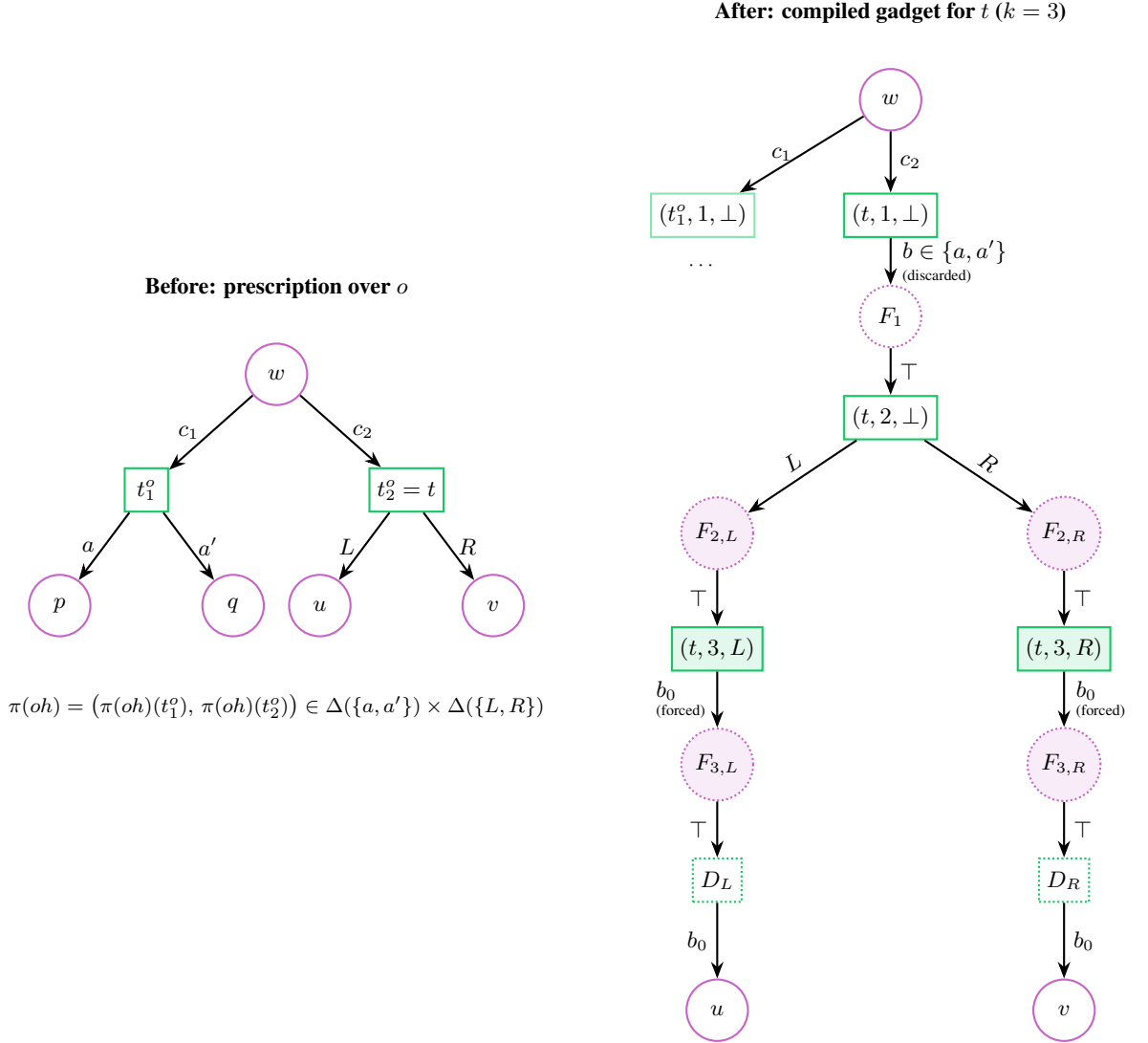

\begin{lemma}[Size and transitions of $\sg^c$]
\label{lem:compiled-size}
Put $n=|\sone|$, $m=|\stwo|$, $a=|\aone|$, $b=|\atwo|$, and $\ell_t=|\atwo(t)|$ for $t\in\stwo$. Let $r_1=\sum_{s\in\sone}|\aone(s)|$ and $r_2=\sum_{t\in\stwo}\ell_t$ count the effective state--action rows in $\sg$, and let $e_1,e_2$ count their positive-probability edges in $\sg$.

For $t\in\stwo$ with stage index $j=j(t)$: a selection state $(t,i,\beta)$ still has $\beta=\bot$ only, before $\beta$ is recorded ($i\le j$, one choice); from stage $j$ on, $\beta$ ranges over $\atwo(t)$ ($i>j$, $\ell_t$ choices). The forced state $F(t,i,\beta)$ that follows it in the same stage is one step ahead: it sees $\beta=\bot$ only strictly before stage $j$ ($i<j$), and the recorded value already from stage $j$ on ($i\ge j$). Summing over all $k$ stages gives the exact state counts
\[
C=\sum_{t\in\stwo}\bigl(j(t)+(k-j(t))\ell_t\bigr)=|\stwo^{\mathrm{sel}}|,\qquad
H=\sum_{t\in\stwo}\bigl(j(t)-1+(k-j(t)+1)\ell_t\bigr)=|M|.
\]
Each selection state has one outgoing transition per effective action, giving the total
\[
R=\sum_{t\in\stwo}\sum_{i=1}^k |U_{t,i}|\,|B_{\obfntwo(t),i}|
\]
of transitions leaving selection states.

For the construction in Definition~\ref{def:compiled-posg}, the exact counts are
\begin{align*}
|\stwo^c|&=C+r_2,& |\sone^c|&=n+H,\\
|\atwo^c|&=b+1,& |\aone^c|&=a+1,\\
r^c&=r_1+R+H+r_2,& e^c&=e_1+R+H+e_2,\\
|\obsstar^c|&=(k+2)|O_\star| \qquad(\star\in\{\pone,\ptwo\}).
\end{align*}
Here $r^c$ counts effective transition distributions and $e^c$ counts positive triples $(s,a,s')$, not merely state pairs. In particular,
\[
|S^c|=O(n+km(b+1)),\qquad
r^c=O(r_1+kmb^2),\qquad
e^c=O(e_1+e_2+kmb^2).
\]
The construction is polynomial in the explicit input size. When $k=1$, its size is linear in the source size; alternatively the identity game can be used because a min-prescription already has a single component.

\begin{proof}
The state counts $C=|\stwo^{\mathrm{sel}}|$ and $H=|M|$ follow directly from the case split above; adding the $r_2$ exit states (one per $t$ and $\beta\in\atwo(t)$) gives $|\stwo^c|=C+r_2$, and $|\sone^c|=n+H$.

For transitions: the rows inherited unchanged from $\sg$, $\delta^c(s,a)((t,1,\bot)):=\delta(s,a)(t)$ for $s\in\sone,a\in\aone(s)$, number exactly $r_1$ with $e_1$ edges. Selection states contribute the $R$ transitions counted above. Every forced state has exactly one outgoing transition, giving $H$ rows and edges. Each exit state $D(t,\beta)$ has the single effective action $b_0$ and copies the source min-transition $\delta(t,\beta)$ exactly, giving $r_2$ rows and $e_2$ edges. Summing gives $r^c$ and $e^c$.

The observation count follows from the disjoint tagged alphabets: $\obsstar^c=\obsstar\uplus\{(o,i):o\in O_\star,1\le i\le k+1\}$ has $|O_\star|+(k+1)|O_\star|=(k+2)|O_\star|$ elements.

Since $|U_{t,i}|,|V_{t,i}|,|B_{o,i}|\le b$, the polynomial bounds follow by bounding each sum by its number of terms times $b$ (or $b^2$ for $R$). Probabilities are copied or equal to $1$, with no increase in rational bit length. For $k=1$, every $t$ has $j(t)=1$, so $C=m$, $H=R=r_2$, and the formulas give a linear-size expansion. 
\end{proof}
\end{lemma}

\begin{lemma}[Prescription Compilation]
\label{lem:prescription-compilation}
Let $\sg$ be an ac-POSG in which the max-player uses ordinary observation-based strategies and the min-player uses factored prescriptions, and let $\sg^c$ be its compiled game (Definition~\ref{def:compiled-posg}). Then $\sg^c$ is an ac-POSG in which both players use ordinary observation-based strategies. Each player's strategies admit mutually inverse maps $\Phi_\star,\Psi_\star$ ($\star\in\{\pone,\ptwo\}$), such that for every source pair $(\sigma,\pi)$ and compiled pair $(\sigma^c,\pi^c)$,
\[
\probm_\sg^{\sigma,\pi}(W)=\probm_{\sg^c}^{\Phi_\pone(\sigma),\Phi_\ptwo(\pi)}(W^c),
\qquad
\probm_{\sg^c}^{\sigma^c,\pi^c}(W^c)=\probm_\sg^{\Psi_\pone(\sigma^c),\Psi_\ptwo(\pi^c)}(W),
\]
where $W$ is a parity objective and $W^c$ is its lift (priorities as in Definition~\ref{def:compiled-posg}). 
Moreover, a play can happen under $(\sigma,\pi)$ in $\sg$ if and only if a corresponding play can happen under the lifted strategies in $\sg^c$. Sure winning needs exactly this: it depends on which plays are possible, not on their probabilities. Hence sure, almost-sure, limit-sure, and quantitative analysis of $W$ in $\sg$ reduce in polynomial time to the same analyses of $W^c$ in $\sg^c$, and back, at the construction size of Lemma~\ref{lem:compiled-size}.
Memoryless, finite-memory, and infinite-memory strategies correspond exactly, with no increase in the number of memory states. If the min-player already has full observation ($k=1$), $\sg=\sg^c$.

\begin{proof}
\emph{Objective preservation.} Every visit to $t$ stretches into $2k+1$ gadget states, all carrying priority $p^\sg(t)$ by construction. A parity objective depends only on which priorities occur infinitely often along a play. If $\rho^c$ is obtained from $\rho$ by locally repeating each visited priority some bounded number of times (here at most $2k+1$), the set of priorities occurring infinitely often is unchanged: a priority visited infinitely often in $\rho$ is still visited infinitely often in $\rho^c$, since each visit contributes at least one copy, and a priority visited only finitely often in $\rho$ still contributes only finitely many copies in $\rho^c$, since each of its finitely many visits contributes at most $2k+1$ copies. So a play $\rho$ satisfies $W$ in $\sg$ exactly when its counterpart $\rho^c$ satisfies $W^c$ in $\sg^c$.

\emph{Plays.} A single visit to a min-state $t=t^o_j$ in a play of $\sg$ unfolds, in $\sg^c$, into a full gadget traversal: $k$ selection stages interleaved with forced max-states, then the exit $D(t,\beta)$. Only the action played at stage $j$ matters; every other stage's action is discarded and can be anything. So one $\sg$-play corresponds to many $\sg^c$-plays, and there is no useful function between them at the level of full plays. We work with observation histories instead, where the correspondence is clean.

\emph{Observation histories.} Every visited state, min-owned or max-owned, contributes an observation to each player's history (Section~\ref{sec:posg}), so one gadget contributes \emph{two} tagged observations per stage: one from the selection state $(t,i,\beta)$, one from the forced state $F(t,i,\beta)$ right after it, both tagged $(\obfnstar(t),i)$ by construction. Write
\[
\mathcal B(o)=(o,1),(o,1),(o,2),(o,2),\ldots,(o,k),(o,k),(o,k+1)
\]
for a complete gadget's tag sequence, ending in the lone $(o,k+1)$ since the exit $D(t,\beta)$ has no forced state after it.

A min-observation history in $\sg^c$ is a sequence of complete blocks $\mathcal B(o)$ for finished gadgets, followed optionally by a partial block recording the stage reached in the current one. Let $\Theta_\ptwo(oh,i)$ build this partial history for a source history $oh$ ending at $o=\Last(oh)$: expand every earlier observation to a complete block, and expand $o$ itself only up to stage $i$ ($1\le i\le k+1$; $i=k+1$ means the block is complete). Let $\Upsilon_\ptwo$ collapse any block, complete or partial, back to the single observation it belongs to. These are mutual inverses,
\[
\Upsilon_\ptwo(\Theta_\ptwo(oh,i))=oh,\qquad
\Theta_\ptwo(\Upsilon_\ptwo(oh^c),i)=oh^c,
\qquad\text{where }\Last(oh^c)=(o,i)\text{ for some }o\in\obstwo.
\]
For the first identity: $\Theta_\ptwo(oh,i)$ is, by construction, one block per observation of $oh$, in the same order -- complete blocks for every observation but the last, a partial block for the last -- so collapsing each block back to the single observation it belongs to just reads off $oh$ again. For the second: every reachable min-observation history of $\sg^c$ really does have this shape, since a gadget is entered only at its first tag $(o,1)$ and left only once its exit tag $(o,k{+}1)$ is reached, so no history can end mid-gadget in any other way; consequently $\Upsilon_\ptwo(oh^c)$ recovers the same observation sequence $oh$ that produced $oh^c$, and the stage $i$ is exactly the length of $oh^c$'s trailing partial block, i.e.\ $\Last(oh^c)=(o,i)$. Re-expanding every observation of $oh$ but the last into a complete block, and the last into its first $i$ tags, therefore reconstructs $oh^c$ tag for tag. The max-player's maps $\Theta_\pone,\Upsilon_\pone$ are defined the same way, over max-observation histories.

\emph{Strategies.} At a history $oh^c$ ending at stage $i$ for observation $o$, define the strategy mappings $\Phi_\ptwo:\Pi_{\sg}\rightarrow \Pi_{\sg^c},\Psi_\ptwo: \Pi_{\sg^c} \rightarrow \Pi_\sg$ as follows:
\[
\Phi_\ptwo(\pi)(oh^c) := \pi(\Upsilon_\ptwo(oh^c))(t^o_i) \ \ (i\le k_o),
\qquad
\Phi_\ptwo(\pi)(oh^c) := \dirac{b_0} \ \ (i>k_o,\text{ or at an exit state}),
\]
\[
\Psi_\ptwo(\pi^c)(oh)(t^o_i) := \pi^c(\Theta_\ptwo(oh,i)).
\]
Informally, $\Phi_\ptwo(\pi)$ answers each stage of a gadget the way the source prescription would have answered at the state that stage stands for: at stage $i\le k_o$, it plays whatever $\pi$ would play at $t^o_i$; once $i$ exceeds $k_o$, or at the exit, there is no state left to represent, so it simply plays $b_0$. Conversely, $\Psi_\ptwo(\pi^c)$ recovers $\pi$'s $i$-th component by asking the compiled strategy $\pi^c$ what it plays once told, via $\Theta_\ptwo(oh,i)$, that stage $i$ of the fully-expanded gadget for $t^o_i$'s observation has been reached. And symmetrically $\Phi_\pone(\sigma)$ plays $\sigma(\Upsilon_\pone(oh^c))$ at an $\sone$-state and $\anew$ at a forced state, with $\Psi_\pone(\sigma^c)(oh):=\sigma^c(\Theta_\pone(oh))$. Each right-hand side is supported on the correct effective actions: $\atwo(t^o_i)$ at a selection stage, $\{b_0\}$ at a stage past $k_o$ or at the exit, $\{\anew\}$ at a forced state.

These maps are mutually inverse because $\Theta_\star,\Upsilon_\star$ are. For a selection-stage history $oh^c=\Theta_\ptwo(oh,i)$ with $i\le k_o$,
\[
\Psi_\ptwo(\Phi_\ptwo(\pi))(oh)(t^o_i)
=\Phi_\ptwo(\pi)(\Theta_\ptwo(oh,i))
=\pi(\Upsilon_\ptwo(\Theta_\ptwo(oh,i)))(t^o_i)
=\pi(oh)(t^o_i),
\]
so $\Psi_\ptwo\circ\Phi_\ptwo=\mathrm{id}$; substituting the other way, using $\Theta_\ptwo(\Upsilon_\ptwo(oh^c),i)=oh^c$, gives $\Phi_\ptwo\circ\Psi_\ptwo=\mathrm{id}$. Stages past $k_o$ and the max-player's maps invert the same way.

\emph{Probability preservation.} Fix a history ending at $\min$-state $t=t^o_j$. Under $(\sigma,\pi)$ in $\sg$, the next action is $b\sim\pi(oh)(t)$ and the next $\pone$-state is drawn from $\delta(t,b)$.

Under $(\Phi_\pone(\sigma),\Phi_\ptwo(\pi))$ in $\sg^c$: at every stage $i\ne j$, whichever action is played, $\delta^c((t,i,\beta),b)=\dirac{F(t,i,\beta')}$ lands at the same place regardless of $b$ (since $\beta'=\beta$ when $i\ne j$), so the total continuation probability is $\sum_b \Phi_\ptwo(\pi)(oh_i)(b)=1$ no matter what is played. At stage $j$, $\Phi_\ptwo(\pi)$ plays exactly $\pi(oh)(t)$, recording $\beta=b\sim\pi(oh)(t)$. Every forced state passes through deterministically, and the exit plays $b_0$ deterministically, so 

\[
\Pr\nolimits_{\sg^c}[\text{exit state}=s'] = \sum_{b}\pi(oh)(t)(b)\,\delta(t,b)(s') = \Pr\nolimits_\sg[\text{next state}=s'].
\]
By induction on history length, corresponding cylinder sets under $(\sigma,\pi)$ and $(\Phi_\pone(\sigma),\Phi_\ptwo(\pi))$ have equal probability, giving the first displayed equality for any measurable $W$; the reverse equality follows the same way using $\Psi_\pone,\Psi_\ptwo$.

The same idea gives support preservation, in plain terms: a $\sg^c$-play matches a possible $\sg$-play exactly when the action recorded at stage $j$ is one the strategy could really play there. Whatever is played at every other stage doesn't matter, so long as it's also something the strategy could really play, since any such choice still projects down to the same $\sg$-play. So nothing about which plays are possible in $\sg$ is lost or added by moving to $\sg^c$.

Together with probability and support preservation, this gives $\val_\sg(W)=\val_{\sg^c}(W^c)$, since $\val$ is the supremum over max-strategies of the infimum over min-strategies of the matching probabilities, and those probabilities agree pairwise under $\Phi_\pone,\Phi_\ptwo$ (and $\Psi_\pone,\Psi_\ptwo$). The same ingredients give the other three analyses directly:
\begin{itemize}
\item \emph{Sure winning} holds for $(\sigma,\pi)$ in $\sg$ iff every play it can produce satisfies $W$, iff -- by play correspondence and objective equivalence -- every play $(\Phi_\pone(\sigma),\Phi_\ptwo(\pi))$ can produce in $\sg^c$ satisfies $W^c$, iff the lifted pair is sure winning in $\sg^c$.
\item \emph{Almost-sure winning} holds iff the matching probability equals $1$, preserved directly by the displayed probability equality.
\item \emph{Limit-sure and quantitative winning} are threshold statements about $\val_\sg(W)$ and $\val_{\sg^c}(W^c)$, already shown equal above.
\end{itemize}

\emph{Memory regimes.} A finite-memory strategy is realized by a transducer $T=(Q,q_0,u,g)$: $Q$ is a finite memory set with initial state $q_0$, $u$ updates the memory on each new observation, and $g(q)$ outputs the strategy's full action assignment at memory state $q$, so the strategy played after observation history $oh$ is $g(q)$ for $q$ the memory state reached by applying $u$ along $oh$ from $q_0$. The stage index $i$ is never hidden -- it is written into the observation $(o,i)$ itself -- so no extra memory is ever needed just to track it.

\textbf{From $\pi$ to $\Phi_\ptwo(\pi)$.} Build $T^c$ over the same $Q$: hold memory fixed across a gadget, updating only once, on its closing tag: $u^c(q,(o,k{+}1)):=u(q,o)$, and $u^c(q,(o,i)):=q$ otherwise. Output $g^c(q,(o,i)):=g(q)(t^o_i)$ at a selection stage ($i\le k_o$) and $\dirac{b_0}$ at a stage past $k_o$ or at the exit.

\textbf{From $\pi^c$ to $\Psi_\ptwo(\pi^c)$.} Build $T$ over the same $Q$: to process one source observation $o$, replay $T^c$ over $\mathcal B(o)$ without storing intermediate states, $u(q,o):=u^{c,*}(q,\mathcal B(o))$, where $u^{c,*}(q,w)$ denotes $u^c$ applied one tag at a time along the sequence $w$, starting from $q$ (so $u^{c,*}(q,\varepsilon)=q$ and $u^{c,*}(q,w\cdot x)=u^c(u^{c,*}(q,w),x)$ for a tag $x$); to compute the output at stage $i$, stop the same replay just before $(o,i)$, $g(q)(t^o_i):=g^c(u^{c,*}(q,\mathcal B_{<i}(o)),(o,i))$. Both computations reuse $Q$, with no extra memory.

That these transducers realize $\Phi_\ptwo(\pi)$ and $\Psi_\ptwo(\pi^c)$ follows from $\Theta_\ptwo,\Upsilon_\ptwo$ being mutual inverses, exactly as in the Strategies part. The max-player's transducers expand and contract the same way, with forced output $\dirac{\anew}$ at every forced state. Both directions therefore preserve the number of memory states exactly: memoryless, finite-memory, and infinite-memory strategies all carry over unchanged. When $k=1$, $\sg^c=\sg$ and the transducers coincide.
\end{proof}
\end{lemma}

\begin{example}
\label{ex:prescription-compilation}
Let $t$ be a min-state in an observation class $\{t^o_1,t^o_2\}$ of size $k_o=2$, with $t=t^o_2$, and suppose the game's global stage count is $k=3$, set by a different observation class elsewhere with $3$ states. Take $\atwo(t^o_1)=\{a,a'\}$, $\atwo(t)=\atwo(t^o_2)=\{L,R\}$, $\delta(t,L)=\dirac{u}$, and $\delta(t,R)=\dirac{v}$ for max-states $u,v$. The compiled min-alphabet is $\{a,a',b_0\}$ at stage~1, $\{L,R,b_0\}$ at stage~2, and $\{b_0\}$, forced, at stage~3, with $b_0\notin\{a,a',L,R\}$. Figure~\ref{fig:prescription-compilation}
shows the compiled gadget. A prescription value $\pi(oh)(t)=\{L\mapsto 0.3, R\mapsto
0.7\}$ is realized by playing $\pi(oh)(t^o_1)\in\{a,a'\}$ at stage~$1$ (discarded on the displayed $t^o_2$ branch), $\{L\mapsto0.3,R\mapsto0.7\}$ at stage~$2$, the forced $b_0$ at the padding stage~$3$, and the fresh $b_0$ deterministically at the exit state
$D(t,\beta)$; the exit distribution over $\{u,v\}$ is exactly
$\{u\mapsto0.3,v\mapsto0.7\}$, matching $\sg$. This branch has four selection states, five forced max-states, and two exits; the action stored in the two exits is $L$ or $R$, never $b_0$.
\end{example}

\begin{lemma}[Ordinary strategies as prescriptions]
\label{lem:ordinary-as-prescription}
Let $\widehat\Pi_\sg$ denote the set of ordinary observation-based min-strategies
$\widehat\pi:\OH_\ptwo^\sg\to\dist(\atwo)$, with no restriction on which actions
$\widehat\pi$ supports at any history -- in particular $\widehat\pi(oh)$ need not avoid
actions ineffective at any state of $\Last(oh)$'s class. The \emph{broadcast map}
$\iota_\ptwo:\widehat\Pi_\sg\to\Pi_\sg$, defined by
\[
\iota_\ptwo(\widehat\pi)(oh)(t):=\widehat\pi(oh) \qquad\text{for every } t\in\stwo(\Last(oh)),
\]
is a bijection from $\widehat\Pi_\sg$ onto $\Pi_\sg^0:=\iota_\ptwo(\widehat\Pi_\sg)$, the
\emph{class-constant} prescriptions -- those $\pi\in\Pi_\sg$ with $\pi(oh)(t)=\pi(oh)(t')$
for all $t,t'\in\stwo(\Last(oh))$ -- and it preserves outcome probabilities exactly:
\[
\probm_\sg^{\sigma,\widehat\pi}(W) = \probm_\sg^{\sigma,\iota_\ptwo(\widehat\pi)}(W)
\qquad\text{for every } \sigma\in\Sigma_\sg,\ \widehat\pi\in\widehat\Pi_\sg,\text{ and measurable } W.
\tag{$*$}
\]
Consequently the game value, and the sure, almost-sure, limit-sure, and quantitative
analyses of $\sg$, are exactly the same whether the min-player ranges over $\widehat\Pi_\sg$
or over $\Pi_\sg^0$:
\[
\Val_\sg^{\widehat\Pi_\sg}(W) \;=\; \Val_\sg^{\Pi_\sg^0}(W), \qquad\text{where}\qquad
\Val_\sg^{X}(W):=\sup_{\sigma\in\Sigma_\sg}\inf_{\pi\in X}\probm_\sg^{\sigma,\pi}(W).
\]
A prescription outside $\Pi_\sg^0$ genuinely
discriminates within a class -- min-player strategies are then strictly more expressive
than $\widehat\Pi_\sg$ -- and it is exactly for such prescriptions that
Lemma~\ref{lem:prescription-compilation}'s compiled game is needed; nothing here asserts
$\Val_\sg^{\widehat\Pi_\sg}(W)=\Val_\sg^{\Pi_\sg}(W)$ for the \emph{full} prescription space,
which can fail once prescriptions outside $\Pi_\sg^0$ are allowed to compete. 

\begin{proof}
\emph{Well-defined.} $\iota_\ptwo(\widehat\pi)(oh)$ assigns the single distribution
$\widehat\pi(oh)\in\dist(\atwo)$ to every $t\in\stwo(o)$, $o=\Last(oh)$, so it lies in
$\prod_{t\in\stwo(o)}\dist(\atwo)$, as the prescription type requires.

\emph{Bijective onto $\Pi_\sg^0$.} For any $t\in\stwo(\Last(oh))$,
$\widehat\pi(oh)=\iota_\ptwo(\widehat\pi)(oh)(t)$, so $\widehat\pi$ is recovered from
$\iota_\ptwo(\widehat\pi)$ by reading off any one component -- this shows $\iota_\ptwo$ is
injective. Every $\iota_\ptwo(\widehat\pi)$ is class-constant by construction, so
$\Pi_\sg^0\subseteq\{\text{class-constant prescriptions}\}$; conversely, given a
class-constant $\pi$, setting $\widehat\pi(oh):=\pi(oh)(t)$ for any $t\in\stwo(\Last(oh))$ is
well-defined (class-constancy says the choice of $t$ does not matter) and satisfies
$\iota_\ptwo(\widehat\pi)=\pi$. So $\iota_\ptwo$ is onto the class-constant prescriptions
exactly, and, being injective, a bijection $\widehat\Pi_\sg\to\Pi_\sg^0$.

\emph{Probability preservation, $(*)$.} At a min-state $t=t^o_j$, playing $\widehat\pi$
draws the next action from $\widehat\pi(oh)$; playing $\iota_\ptwo(\widehat\pi)$ draws it
from $\iota_\ptwo(\widehat\pi)(oh)(t)=\widehat\pi(oh)$ -- the same distribution, since
$t\in\stwo(o)$. Every other component of the game -- the states, $\sigma$'s response, the
transition function $\delta$ -- is literally unchanged. So
$(\sigma,\widehat\pi)$ and $(\sigma,\iota_\ptwo(\widehat\pi))$ induce identical transition
probabilities at every step of every history; by induction on history length and 
the probability measure over runs is uniquely determined by its values on cylinder sets, it suffices to
show that corresponding cylinder sets have the same probability, giving $(*)$ for any measurable $W$, with no size blow-up and no change to
memory requirements.

\emph{Value preservation.} Fix $\sigma\in\Sigma_\sg$. By $(*)$, the bijection
$\widehat\pi\mapsto\iota_\ptwo(\widehat\pi)$ from $\widehat\Pi_\sg$ to $\Pi_\sg^0$ matches
each ordinary strategy with a prescription achieving the \emph{same} probability of $W$
against $\sigma$. Hence the two sets of achievable probabilities coincide exactly, not just
their bounds:
\[
\big\{\probm_\sg^{\sigma,\widehat\pi}(W):\widehat\pi\in\widehat\Pi_\sg\big\}
=\big\{\probm_\sg^{\sigma,\iota_\ptwo(\widehat\pi)}(W):\widehat\pi\in\widehat\Pi_\sg\big\}
=\big\{\probm_\sg^{\sigma,\pi}(W):\pi\in\Pi_\sg^0\big\},
\]
the last equality because $\iota_\ptwo$ is onto $\Pi_\sg^0$. Two sets of reals that coincide
have equal infima, so
\[
\inf_{\widehat\pi\in\widehat\Pi_\sg}\probm_\sg^{\sigma,\widehat\pi}(W)
= \inf_{\pi\in\Pi_\sg^0}\probm_\sg^{\sigma,\pi}(W)
\qquad\text{for every } \sigma\in\Sigma_\sg.
\]
This holds pointwise in $\sigma$; taking $\sup_{\sigma\in\Sigma_\sg}$ of both sides -- valid
since the two sides agree termwise, with no further structure on $\Sigma_\sg$ needed --
gives $\Val_\sg^{\widehat\Pi_\sg}(W)=\Val_\sg^{\Pi_\sg^0}(W)$. The same termwise argument,
applied to $\{\sigma:\probm_\sg^{\sigma,\widehat\pi}(W)=1\ \forall\widehat\pi\}$ and its
prescription counterpart (and likewise at any limit-sure or quantitative threshold),
transfers sure, almost-sure, and limit-sure winning exactly between the two strategy
spaces, since these notions are themselves defined from the same probabilities compared in
$(*)$.

Theargument uses only that $\iota_\ptwo$ is an exact, probability-preserving bijection onto
$\Pi_\sg^0$. It does \emph{not} extend to $\Val_\sg^{\Pi_\sg}(W)$ for the full prescription
space: once prescriptions outside $\Pi_\sg^0$ are allowed to compete in the infimum, the
min-player may do strictly better by discriminating within a class, and
$\Val_\sg^{\Pi_\sg}(W)\le\Val_\sg^{\widehat\Pi_\sg}(W)$ can be strict as shown in example \ref{ex:strict-gap}. 
\end{proof}
\end{lemma}

\begin{figure}[t]
\centering
\begin{tikzpicture}[
    >={Stealth[scale=1.1]},
    maxstate/.style={circle, draw=maxcolor, thick, minimum size=0.85cm},
    minstate/.style={rectangle, draw=mincolor, thick, minimum size=0.6cm, inner sep=3pt},
    edge/.style={thick,->,>=Stealth},
    obsbox/.style={rounded corners, draw=mincolor, densely dashed, thick, inner sep=6pt},
    every node/.style={font=\small}
]

\begin{scope}
\node[maxstate] (s0L) at (0,5.2) {$s_0$};
\node[minstate] (t1L) at (-1.4,3.4) {$t_1$};
\node[minstate] (t2L) at (1.4,3.4) {$t_2$};
\node[maxstate,fill=green!12] (gL) at (-1.4,0.6) {$\mathsf{good}$};
\node[maxstate,fill=red!10] (bL) at (1.4,0.6) {$\mathsf{bad}$};

\draw[edge] (s0L) -- node[left]{$\tfrac12$} (t1L);
\draw[edge] (s0L) -- node[right]{$\tfrac12$} (t2L);
\draw[edge] (t1L) -- node[left,pos=0.55]{$L\,(q)$} (gL);
\draw[edge] (t1L) -- node[above,sloped,pos=0.8]{$R\,(1{-}q)$} (bL);
\draw[edge] (t2L) -- node[above,sloped,pos=0.7]{$R\,(1{-}q)$} (gL);
\draw[edge] (t2L) -- node[right,pos=0.55]{$L\,(q)$} (bL);

\node[obsbox,fit=(t1L)(t2L)] (obsL) {};
\node[] at (0,3.4) {\scriptsize same observation $o$};

\node at (0,6.2) {\textbf{Ordinary strategy: one $q$ for both}};
\node[align=center] at (0,-0.6) {\small $\Pr[\mathsf{good}]=\tfrac12 q+\tfrac12(1{-}q)=\tfrac12$\\[-1pt]\small for \emph{every} $q\in[0,1]$};
\end{scope}

\begin{scope}[xshift=7.2cm]
\node[maxstate] (s0R) at (0,5.2) {$s_0$};
\node[minstate] (t1R) at (-1.4,3.4) {$t_1$};
\node[minstate] (t2R) at (1.4,3.4) {$t_2$};
\node[maxstate,fill=green!12] (gR) at (-1.4,0.6) {$\mathsf{good}$};
\node[maxstate,fill=red!10] (bR) at (1.4,0.6) {$\mathsf{bad}$};

\draw[edge] (s0R) -- node[left]{$\tfrac12$} (t1R);
\draw[edge] (s0R) -- node[right]{$\tfrac12$} (t2R);
\draw[edge,draw=mincolor!35] (t1R) -- node[left,pos=0.55]{$L$} (gR);
\draw[edge] (t1R) -- node[below,sloped,pos=0.85]{$R$} (bR);
\draw[edge] (t2R) -- node[above,sloped,pos=0.7]{$L$} (bR);
\draw[edge,draw=mincolor!35] (t2R) -- node[right,pos=0.9]{$R$} (gR);

\node[obsbox,fit=(t1R)(t2R)] (obsR) {};
\node[] at (0,3.4) {\scriptsize same observation $o$};

\node at (0,6.2) {\textbf{Prescription: $t_1\mapsto R$, $t_2\mapsto L$}};
\node[align=center] at (0,-0.6) {\small $\Pr[\mathsf{good}]=\tfrac12\cdot 0+\tfrac12\cdot 0=0$};
\end{scope}

\end{tikzpicture}
\caption{The observation class $\stwo(o)=\{t_1,t_2\}$ is dashed to mark that both states
produce the same history $oh$, so any \emph{ordinary} strategy (left) must answer them with
the same distribution $(q,1{-}q)$ over $(L,R)$ -- and every choice of $q$ yields exactly
$\Pr[\mathsf{good}]=\tfrac12$, since $t_1,t_2$ want opposite answers and the two errors
always average out. A \emph{prescription} (right) answers each state correctly and forces
$\Pr[\mathsf{good}]=0$, faded edges marking the actions it never plays.}
\label{fig:strict-gap}
\end{figure}

\begin{example}[The expressiveness gap is strict]
\label{ex:strict-gap}
Let $\sg$ have max-states $\sone=\{s_0,\mathsf{good},\mathsf{bad}\}$, min-states $t_1,t_2$
sharing a single observation $o$, so $\stwo(o)=\{t_1,t_2\}=\{t^o_1,t^o_2\}$ ($k_o=2$), and
actions $\atwo=\{L,R\}$; $s_0$ has a single available action $a$. Transitions:
\[
\delta(s_0,a) := \tfrac12\dirac{t_1}+\tfrac12\dirac{t_2},
\delta(t_1,L):=\dirac{\mathsf{good}},\ \ \delta(t_1,R):=\dirac{\mathsf{bad}},
\]
\[
\delta(t_2,L):=\dirac{\mathsf{bad}},\ \ \delta(t_2,R):=\dirac{\mathsf{good}}.
\]
$\mathsf{good}$ and $\mathsf{bad}$ are absorbing with
$\mathsf{good}$ max-winning and $\mathsf{bad}$ min-winning; take $W$ to be reachability to
$\mathsf{good}$ (e.g.\ priority $0$ on $\mathsf{good}$'s region, priority $1$ elsewhere).
Figure~\ref{fig:strict-gap} shows both regimes. Since $s_0$ has only one available action,
$\Sigma_\sg$ is a singleton and $\Val_\sg^X(W)=\inf_{\pi\in X}\probm_\sg^{\sigma,\pi}(W)$ for
the unique $\sigma$, whichever strategy space $X$ ranges over. Both $t_1,t_2$ produce the
identical min-observation history $oh$ after $s_0,a$, since $\obfntwo(t_1)=\obfntwo(t_2)=o$.

\emph{Ordinary strategies.} $\widehat\pi(oh)=(q,1-q)$ over $(L,R)$, for any fixed
$q\in[0,1]$, gives
\[
\probm_\sg^{\sigma,\widehat\pi}(W)
= \tfrac12\big(q\cdot1+(1-q)\cdot0\big) + \tfrac12\big(q\cdot0+(1-q)\cdot1\big)
= \tfrac12,
\]
independently of $q$: every ordinary strategy scores exactly $\tfrac12$, so
$\Val_\sg^{\widehat\Pi_\sg}(W)=\tfrac12$. This is exactly the phenomenon underlying the
well-formedness argument of Lemma~\ref{lem:ordinary-as-prescription}: $\widehat\pi(oh)$ must
answer the same way regardless of which of $t_1,t_2$ is truly occupied, and here the two
states want opposite answers, so no single $q$ helps.

\emph{Prescriptions.} $\pi(oh)(t_1):=\dirac R$, $\pi(oh)(t_2):=\dirac L$ sends both states
to $\mathsf{bad}$ deterministically, so
$\probm_\sg^{\sigma,\pi}(W)=\tfrac12\cdot0+\tfrac12\cdot0=0$, giving
$\Val_\sg^{\Pi_\sg}(W)=0$.
\end{example}

\section{Reduction of RPOMDPs to POSGs}
\label{app:rpomdp-posg}

\smallskip\noindent{\bf Definition:} Let \( \rmdp = (S, A, \uncert,\obsa,\obse,\obfna,\obfne,s_0) \) be a polytopic RPOMDP. We define the POSG induced by $\rmdp{}$ as \( \sg =  (\sone^{\sg},\stwo^{\sg},\aone^{\sg},\atwo^{\sg},\delta^{\sg}, s_0^{\sg},\obsone^{\sg},\obstwo^{\sg},\obfnone^{\sg},\obfntwo^{\sg}) \). For each $(s,a)\in S\times A$, let $V_{s,a}$ denote the set of vertices of the local uncertainty polytope $\uncertsa(s,a)$.
\begin{itemize}
    \item {\em States:} $\sone^\sg = S$ and $\stwo^\sg = \{(s,a):s\in S,\ a\in A_\rmdp(s)\}$. For an original RPOMDP with every action effective, this is $S\times A$. Let $S^\sg = \sone^{\sg} \cup \stwo^{\sg}$.
    \item {\em Actions:} $\aone^\sg=A$ and
    $\atwo^\sg=\bigcup_{(s,a)\in\stwo^\sg}V_{s,a}$. The effective alphabets are $\aone^\sg(s)=A_\rmdp(s)$ where $A_\rmdp(s)$ denotes the set of effective actions at state $s$ and $\atwo^\sg((s,a))=V_{s,a}$. Let $A^\sg=\aone^\sg\cup\atwo^\sg$.
    \item {\em Transition function} $\delta^{\sg}$ on the effective state--action pairs is defined as follows:
    \begin{itemize}
        \item For $s\in\sone^\sg$ and $a\in A_\rmdp(s)$, the transition function is the Dirac distribution with probability mass on $(s,a)\in\stwo^\sg$
        \begin{equation*}
        \delta^{\sg}(s,a)(s') = \begin{cases}
            1 &\text{if } s' = (s,a) \\
            0 &\text{otherwise}
        \end{cases}
        \end{equation*}
        \item For an effective min-action $v\in V_{s,a}$ at $(s,a)\in\stwo^\sg$, the transition function is the probability distribution over $\sone^\sg$ given by
        \begin{equation*}
        \delta^{\sg}((s,a),v)(s') = \begin{cases}
            v(s') &\text{where}~ s' \in \sone^\sg ~\text{and}~v(s')~\text{denotes the}~ s'\text{th entry in}~v\\
            0 &\text{otherwise}
        \end{cases} 
        \end{equation*}
        Other min-actions have probability zero by the standing strategy convention.
    \end{itemize}
    \item {\em Initial state:} $s_0^{\sg} = s_0 \in \sone^\sg$.
    \item {\em Observation functions:} $\obfnone^{\sg}$ and $\obfntwo^{\sg}$ are functions from states $S^{\sg}$ to $\obsone^{\sg}=\obsa$ and $\obstwo^{\sg}=\obse$ respectively.
    \begin{itemize}
        \item For $s \in \sone^\sg$, $\obfnone^{\sg}(s) = \obfna(s)$
        \item For $(s,a) \in \stwo^\sg$, $\obfnone^{\sg}((s,a)) = \obfna(s)$
        \item For $s \in \sone^\sg$, $\obfntwo^{\sg}(s) = \obfne(s)$
        \item For $(s,a) \in \stwo^\sg$, $\obfntwo^{\sg}((s,a)) = \obfne(s)$
    \end{itemize}
\end{itemize}

A play in $\sg$ has the form $\rho=s_0,a_0,s_1,a_1,\ldots\in\Playssg$, where for every $k\geq0$, $s_{2k}\in\sone^\sg$, $a_{2k}\in A_\rmdp(s_{2k})$, $s_{2k+1}=(s_{2k},a_{2k})\in\stwo^\sg$, and $a_{2k+1}\in V_{s_{2k},a_{2k}}$. At step $t$, the history is $h_t=s_0,(o^1_0,o^2_0),s_1,\ldots,s_t,(o^1_t,o^2_t)$, where $o^1_i=\obfnone^\sg(s_i)$ and $o^2_i=\obfntwo^\sg(s_i)$. \par

\subsection{Mapping Plays, Observation Histories and Strategies from $\rmdp$ to $\sg$}
\label{app:map}
In this section, we establish a formal mapping between the plays, observation histories and strategies of the RPOMDP $\rmdp$ 
and the POSG $\sg$. 

\textbf{Mapping on Plays:}
\begin{itemize}
    \item Define a mapping $\Theta:\Plays_\rmdp \rightarrow 2^{\Plays_{\sg}}$ as follows. Let $\rho = s_0,a_0,s_1,a_1,\cdots \in \Playsr$. Then $\Theta(\rho)$ is the set of all plays $\rho' = s'_0,a'_0,s'_1,a'_1,s'_2,a'_2,\cdots \in \Plays_{\sg}$ is given index-wise, $\forall k\geq 0:$ $s'_{2k}=s_k,a'_{2k}=a_k,s'_{2k+1}=(s'_{2k},a'_{2k}) = (s_k,a_k),a'_{2k+1}=v \in \{v \in V_{s_k, a_k} \mid v(s_{k+1}) > 0\}$ (arbitrary $v$), since $P(s_k, a_k)(s_{k+1}) > 0$ and $P(s_k, a_k) \in \operatorname{conv}(V_{s_k, a_k})$, this subset is non-empty.
. Thus, $\Theta(\rho)$ consists of all plays in $G$ obtained by resolving each transition in $\rho$ using an arbitrary choice of vertex from the uncertainty set at each step. For $X \subseteq \Plays_\rmdp,\quad \Theta(X) = \bigcup\limits_{\rho \in X} \Theta(\rho)$.
    \item Define $\Upsilon: \Plays_{\sg} \rightarrow \Plays_\rmdp$ as follows. Let $\rho' = s'_0,a'_0,s'_1,a'_1,s'_2,a'_2,\cdots$ be a play in $\sg$. Then $\Upsilon(\rho') = \rho = s_0,a_0,s_1,a_1,\cdots \in \Plays_\rmdp$ is given index-wise by $s_i=s'_{2i}$ and $a_i=a'_{2i}$ for every $i\geq0$. The map $\Upsilon$ is many-to-one. For $Y\subseteq\Plays_\sg$, define $\Upsilon(Y)=\{\Upsilon(\rho')\mid\rho'\in Y\}$; for $X\subseteq\Plays_\rmdp$, define its preimage $\Upsilon^{-1}(X)=\{\rho'\in\Plays_\sg\mid\Upsilon(\rho')\in X\}$.
\end{itemize}

\textbf{Mapping on Observation Histories:} $\OH_a,\OH_e$ denote the set of agent and environment observation histories respectively in RPOMDP $\rmdp$. $\OH^\sg_\pone,\OH^\sg_\ptwo$ denote observation histories of $\pone,\ptwo$ players respectively in POSG $\sg$. \par
\begin{enumerate}
    \item Define $\Theta_a:\OH_a \rightarrow \OH_{\pone}^\sg$. For agent observation history $oh^a = o_0,o_1\cdots, o_t \in \OH_a$, the mapping is given by $\Theta_a(oh^a) = oh^{\pone} = o'_0,o'_1,\cdots,o'_{2t}$ is given index-wise, $\forall\:0 \leq k \leq t-1, o'_{2k} = o'_{2k+1} = o_k, o'_{2t} = o_t$.

    \item Define $\Upsilon_a: \OH_{\pone}^\sg \rightarrow \OH_a$. For $\pone$-player observation history $oh^{\pone} = o_0,o_1,\cdots,o_{t-1},o_{t}$, the mapping is given by $\Upsilon_a(oh^{\pone}) = oh^a = o'_0,o'_1,\cdots,o'_{t/2}$ is given index-wise, $\: \forall 0\leq i \leq t/2, \: \: o'_i = o_{2i}$.
    
    \item Define $\Theta_e:\OH_e \rightarrow \OH_{\ptwo}^\sg$. Let $oh^e = o_0,o_1,\cdots,o_t \in \OH_e$. The function is given by $\Theta_e(oh^e) = oh^{\ptwo} = o'_0,o'_1,\cdots,o'_{2t},o'_{2t+1}$ is given index-wise, $\forall 0\leq k\leq t, o'_{2k} = o'_{2k+1} =o_k$.

    \item Define $\Upsilon_e: \OH_{\ptwo}^\sg \rightarrow \OH_e$. Let $oh^{\ptwo} = o_0,o_1,\cdots,o_t \in \OH_{\ptwo}^\sg$. The function is given by $\Upsilon_e(oh^{\ptwo}) = oh^e = o'_0,o'_1,\cdots,o'_{(t-1)/2}$ is given index-wise, $\forall 0\leq i \leq (t-1)/2, \: \: o'_i = o_{2i}$.
\end{enumerate}

\begin{lemma}
\label{ref:mdpsg_oh10}
    Let $\star \in \{a,e\}$. $\Theta_\star$ and $\Upsilon_\star$ are inverses of each other.
    \begin{proof}
        \begin{enumerate}
            \item For all $oh = o_0,o_0,o_1,o_1,o_2,o_2,\cdots, o_{t-1},o_{t-1},o_t \in \OH^\sg_\pone$, \[\Theta_a(\Upsilon_a(oh)) = \Theta_a(o_0,o_1,o_2,\cdots,o_{t-1},o_t) = o_0,o_0,o_1,o_1,o_2,o_2,\cdots, o_{t-1},o_{t-1},o_t = oh\]
            \item For all $oh = o_0,o_1,o_2,\cdots,o_{t-1},o_t \in \OH_a$, \[\Upsilon_a(\Theta_a(oh)) = \Upsilon_a(o_0,o_0,o_1,o_1,o_2,o_2,\cdots, o_{t-1},o_{t-1},o_t) = o_0,o_1,o_2,\cdots,o_{t-1},o_t = oh\]
            \item For all $oh = o_0,o_0,o_1,o_1,o_2,o_2,\cdots, o_{t-1},o_{t-1},o_t,o_t \in \OH^\sg_\ptwo$, \[\Theta_e(\Upsilon_e(oh)) = \Theta_e(o_0,o_1,o_2,\cdots,o_{t-1},o_t) = o_0,o_0,o_1,o_1,o_2,o_2,\cdots, o_{t-1},o_{t-1},o_t,o_t = oh\]
            \item For all $oh = o_0,o_1,o_2,\cdots,o_{t-1},o_t \in \OH_e$, \[\Upsilon_e(\Theta_e(oh)) = \Upsilon_e(o_0,o_0,o_1,o_1,o_2,o_2,\cdots, o_{t-1},o_{t-1},o_t,o_t) = o_0,o_1,o_2,\cdots,o_{t-1},o_t = oh \]
        \end{enumerate}
    \end{proof}
\end{lemma}

\textbf{Mapping of Strategies}
For each $(s,a)\in S\times A$, define the barycenter map
\[
\mathsf{bar}_{s,a}:\dist(V_{s,a})\to\uncertsa(s,a),
\qquad
\text{such that for $\lambda\in\dist(V_{s,a})$,}\:\mathsf{bar}_{s,a}(\lambda)=\sum_{v\in V_{s,a}}\lambda(v)v.
\]
For each $P\in\uncertsa(s,a)$, choose coefficients
$\alpha^{P}_{s,a}\in\dist(V_{s,a})$ such that
$\mathsf{bar}_{s,a}(\alpha^{P}_{s,a})=P$.
Such coefficients need not be unique; fix one choice for every $P$ once and for all.
We identify $\dist(V_{s,a})$ with distributions over $\atwo^\sg$ supported on $V_{s,a}$.
Also fix an arbitrary reference distribution $P^0_{s,a}\in\uncertsa(s,a)$ for every $(s,a)\in S\times A$.
\begin{enumerate}
    \item The agent strategy $\sigma \in \Sigma_\rmdp$ for RPOMDP $\rmdp$ is given by
            \[
        \sigma:\OH_a\to\dist(A).
            \]
    \item The environment strategy $\pi \in \Pi_\rmdp$ for RPOMDP $\rmdp$ is given by
            \[
        \pi:\OH_e\times A\to\uncert.
            \]
    \item The $\pone$-player strategy $\sigma' \in \Sigma_{\sg}$ for POSG $\sg$ from the reduction is given by
            \[
        \sigma':\OH_\pone^\sg\to\dist(\aone^\sg).
            \]
    \item The $\ptwo$-player strategy $\pi' \in \Pi_{\sg}$ for POSG $\sg$ from the reduction is given by
            \[
        \pi':\OH_\ptwo^\sg\to
        \bigcup_{o\in\obstwo^\sg}
        \prod_{s\in\stwo^\sg(o)}\dist(\atwo^\sg)
            \]
            Every compatible component $\pi'(oh)((s,a))$ is supported on $V_{s,a}$.
\end{enumerate}

Define the mappings
\begin{itemize}
    \item $\mdpsga_a: \Sigma_\rmdp \rightarrow \Sigma_{\sg}$. For all $\sigma \in \Sigma_\rmdp$ and $oh \in \OH_{\pone}^\sg$,
            \[
                \mdpsga_a(\sigma)(oh) = \sigma(\Upsilon_a(oh)).
            \]
    \item $\mdpsge_a: \Sigma_{\sg} \rightarrow \Sigma_{\rmdp}$. For all $\sigma' \in \Sigma_{\sg}$ and $oh \in \OH_{a}$,
            \[
                \mdpsge_a(\sigma')(oh) = \sigma'(\Theta_a(oh)).
            \]
    \item 
    $\mdpsga_e: \Pi_\rmdp \rightarrow \Pi_{\sg}$. For all $\pi \in \Pi_\rmdp$, $oh \in \OH_{\ptwo}^\sg$, and $(s,a) \in \stwo^\sg$ with $\obfntwo^\sg((s,a))=\Last(oh)$, use the component notation above. 
            \[
                \mdpsga_e(\pi)(oh)((s,a))
                =\alpha^{\pi(\Upsilon_e(oh),a)(s,a)}_{s,a}.
            \]

    \item $\mdpsge_e: \Pi_{\sg} \rightarrow \Pi_{\rmdp}$. For all $\pi' \in \Pi_{\sg}$, $oh \in \OH_{e}$, and $a\in A$, define the complete assignment $\mdpsge_e(\pi')(oh,a)\in\uncert$ componentwise as follows.
            \[
            \mdpsge_e(\pi')(oh,a)(x,b)=
            \begin{cases}
            \mathsf{bar}_{x,a}\!\left(\pi'(\Theta_e(oh))((x,a))\right)
                & \begin{array}{l}\text{if }b=a\in A_\rmdp(x) \text{ and }\obfne(x)=\Last(oh),\end{array}\\
            P^0_{x,b} & \text{otherwise.}
            \end{cases}
            \]
            Hence every component belongs to its corresponding local uncertainty set.
\end{itemize}

\begin{remark}
\label{mdpsge}
    For every $\pi \in \Pi_\rmdp$, $oh \in \OH_{\ptwo}^\sg$, and compatible $(s,a)$,
    \[
    \pi(\Upsilon_e(oh),a)(s,a)
    =\mathsf{bar}_{s,a}\!\left(\mdpsga_e(\pi)(oh)((s,a))\right).
    \]
\end{remark}

\begin{lemma}
\label{mdpsg_sp}
    The agent maps $\mdpsga_a$ and $\mdpsge_a$ are mutual inverses. The environment maps satisfy, at every compatible $oh$ and effective pair $(s,a)$ with $\obfntwo^\sg(s) = \Last(oh)$,
    \begin{align*}
    \mdpsge_e(\mdpsga_e(\pi))(oh,a)(s,a) &=\pi(oh,a)(s,a)\\
    \mathsf{bar}_{s,a}\!\left(\mdpsga_e(\mdpsge_e(\pi'))(oh)((s,a))\right)&=\mathsf{bar}_{s,a}\!\left(\pi'(oh)((s,a))\right)
    \end{align*}
    Thus both compositions preserve the next-state distribution.
    \begin{proof}
        The agent identities follow from Lemma~\ref{ref:mdpsg_oh10}: both maps compose observation histories and copy a single action distribution. The local effective max alphabets agree with the RPOMDP agent alphabets, so the maps preserve support. For every min-strategy, $\pi'(oh)((s,a))\in\dist(V_{s,a})$; hence the reverse barycenter is well-defined without a further strategy restriction. The environment identities follow from
        $\mathsf{bar}_{s,a}(\alpha^P_{s,a})=P$ and the two history identities in Lemma~\ref{ref:mdpsg_oh10}.
        \begin{align*}
    \mdpsge_e(\mdpsga_e(\pi))(oh,a)(s,a) &= \mathsf{bar}_{s,a}\!\left(\mdpsga_e(\pi)(\Theta_e(oh))((s,a))\right)
      =\pi(\Upsilon_e(\Theta_e(oh)),a)(s,a)=\pi(oh,a)(s,a)\\
    \mathsf{bar}_{s,a}\!\left(\mdpsga_e(\mdpsge_e(\pi'))(oh)((s,a))\right)& =\mdpsge_e(\pi')(\Upsilon_e(oh),a)(s,a)=\mathsf{bar}_{s,a}\!\left(\pi'(\Theta_e(\Upsilon_e(oh)))((s,a))\right)
      =\mathsf{bar}_{s,a}\!\left(\pi'(oh)((s,a))\right)
    \end{align*}

    \end{proof}
\end{lemma}
\subsection{Lifting Objectives}
\label{app:lift}
Let $W$ be a parity objective in RPOMDP $\rmdp$ given by the priority function $p:S \rightarrow \{0,1,\cdots,d\},\ d \in \nat$. Then, the \emph{lifted} objective in POSG $\sg$, $W' \subseteq \Plays^\sg$ is given by the priority function $p^\sg: \sone^\sg \cup \stwo^\sg \rightarrow \{0,1,\cdots,d\}$ defined as
    \[p^\sg(s) =     
    \begin{cases}
        p(s) & \text{if } s \in \sone^\sg\\
        p(s') & \text{if } s = (s',a) \in \stwo^\sg\\
    \end{cases}
    \]

\begin{lemma}
\label{lem: rglo}
    Let $W$ be parity objective in RPOMDP $\rmdp$ given by the priority function $p$ and let $W'$ be the \emph{lifted} objective in POSG $\sg$. Then,
    
    \[ \rho \in W \iff \Theta(\rho) \subseteq W'\]
    \begin{proof}
 We have $W = \parity{\rmdp}{p}$ and $W' = \parity{\sg}{p^\sg}$ where $p^\sg$ is defined as above. Any play $\rho = s_0,a_0,s_1,a_1,\cdots\in \Plays_\rmdp$ has parity sequence: $p(\rho) = p_0,p_1,p_2,\cdots$ such that $p_i = p(s_i)$. Any play in $\Theta(\rho)$ is of the form $\rho' = s_0,a_0,(s_0,a_0),a'_1,s_1,\cdots$ with parity sequence $p^\sg(\rho') = p_0,p_0,p_1,p_1,\cdots$, since $p^\sg(s_i) = p^\sg((s_i,a_i)) = p(s_i)$. Due to repeating values $\liminf_{i \to \infty}$ of the sequence do not change, then $\liminf_{i \to \infty}(p(\rho)) = \liminf_{i \to \infty}(p^\sg(\rho'))$. Thus, $\rho \in W \implies \rho' \in W' \iff \Theta(\rho) \subseteq W'$. The reverse direction follows similarly.

    \end{proof}
\end{lemma}

\subsection{Value Preservation}
\label{app:value}

\begin{lemma}
\label{mdpsg_vp}
        The strategy maps preserve the induced probability measure in both directions.  In particular,
        \begin{align*}
        \probm_{\rmdp}^{\sigma,\pi}(W)
            &=\probm_{\sg}^{\mdpsga_a(\sigma),\mdpsga_e(\pi)}(W')
            &&\text{for all $\sigma\in\Sigma_\rmdp$, $\pi\in\Pi_\rmdp$},\\
        \probm_{\sg}^{\sigma',\pi'}(W')
            &=\probm_{\rmdp}^{\mdpsge_a(\sigma'),\mdpsge_e(\pi')}(W)
            &&\text{for all $\sigma'\in\Sigma_\sg$, $\pi'\in\Pi_\sg$},
        \end{align*}
    where $W \subseteq \Plays_\rmdp$ is a parity objective of $\rmdp$ and $W'=\Theta(W)$ is the corresponding objective in $\sg$.

    \begin{proof}
    For a history $h$, let cylinder set ($\Cone(h)$) denote the set of all runs having $h$ as a prefix. Since the probability measure over runs is uniquely determined by its values on cylinder sets, it suffices to show that corresponding cylinder sets have the same probability.

    Let $H_\rmdp$(resp. $H_\sg$) denote the set of histories in RPOMDP $\rmdp$ (resp. POSG $\sg$). Define the mapping $\Theta_h: H_\rmdp \rightarrow 2^{H_\sg}$ as follows: Let $h = s_0,(o^a_0,o^e_0), s_1, \dots s_t, (o^a_t,o^e_t) \in H_\rmdp$. Then $\Theta_h(h)$ is the set of all histories $h' = s'_0,(o^{max}_0,o^{min}_0),s'_1, \dots s'_{2t}, (o^{max}_{2t},o^{min}_{2t})$ is given index-wise, $\forall k\geq 0:$ $s'_{2k}=s_k,s'_{2k+1}=(s_k,a_k)$ where $a_k \in \aone^\sg$ and $\forall i\geq 0$ $o^{max}_i = \obfnone^\sg(s'_i),o^{min}_i = \obfntwo^\sg(s'_i)$. Thus, $\Theta_h(h)$ consists of all histories in $G$ obtained by resolving each transition in $h$ using an arbitrary choice of action from $\atwo^\sg$ at each step.\par
     
        We show that for a finite history $h$ of every play in $W$, there exists corresponding set of finite histories $H' = \Theta_h(h)$ of some plays in $W'$ such that 
        \[
            \probm_{\rmdp}^{\sigma, \pi} (\Cone(h)) =  \sum_{h' \in H'}\probm_{\sg}^{\mdpsga_a(\sigma), \mdpsga_e(\pi)} (\Cone(h'))
        \]

        The proof follows induction on the length of $h$ with base condition $h=s_0$ and $H' = \{s_0\}$ with values $\probm_{\rmdp}^{\sigma, \pi} (\Cone(h)) = \sum_{h' \in H'}\probm_{\sg}^{\mdpsga_a(\sigma), \mdpsga_e(\pi)} (\Cone(h')) = 1$.\par
        For induction, let $h=h_1,s_1,(o_1,o'_1)$ where $o_1 = \obfna(s_1),o'_1 = \obfne(s_1)$ and $h_1$ ending in state $s$ with observation history pair $(o,o')$. Let $H'_1 = \Theta_h(h_1)$. A history $h' \in H'$ is of the form $h' = h'_1,(s,a),(o,o'),s_1,(o_1,o'_1)$ where $h'_1 \in H'_1,\:a\in A$ (an arbitrary action $a$). Let $oh^a,oh^e$ be the agent and environment observation histories for history $h_1$. Let $oh^{\pone}$ be the $\pone$ player observation history for all histories $h'_1 \in H'_1$. Then $oh^a = \Upsilon_a(oh^{\pone})$. Let $oh^{\ptwo}$ be $\ptwo$ player observation history for $h'_1,(s,a),(o,o')$ where $h'_1 \in H'_1$. Then $oh^e = \Upsilon_e(oh^{\ptwo})$.\par

        For a given $\sigma' \in \Sigma_\sg,\pi' \in \Pi_\sg$ in history $h'$, the transition probability from $s$ to $(s,a)$ is given by $\sigma'(oh^{\pone})(a)$ and the transition probability from $(s,a)$ to $s_1$ is given by all the transitions on the vertices $v \in V_{s,a}$ i.e., $\left(\sum_{v \in V_{s,a}}\pi'(oh^{\ptwo})((s,a))(v)\cdot v(s_1)\right)$.

        \begin{adjustwidth}{}{-20pt}
        \begin{align*}
            \sum_{h' \in H'}&\probm_{\sg}^{\mdpsga_a(\sigma), \mdpsga_e(\pi)} (\Cone(h')) \\ &= \sum_{h' \in H'}\left(\probm_{\sg}^{\mdpsga_a(\sigma), \mdpsga_e(\pi)} (\Cone(h_1')) \times \mdpsga_a(\sigma)(oh^{\pone})(a) \times \left(\sum_{v \in V_{s,a}}\mdpsga_e(\pi)(oh^{\ptwo})((s,a))(v)\cdot v(s_1)\right)\right)\\
            &= \sum_{h' \in H'_1}\probm_{\sg}^{\mdpsga_a(\sigma), \mdpsga_e(\pi)} (\Cone(h_1')) \times \sum_{a \in A}\left(\mdpsga_a(\sigma)(oh^{\pone})(a) \times \left(\sum_{v \in V_{s,a}}\mdpsga_e(\pi)(oh^{\ptwo})((s,a))(v)\cdot v(s_1)\right)\right)        
        \end{align*}
        \end{adjustwidth}
        By induction hypothesis, $\sum_{h' \in H'_1}\probm_{\sg}^{\mdpsga_a(\sigma), \mdpsga_e(\pi)} (\Cone(h'_1)) = \probm_{\rmdp}^{\sigma, \pi} (\Cone(h_1))$. From definition of $\mdpsga_a$, $\mdpsga_a(\sigma)(oh^{\pone})(a) = \sigma(\Upsilon_a(oh^\pone))(a) = \sigma(oh^a)(a)$. From remark $\ref{mdpsge}$, $\left(\sum_{v \in V_{s,a}}\mdpsga_e(\pi)(oh^{\ptwo})((s,a))(v)\cdot v(s_1) \right) = \pi(\Upsilon_e(oh^{\ptwo}),a)(s,a)(s_1) = \pi(oh^e,a)(s,a)(s_1)$. So, 
        \begin{align*}
            \sum_{h' \in H'}\probm_{\sg}^{\mdpsga_a(\sigma), \mdpsga_e(\pi)} (\Cone(h')) & = \probm_{\rmdp}^{\sigma, \pi} (\Cone(h_1)) \times \sum_{a\in A}\left(\sigma(oh^a)(a) \times \pi(oh^e,a)(s,a)(s_1)\right) \\
            & = \probm_{\rmdp}^{\sigma, \pi} (\Cone(h))
        \end{align*}
        The converse cylinder identity follows by the same induction, replacing the vertex distribution by its barycenter.  The second identity of Lemma~\ref{mdpsg_sp} shows that this replacement preserves every next-state probability. Equality on all cylinders yields both asserted measure equalities and hence both objective-probability equalities.

    \end{proof}
\end{lemma}

\begin{lemma}
\label{lem:rg_vp}

Let $W \subseteq \Plays_\rmdp$ be a parity objective of $\rmdp$ and let $W' = \Theta(W)$ be the corresponding objective in $\sg$. Then,
\[
        \val_\rmdp(W) = \val_{\sg}(W').
\]

\begin{proof}
        For fixed $\sigma\in\Sigma_\rmdp$, set
        \[
        q_\rmdp(\sigma)=\inf_{\pi\in\Pi_\rmdp}\probm_\rmdp^{\sigma,\pi}(W),
        \qquad
        q_\sg(\sigma')=\inf_{\pi'\in\Pi_\sg}\probm_\sg^{\sigma',\pi'}(W').
        \]
        The first equality in Lemma~\ref{mdpsg_vp} gives
        $q_\sg(\mdpsga_a(\sigma))\le q_\rmdp(\sigma)$, because every $\pi\in\Pi_\rmdp$ has the image $\mdpsga_e(\pi)\in\Pi_\sg$. 
        \[
        q_\rmdp(\sigma)=\inf_{\pi\in\Pi_\rmdp}\probm_\rmdp^{\sigma,\pi}(W) = \inf_{\mdpsga_e(\pi)\in\Pi_\sg}\probm_{\sg}^{\mdpsga_a(\sigma),\mdpsga_e(\pi)}(W') \ge \inf_{\pi'\in\Pi_\sg}\probm_{\sg}^{\mdpsga_a(\sigma),\pi'}(W') = q_\sg(\mdpsga_a(\sigma))
        \]
        
        Conversely, for every $\pi'\in\Pi_\sg$, the second equality in Lemma~\ref{mdpsg_vp} and
        $\mdpsge_a(\mdpsga_a(\sigma))=\sigma$ give
        \[
        \probm_\sg^{\mdpsga_a(\sigma),\pi'}(W')
            =\probm_\rmdp^{\sigma,\mdpsge_e(\pi')}(W),
        \]
        hence 
        \[
        q_\sg(\mdpsga_a(\sigma)) = \inf_{\pi'\in\Pi_\sg}\probm_\sg^{\mdpsga_a(\sigma),\pi'}(W') = \inf_{\mdpsge_e(\pi')\in\Pi_\rmdp}\probm_\rmdp^{\sigma,\mdpsge_e(\pi')}(W) \ge \inf_{\pi\in\Pi_\rmdp}\probm_\rmdp^{\sigma,\pi}(W) = q_\rmdp(\sigma)
        \]
  Therefore
        \[
        q_\rmdp(\sigma)=q_\sg(\mdpsga_a(\sigma))
        \quad\text{for every $\sigma\in\Sigma_\rmdp$}.
        \]
        Finally, $\mdpsga_a$ is a bijection by Lemma~\ref{mdpsg_sp}, so taking the outer supremum yields
        \begin{align*}
        \val_\rmdp(W)
            &=\sup_{\sigma\in\Sigma_\rmdp}q_\rmdp(\sigma)
             =\sup_{\sigma\in\Sigma_\rmdp}q_\sg(\mdpsga_a(\sigma))\\
            &=\sup_{\sigma'\in\Sigma_\sg}q_\sg(\sigma')
             =\val_\sg(W').
        \end{align*}

\end{proof}
    
\end{lemma}

\begin{lemma}
\label{lem:rgsurewinning}
For mapped strategies, projection preserves the supported plays in both directions:
\begin{align*}
\Upsilon\!\left(\Plays^{\mdpsga_a(\sigma),\mdpsga_e(\pi)}_\sg\right)
  &=\Plays^{\sigma,\pi}_\rmdp,\\
\Upsilon\!\left(\Plays^{\sigma',\pi'}_\sg\right)
  &=\Plays^{\mdpsge_a(\sigma'),\mdpsge_e(\pi')}_\rmdp.
\end{align*}
\begin{proof}
We prove the first equality by induction on finite effective prefixes; the second follows by the same argument using the reverse maps. The initial prefixes coincide. Suppose corresponding prefixes end in the RPOMDP state $s$ and the POSG max-state $s$. Their agent observation histories correspond under $\Theta_a$ and $\Upsilon_a$, and
\[
\mdpsga_a(\sigma)(\Theta_a(oh^a))=\sigma(oh^a).
\]
Consequently, an action $a$ has positive probability after one prefix exactly when it has positive probability after the other. The POSG then moves deterministically to $(s,a)$.

At the min-state $(s,a)$, let
$\lambda=\mdpsga_e(\pi)(\Theta_e(oh^e))((s,a))$ and
$P=\pi(oh^e,a)(s,a)=\mathsf{bar}_{s,a}(\lambda)$. For every successor $s'$, nonnegativity gives
\[
P(s')>0
\quad\Longleftrightarrow\quad
\lambda(v)v(s')>0\text{ for some }v\in V_{s,a}.
\]
Thus an RPOMDP extension to $s'$ is supported exactly when some POSG extension through an effective vertex $v$ to $s'$ is supported. This proves the induction step and the first projected-prefix equality. For arbitrary $\sigma',\pi'$, replace $\lambda$ by $\pi'(oh)((s,a))$ and use the barycenter identity in Lemma~\ref{mdpsg_sp}; the same induction proves the second equality. Finally, an infinite play is supported exactly when all its finite prefixes are supported, yielding both displayed play-set equalities. 
\end{proof}
\end{lemma}

\begin{lemma}[Sure winning analysis]
\label{lem:rg_sw}
        Let $W \subseteq \Plays_\rmdp$ be a parity objective of $\rmdp$ and $W'=\Upsilon^{-1}(W)=\Theta(W)$ its lifted objective in $\sg$. The preimage $\Upsilon^{-1}$ was defined in Appendix~\ref{app:map}. For $\sigma\in\Sigma_\rmdp$, let $\sigma'=\mdpsga_a(\sigma)$. Then,
        \[\sigma \text{ is sure winning for } W \iff \sigma' \text{ is sure winning for } W'.\]
    \begin{proof}
        Suppose first that $\sigma$ is sure winning and take arbitrary $\pi'\in\Pi_\sg$.  By Lemma~\ref{lem:rgsurewinning},
        \[
        \Upsilon\!\left(\Plays^{\sigma',\pi'}_\sg\right)
            =\Plays^{\sigma,\mdpsge_e(\pi')}_\rmdp\subseteq W,
        \]
        so $\Plays^{\sigma',\pi'}_\sg\subseteq\Upsilon^{-1}(W)=W'$.  Conversely, suppose that $\sigma'$ is sure winning and take arbitrary $\pi\in\Pi_\rmdp$.  The first equality in Lemma~\ref{lem:rgsurewinning} gives
        $\Upsilon(\Plays^{\sigma',\mdpsga_e(\pi)}_\sg)=\Plays^{\sigma,\pi}_\rmdp$; since the game plays lie in $W'$, their projections lie in $W$.
    \end{proof}
\end{lemma}

\begin{lemma}[Almost-sure analysis]
\label{lem:rg_as}
Let $W \subseteq \Plays_\rmdp$ be a parity objective of $\rmdp$ and $W'=\Theta(W)$ its lifted objective in $\sg$. For $\sigma\in\Sigma_\rmdp$, let $\sigma'=\mdpsga_a(\sigma)$. Then,
        \[\sigma \text{ is almost-sure winning for } W \iff \sigma' \text{ is almost-sure winning for } W'.\]
    \begin{proof}
        The fixed-strategy equality established in the proof of Lemma~\ref{lem:rg_vp} gives
        \[
        \inf_{\pi\in\Pi_\rmdp}\probm_\rmdp^{\sigma,\pi}(W)
            =\inf_{\pi'\in\Pi_\sg}\probm_\sg^{\sigma',\pi'}(W').
        \]
        One side equals $1$ if and only if the other does.
    \end{proof}
\end{lemma}

\begin{lemma}[Limit-sure Analysis]
\label{lem:rg_ls}
    Let $W \subseteq \Plays_\rmdp$ be a parity objective of $\rmdp$ and $W' = \Theta(W)$ be the corresponding objective in $\sg$. Then,
    \[W \text{ is limit-sure winning} \iff W' \text{ is limit-sure winning} \]
    \begin{proof}
    Follows directly from Lemma \ref{lem:rg_vp}.
    \end{proof}
\end{lemma}

\begin{lemma}[Quantitative Analysis]
\label{lem: rg_qa}
    Let $W \subseteq \Plays_\rmdp$ be a parity objective of $\rmdp$ and $W' = \Theta(W)$ be the corresponding objective in $\sg$. Then, for $k \in [0,1]$
    \[\Val_\rmdp(W)\geq k \iff \Val_\sg(W') \geq k \]
    \begin{proof}
    Follows directly from Lemma \ref{lem:rg_vp}.
    \end{proof}
\end{lemma}

\subsection{Proof of Theorem \ref{rpomdp-posg}}
\label{app:thm1}
\begin{proof}
    Because parity  subsumes all $\omega$-regular objectives, from Lemma \ref{lem:rg_vp}, $\Val_\rmdp(W) = \Val_\sg(W')$ for any $\omega$-regular objective $W$ in $\rmdp$ and its lifted objective $W'$ in $\sg$.\\
    The ``consequently'' statement follows from lemmas \ref{lem:rg_sw}, \ref{lem:rg_as}, \ref{lem:rg_ls}, and \ref{lem: rg_qa}. 
\end{proof}

\clearpage 

\section{The Pre-$\ptwo$-transformation from ac-POSG $\sg$ to $\sg'$}
\label{app:reverse}
In this section, we begin  with the transformation from $\sg$ to $\sg'$.

Consider any \emph{alternating control} POSG $\sg$.  In $\sg$ the $\pone$-player selects an
action from a $\pone$-state and the output is immediately a $\ptwo$-state (and vice versa).
To make the subsequent RPOMDP construction cleaner, we first insert a dummy $\pone$-state
$d_s$ before every $\ptwo$-state $s \in \stwo$, so that the $\pone$-player always acts
\emph{twice} in a row before the $\ptwo$-player acts.  The $\pone$-action at a dummy
state is always the fixed action $\anew$ with a deterministic transition to the corresponding
$\ptwo$-state; it therefore carries no strategic content.

\begin{definition}[Pre-$\ptwo$ transformation]
\label{def:maxtd-transform}

Given an \emph{alternating control} POSG
$\sg = (\sone, \stwo, \aone, \atwo, \tran, s_0, \obsone, \obstwo, \obfnone,
\obfntwo)$, the \emph{Pre-$\ptwo$ transformation} of $\sg$ is the POSG
$\sg' = (\sone', \stwo', \aone', \atwo', \tran', s_0, \obsone', \obstwo',
\obfnone', \obfntwo')$ where:

\begin{itemize}
    \item \textit{State spaces:}
    \[
        \sone' \;=\; \sone \;\cup\; \underbrace{\{\, d_s \mid s \in \stwo \,\}}_{D},
        \qquad
        \stwo' \;=\; \stwo,
    \]
    where $D$ is a set of fresh dummy $\pone$-states, one for each $\ptwo$-state.

    \item \textit{Actions:}
    \[
        \aone' \;=\; \aone \;\cup\; \{\anew\},
        \qquad
        \atwo' \;=\; \atwo,
    \]
    where $\anew\notin\aone$ is fresh. The effective alphabets are $\aone'(s)=\aone(s)$ at original max-states, $\aone'(d_s)=\{\anew\}$, and $\atwo'(s)=\atwo(s)$ at original min-states. Every strategy uses these alphabets by definition.

    \item \textit{Transition function $\tran'$:}
    \begin{itemize}
        \item For $s_1 \in \sone$, $a_1 \in \aone(s_1)$, and $s_2 \in \stwo$:
              $\tran'(s_1, a_1)(d_{s_2}) = \tran(s_1, a_1)(s_2)$.
        \item For $s_2 \in \stwo$ and $a_2 \in \atwo(s_2)$:
              $\tran'(s_2, a_2) = \tran(s_2, a_2)$.
        \item For $d_s \in D$ and action $\anew$:
              $\tran'(d_s, \anew) = \dirac{s}$.
    \end{itemize}

    \item \textit{Observations:}
    \[
        \obsone' \;=\; \obsone \;\cup\; \underbrace{\{\,\obsnewone_x \mid x \in \obsone\,\}}_{\Obsnewone},
        \qquad
        \obstwo' \;=\; \obstwo \;\cup\; \underbrace{\{\,\obsnewtwo_x \mid x \in \obstwo\,\}}_{\Obsnewtwo}.
    \]

    \item \textit{Observation functions:}
    \begin{itemize}
        \item $\forall s \in \sone \cup \stwo$:
              $\obfnone'(s) = \obfnone(s)$ and $\obfntwo'(s) = \obfntwo(s)$.
        \item $\forall d_s \in D$:
              $\obfnone'(d_s) = \obsnewone_{\obfnone(s)}$ and
              $\obfntwo'(d_s) = \obsnewtwo_{\obfntwo(s)}$.
    \end{itemize}
\end{itemize}
\end{definition}

\subsection{Mapping plays, observation histories and strategies from $\sg$ to $\sg'$}
\label{app:map2}
\smallskip\noindent\textbf{Mapping on plays.}
We define a map $\Gamma: \Plays_\sg \to \Plays_{\sg'}$.  Let
$h = s_0, a_0, s_1, a_1, s_2, a_2, \dots \in \Plays_\sg$.
Since $\sg$ is alternating control, even-indexed states lie in $\sone$ and odd-indexed states
lie in $\stwo$.  We set $\Gamma(h) = h' = s'_0,a'_0,s'_1,a'_1,\cdots$ index-wise as: For $k \geq 0$,
\begin{align*}
    s'_{3k} &= s_{2k}\\
    a'_{3k} &= a_{2k}\\
    s'_{3k+1} &= d_{s_{2k+1}}\\
    a'_{3k+1} &= \anew\\
    s'_{3k+2} &= s_{2k+1}\\
    a'_{3k+2} &= a_{2k+1}
\end{align*}

For $X \subseteq \Plays_\sg$, define $\Gamma(X) = \{\Gamma(\rho) \mid \rho\in X\}$.

\begin{observation}
By construction of $\sg'$, $\Gamma(h) \in \Plays_{\sg'}$ for every $h \in \Plays_\sg$, so
$\Gamma$ is well-defined.
\end{observation}

\begin{observation}
Every play $h' = s_0,a_0,s_1,a_1,s_2,a_2,\cdots \in \Plays_{\sg'}$ satisfies, for each $k \geq 0$:
\[
s_{3k} \in \sone, \quad a_{3k} \in \aone,
\quad s_{3k+1} \in D, \quad a_{3k+1} = \anew, \quad s_{3k+2} \in \stwo,\quad a_{3k+2} \in \atwo
\]
\end{observation}

\begin{lemma}
Let $h' \in \Plays_{\sg'}$ and let $h$ be obtained by removing all occurrences of states in
$D$ and the action $\anew$ from $h'$.  Then $h \in \Plays_\sg$ and $\Gamma(h) = h'$.
\end{lemma}

\begin{lemma}
\label{lemma:gammabij}
The map $\Gamma: \Plays_\sg \to \Plays_{\sg'}$ is a bijection.
\end{lemma}

\smallskip\noindent\textbf{Mapping on observation histories.}
We extend $\Gamma$ to observation histories.  For the $\pone$-player, define
$\GammaOpone: \OH^\sg_\pone \to \OH^{\sg'}_\pone$ as follows.  Let
$oh = x_0, y_0, \dots, x_{n-1}, y_{n-1}, x_{n} \in \OH^\sg_\pone$. Then $\GammaOpone(oh) = x_0, z_0, y_0, x_1, z_1, y_1, \dots, x_{n-1},z_{n-1}, y_{n-1},  x_n $  where each $z_i =\obsnewone_{y_{i}} $.

Since $\GammaOpone(oh)$ cannot end with a $\obsnewone$-observation by construction, the map
is not surjective.  We nevertheless define an inverse
$(\GammaOpone)^{-1}: \OH^{\sg'}_\pone \to \OH^\sg_\pone$ by removing all occurrences of
observations in $\Obsnewone$.  We define $\GammaOptwo$ and
$(\GammaOptwo)^{-1}$ for the $\ptwo$-player analogously by removing all occurrences of observations in $\Obsnewtwo$.

\begin{lemma}
\label{lemma:ginv-g}
Let $\star \in \{\pone, \ptwo\}$.  For all $oh \in \OH^\sg_\star$,
$(\GammaOstar)^{-1}(\GammaOstar(oh)) = oh$.
\end{lemma}

\begin{lemma}
\label{lemma:g-ginv}
For all $oh' = (oh'_1, o) \in \OH^{\sg'}_\star$,
\[
\GammaOstar\!\bigl((\GammaOstar)^{-1}(oh'_1, o)\bigr) =
\begin{cases}
oh' & \text{if } o \notin  \Obsnewone\cup\Obsnewtwo, \\
oh'_1 & \text{otherwise.}
\end{cases}
\]

\end{lemma}

\smallskip\noindent\textbf{Mapping of strategies.}
As in the main text, max-strategies output a single action distribution, whereas min-strategies output a prescription:
\[
\sigma:\OH_\pone^\sg\to\dist(\aone),\qquad
\pi:\OH_\ptwo^\sg\to\bigcup_{o\in\obstwo}
\prod_{s\in\stwo(o)}\dist(\atwo).
\]
The corresponding types in $\sg'$ use $\aone'$ and $\atwo'$. All distributions obey the effective-action support convention. In particular every max-strategy in $\sg'$ plays $\anew$ at a dummy observation, not only strategies produced by the lifting map.

We define the following four maps lifting strategies between $\sg$ and $\sg'$:
\begin{enumerate}
    \item $\Phi_\pone : \stratpone_\sg \to \stratpone_{\sg'}$: for each
            $\sigma \in \stratpone_{\sg}$ and $oh'\in\OH_\pone^{\sg'}$, define
          \[
            \Phi_\pone(\sigma)(oh')=
            \begin{cases}
            \sigma((\GammaOpone)^{-1}(oh'))&\Last(oh')\notin\Obsnewone,\\
            \dirac{\anew}&\Last(oh')\in\Obsnewone.
            \end{cases}
          \]

    \item $\Phi_\ptwo : \stratptwo_{\sg} \to \stratptwo_{\sg'}$: for each
          $\pi \in \stratptwo_\sg$, $oh' \in \OH^{\sg'}_\ptwo$ and $\forall s \in \stwo$ such that $\obfntwo'(s) = \Last(oh')$, then
          \[
          \Phi_\ptwo(\pi)(oh')(s) = \pi\!\left((\GammaOptwo)^{-1}(oh')\right)(s)
          \]

    \item $\Psi_\pone : \stratpone_{\sg'} \to \stratpone_\sg$: for each
            $\sigma' \in \stratpone_{\sg'}$ and $oh \in \OH^\sg_\pone$, define
          \[
            \Psi_\pone(\sigma')(oh) = \sigma'\!\left(\GammaOpone(oh)\right).
          \]

    \item $\Psi_\ptwo : \stratptwo_{\sg'} \to \stratptwo_\sg$: for each
          $\pi' \in \stratptwo_{\sg'}$, $oh \in \OH_\ptwo^\sg$ and $\forall s \in \stwo$ such that $\obfntwo(s) = \Last(oh)$, then
          \[
          \Psi_\ptwo(\pi')(oh)(s) = \pi'\!\left(\GammaOptwo(oh)\right)(s).
          \]
\end{enumerate}

\begin{lemma}
\label{lemma:phi-psi-inverse}
For $\star \in \{\pone, \ptwo\}$, $\Phi_\star$ and $\Psi_\star$ are inverses of each other:
$\Phi_\star = (\Psi_\star)^{-1}$ and $\Psi_\star = (\Phi_\star)^{-1}$.
\begin{proof}
For max, Lemma~\ref{lemma:ginv-g} gives, as an equality of action distributions,
\[
\Psi_\pone(\Phi_\pone(\sigma))(oh) = \Phi_\pone(\sigma)(\GammaOpone(oh))
=\sigma((\GammaOpone)^{-1}(\GammaOpone(oh)))=\sigma(oh).
\]
Conversely, at a history $oh'$ ending at an original max-state, Lemma~\ref{lemma:g-ginv} gives
\[
\Phi_\pone(\Psi_\pone(\sigma'))(oh') = \Psi_\pone(\sigma')((\GammaOpone)^{-1}(oh'))
=\sigma'(\GammaOpone((\GammaOpone)^{-1}(oh')))=\sigma'(oh').
\]
At a dummy history, both sides equal $\dirac{\anew}$ by the strategy definition. The copied distributions have the same effective alphabet at the corresponding original observation, so both maps are well-defined on all max-strategies.

For min, the maps preserve the local action alphabet at each original min-state. We first show
$\Psi_\ptwo(\Phi_\ptwo(\pi))(oh)(s) = \pi(oh)(s)$ for all $\pi \in \stratptwo_\sg$, $oh \in \OH_\ptwo^\sg$ and $s \in \stwo$ such that $\obfntwo(s) = \Last(oh)$.
\begin{align*}
\Psi_\ptwo(\Phi_\ptwo(\pi))(oh)(s)
  &= \Phi_\ptwo(\pi)\!\left(\GammaOptwo(oh)\right)(s) && \text{by definition of }  \Psi_\ptwo\\
  &= \pi\!\left((\GammaOptwo)^{-1}\!\left(\GammaOptwo(oh)\right)\right)(s) && \text{by definition of }  \Phi_\ptwo\\
   &= \pi(oh)(s)
     && \text{by Lemma \ref{lemma:ginv-g}.}
\end{align*}
Conversely, for all $\pi' \in \stratptwo_{\sg'}$, $oh' \in \OH_\ptwo^{\sg'}$ and $s \in \stwo$ such that $\obfntwo'(s) = \Last(oh')$:
\begin{align*}
\Phi_\ptwo(\Psi_\ptwo(\pi'))(oh')(s)
  &= \Psi_\ptwo(\pi')\!\left((\GammaOptwo)^{-1}(oh')\right)(s) && \text{by definition of }  \Phi_\ptwo
      \\
  &= \pi'\!\left(\GammaOptwo\!\left((\GammaOptwo)^{-1}(oh')\right)\right)(s) && \text{by definition of }  \Psi_\ptwo \\
  &= \pi'(oh')(s)
     && \text{by Lemma \ref{lemma:g-ginv} and  $\Last(oh') \notin \Obsnewtwo$.}
\end{align*}
\end{proof}
\end{lemma}
\subsection{Lifting objectives from $\sg$ to $\sg'$}
\label{app:lift2}
 Let $\prfn: \sone \cup \stwo \to \{0,1,\dots,d\}$ be
          the priority function for $\sg$.  Define $\prfn': \sone' \cup \stwo' \to \{0,1,\dots,d\}$
          by $\prfn'(s) = \prfn(s)$ for $s \in \sone \cup \stwo$ and
          $\prfn'(d_s) = \prfn(s)$ for $d_s \in D$.
    
\begin{lemma}
\label{lemma:gamma-parity}
Let $\sg$, $\sg'$, $\prfn$, and $\prfn'$ be as above, and let $\rho \in \Plays_\sg$.  Then
\[
\rho \in \parity{\sg}{\prfn} \iff \Gamma(\rho) \in \parity{\sg'}{\prfn'}.
\]
\begin{proof}

The play $\Gamma(\rho)$ visits every dummy state $d_s \in D$ exactly once between visiting
$s' \in \sone$ and $s \in \stwo$; $\prfn'(d_s) = \prfn(s)$, so no new priority value is
introduced.  The sequence of priorities seen along $\Gamma(\rho)$ is therefore a repetition
of the sequence seen along $\rho$ (each value appears twice in a row at the dummy-state
interleaving), which does not change the $\liminf$.  Hence $\Gamma(\rho)$ satisfies the
parity condition with respect to $\prfn'$ if and only if $\rho$ satisfies it with respect to
$\prfn$.
\end{proof}
\end{lemma}

\subsection{Value Preservation}
\label{app:value2}
\begin{lemma}
\label{lemma:phi-equiv}
The strategy maps preserve objective probabilities in both directions:
\begin{align*}
\probm_{\sg}^{\sigma,\pi}(W)
&=\probm_{\sg'}^{\Phi_\pone(\sigma),\Phi_\ptwo(\pi)}(W'),
&&\text{for all $\sigma\in\Sigma_\sg$, $\pi\in\Pi_\sg$},\\
\probm_{\sg'}^{\sigma',\pi'}(W')
&=\probm_{\sg}^{\Psi_\pone(\sigma'),\Psi_\ptwo(\pi')}(W),
&&\text{for all $\sigma'\in\Sigma_{\sg'}$, $\pi'\in\Pi_{\sg'}$},
\end{align*}
where $W \subseteq \Plays_\sg$ is a parity objective of $\sg$ and $W' = \Gamma(W)$ is its corresponding objective in $\sg'$.

\begin{proof}
Let $\rho \in W$ be a play of $\sg$ and let $\rho' = \Gamma(\rho) \in W'$.  It suffices to
show that for every finite history $h$ of every play in $W$  the corresponding history
$h'$ (constructed index-wise exactly as $\Gamma$ constructs $\rho'$ from
$\rho$, but truncated to end at the same state as $h$) satisfies

\[
\probm_{\sg}^{\sigma,\pi}(\cyl{h})
= \probm_{\sg'}^{\Phi_\pone(\sigma),\Phi_\ptwo(\pi)}(\cyl{h'}).
\]
Since the probability measure over plays is uniquely determined by its values on cylinder
sets, this suffices.
The proof proceeds by induction on $|h|$. The base case $h=s_0$
holds trivially since $h' = s_0$ and both cylinders have probability $1$.

For the inductive step, write $h = h_1, s$ and distinguish two cases.

\smallskip
\noindent\textbf{Case 1: $s \in \sone$.}
Then $\Last(h_1) \in \stwo$.  By construction, $h' = h'_1,s$ where
$h'_1$ is the prefix corresponding to $h_1$ and $oh_1' = \obfntwo'(h_1')$.  Let
$oh_1 = \obfntwo(h_1)$ and $s_1 = \Last(h'_1)$.
\begin{align*}
\probm_{\sg'}^{\Phi_\pone(\sigma),\Phi_\ptwo(\pi)}(\cyl{h'})
&= \probm_{\sg'}^{\Phi_\pone(\sigma),\Phi_\ptwo(\pi)}(\cyl{h'_1})
    \times \sum_{a \in \atwo(s_1)} \left( \Phi_\ptwo(\pi)(oh'_1)(s_1)(a)
  \times \tran'(s_1, a)(s) \right).
\end{align*}
By the induction hypothesis,
$\probm_{\sg'}^{\Phi_\pone(\sigma),\Phi_\ptwo(\pi)}(\cyl{h'_1}) = \probm_{\sg}^{\sigma,\pi}(\cyl{h_1})$. For $a \in \atwo(s_1):$ 
By definition of $\Phi_\ptwo$ and $oh_1 = (\GammaOptwo)^{-1}(oh'_1)$ we have $ \Phi_\ptwo(\pi)(oh'_1)(s_1)(a) = \pi(oh_1)(s_1)(a)$ and $\Last(h_1) = \Last(h'_1) = s_1$ by construction and
$\tran' = \tran$ on original effective min state--action pairs, we have
$ \tran'(s_1, a)(s) = \tran(s_1, a)(s)$. Hence
\begin{align*}
    \probm_{\sg'}^{\Phi_\pone(\sigma),\Phi_\ptwo(\pi)}(\cyl{h'}) &= \probm_{\sg}^{\sigma,\pi}(\cyl{h_1})
    \times \sum_{a \in \atwo(s_1)} \left( \pi(oh_1)(s_1)(a)
  \times \tran(s_1, a)(s) \right) \\
  &= \probm_{\sg}^{\sigma,\pi}(\cyl{h})
\end{align*}

\smallskip
\noindent\textbf{Case 2: $s \in \stwo$.}
Then $\Last(h_1) \in \sone$.  By construction,
$h' = h'_1, d_s, s$, so
$\obfnone'(h') = oh'_1, \obsnewone_o, o$ where $o = \obfnone(s)$ and $oh'_1 = \obfnone'(h'_1)$  .  Let
$oh_1 = \obfnone(h_1)$ and $s_1 = \Last(h'_1) = \Last(h_1)$.
\begin{align*}
\probm_{\sg'}^{\Phi_\pone(\sigma),\Phi_\ptwo(\pi)}(\cyl{h'})
&= \probm_{\sg'}^{\Phi_\pone(\sigma),\Phi_\ptwo(\pi)}(\cyl{h'_1})\\
&  \times \sum_{a \in \aone(s_1)} \left(\Phi_\pone(\sigma)(oh'_1)(a) \times \tran'(s_1, a)(d_s)\right)\\
 & \times \Phi_\pone(\sigma)(oh'_1, \obsnewone_o)(\anew)
  \times \tran'(d_s, \anew)(s).
\end{align*}
By the induction hypothesis the $\probm_{\sg'}^{\Phi_\pone(\sigma),\Phi_\ptwo(\pi)}(\cyl{h'_1}) = \probm_{\sg}^{\sigma,\pi}(\cyl{h_1})$.
For $a \in \aone(s_1)$ and $oh_1 = (\GammaOpone)^{-1}(oh'_1)$, the strategy definition gives $\Phi_\pone(\sigma)(oh'_1)(a) = \sigma(oh_1)(a)$. By construction,
$\tran'(s_1, a)(d_s) = \tran(s_1, a)(s)$.  Since $\anew$ is the sole effective action
at any dummy state, $\Phi_\pone(\sigma)(oh'_1, \obsnewone_o)(\anew) = 1$.
Finally, $\tran'(d_s, \anew)(s) = 1$ by construction of $\sg'$.  Hence
\begin{align*}
    \probm_{\sg'}^{\Phi_\pone(\sigma),\Phi_\ptwo(\pi)}(\cyl{h'})
&= \probm_{\sg}^{\sigma,\pi}(\cyl{h_1})
    \times \sum_{a \in \aone(s_1)} \left(\sigma(oh_1)(a) \times \tran(s_1, a)(s)\right)  \times 1  \times 1 \\
&= \probm_{\sg}^{\sigma,\pi}(\cyl{h})
\end{align*}

The symmetric direction (replacing $\Phi_\star$ by $\Psi_\star = (\Phi_\star)^{-1}$) follows
by an identical argument.
\end{proof}
\end{lemma}

We now state the main correctness theorem for Pre-$\ptwo$ transformation.

\begin{lemma}
\label{thm:prenorm}
Let $\sg$ be any alternating control POSG and $\sg'$ its Pre-$\ptwo$ transformation.  Let $W \subseteq \Plays_\sg$ be parity objective and $W' = \Gamma(W)$ its
corresponding objective in $\sg'$. Then
\[
\val_\sg(W) = \val_{\sg'}(W').
\]
\begin{proof}
For fixed $\sigma\in\Sigma_\sg$, set
\[
q_\sg(\sigma)=\inf_{\pi\in\Pi_\sg}\probm_\sg^{\sigma,\pi}(W),
\qquad
q_{\sg'}(\sigma')=\inf_{\pi'\in\Pi_{\sg'}}\probm_{\sg'}^{\sigma',\pi'}(W')
\]
The first equality in Lemma~\ref{lemma:phi-equiv} gives
$q_{\sg'}(\Phi_\pone(\sigma))\leq q_\sg(\sigma)$ because every $\pi\in\Pi_\sg$ has an image $\Phi_\ptwo(\pi)\in\Pi_{\sg'}$. 
\[
q_\sg(\sigma)=\inf_{\pi\in\Pi_\sg}\probm_\sg^{\sigma,\pi}(W) = \inf_{\Phi_\ptwo(\pi)\in\Pi_{\sg'}}\probm_{\sg'}^{\Phi_\pone(\sigma),\Phi_\ptwo(\pi)}(W') \geq \inf_{\pi'\in\Pi_{\sg'}}\probm_{\sg'}^{\Phi_\pone(\sigma),\pi'}(W') = q_{\sg'}(\Phi_\pone(\sigma))
\]
Conversely, for every $\pi'\in\Pi_{\sg'}$, the second equality there and $\Psi_\pone(\Phi_\pone(\sigma))=\sigma$ give
\[
\probm_{\sg'}^{\Phi_\pone(\sigma),\pi'}(W')
=\probm_\sg^{\sigma,\Psi_\ptwo(\pi')}(W),
\]
so
\[
q_{\sg'}(\Phi_\pone(\sigma)) = \inf_{\pi'\in\Pi_{\sg'}}\probm_{\sg'}^{\Phi_\pone(\sigma),\pi'}(W') = \inf_{\Psi_\ptwo(\pi')\in\Pi_{\sg}}\probm_\sg^{\sigma,\Psi_\ptwo(\pi')}(W) \geq \inf_{\pi\in\Pi_{\sg}}\probm_\sg^{\sigma,\pi}(W) = q_\sg(\sigma)
\]
Hence $q_\sg(\sigma)=q_{\sg'}(\Phi_\pone(\sigma))$. Since $\Phi_\pone$ is a bijection by Lemma~\ref{lemma:phi-psi-inverse}, taking the outer supremum yields
\begin{align*}
\val_\sg(W)
    &=\sup_{\sigma\in\Sigma_\sg}q_\sg(\sigma)
     =\sup_{\sigma\in\Sigma_\sg}q_{\sg'}(\Phi_\pone(\sigma))\\
    &=\sup_{\sigma'\in\Sigma_{\sg'}}q_{\sg'}(\sigma')
     =\val_{\sg'}(W').
\end{align*}

\end{proof}
\end{lemma}

\begin{corollary}
\label{cor:prenorm}
Let $\sg$ be any alternating control POSG and $\sg'$ its Pre-$\ptwo$ transformation.  Let $\prfn$ be a priority function for $\sg$.
Then
\[
\val_\sg\!\left(\parity{\sg}{\prfn}\right) = \val_{\sg'}\!\left(\parity{\sg'}{\prfn'}\right).
\]
\begin{proof}
Follows directly from lemma \ref{thm:prenorm} together with Lemma \ref{lemma:gamma-parity}.
\end{proof}

\end{corollary}

\begin{lemma}
\label{cor:ggsurewinning}
    For any $\sigma \in \stratpone_{\sg}$ and $\pi \in \stratptwo_{\sg}$:
\[
    \Gamma\bigl(\Plays^{\sigma,\pi}_{\sg}\bigr)
    \;=\;
    \Plays^{\Phi_\pone(\sigma),\,\Phi_\ptwo(\pi)}_{\sg'}.
\]
 where $\Plays^{\sigma,\pi}_\sg(resp.\Plays^{\sigma',\pi'}_{\sg'})$ denote the set of plays in $\sg(resp.\sg')$ under the pair of strategies $\sigma,\pi(resp. \sigma',\pi')$.
 \begin{proof} Let $\rho' \in \Gamma(\Plays^{\sigma,\pi}_{\sg})$.  Then there exists $\rho \in \Plays^{\sigma,\pi}_{\sg}$ with $\Gamma(\rho) = \rho'$. Write $\rho = s_0, a_0, s_1, a_1, \ldots$ and $\rho' = s_0, a_0, d_{s_1}, \anew, s_1, a_1, \ldots$ as constructed by $\Gamma$. We verify $\rho'$ is consistent with $(\Phi_\pone(\sigma), \Phi_\ptwo(\pi))$.
\begin{itemize}
  \item {At each \textbf{$\pone$-state} $s_{k} \in \sone \text{ (k is even)}$:} The $\pone$-player observation history in $\sg'$ at state $s_k$ is $oh'_k = \GammaOpone(oh_k)$, where $oh_k$ is the $\pone$-player
        observation history of $\rho$ at $s_k$. By definition of $\Phi_\pone$:
        \begin{align*}
                        \Phi_\pone(\sigma)(oh'_k)& = \sigma\!\left((\GammaOpone)^{-1}(oh'_k)\right)\\
                    &= \sigma(oh_k) & \text{from Lemma $\ref{lemma:ginv-g}$}
        \end{align*}
        Because $\rho \in \Plays^{\sigma,\pi}_{\sg}$, action $a_k$ is in the support of $\sigma(oh_k)$, hence also of $\Phi_\pone(\sigma)(oh'_k)$. The transition in $\sg'$ sends $(s_k,a_k)$ to $d_{s_{k+1}}$ with probability $\tran'(s_k,a_k)(d_{s_{k+1}})=\tran(s_k,a_k)(s_{k+1})$, matching the source play.

    \item At each\textbf{ dummy state $d_{s_{2k+1}} \in D$:} By definition of $\Phi_\pone$, the strategy assigns $\dirac{\anew}$ to the sole effective action; all other globally enabled actions are max-losing defaults. Moreover, $\tran'(d_{s_{2k+1}}, \anew)(s_{2k+1})=1$.
  \item At each\textbf{ $\ptwo$-state $s_{k} \in \stwo$} (k is odd): The $\ptwo$-player observation history in $\sg'$ at $s_{k}$ is $oh' = \GammaOptwo(oh)$, where $oh$ is the $\ptwo$-player observation history of $\rho$ at state $s_{k}$. By definition of $\Phi_\ptwo$:
  \begin{align*}
      \Phi_\ptwo(\pi)(oh')(s_k) &= \pi\!\left((\GammaOptwo)^{-1}(oh')\right)(s_k)\\
          &= \pi(oh)(s_k) & \text{from lemma } \ref{lemma:ginv-g}
  \end{align*}
  Because $\rho \in \Plays^{\sigma,\pi}_\sg$, action $a_{k}$ is in the support of $\pi(oh)(s_{k})$, hence action $a_{k}$ is also in the support of $\Phi_\ptwo(\pi)(oh')(s_{k})$.
\end{itemize}
Thus $\rho' \in \Plays^{\Phi_\pone(\sigma),\Phi_\ptwo(\pi)}_{\sg'}$, hence $\Gamma\bigl(\Plays^{\sigma,\pi}_{\sg}\bigr) \subseteq \Plays^{\Phi_\pone(\sigma),\,\Phi_\ptwo(\pi)}_{\sg'}$.\\
Since $\Gamma: \Plays_\sg \to \Plays_{\sg'}$ is a bijection (Lemma $\ref{lemma:gammabij}$), every $\rho' \in \Plays^{\Phi_\pone(\sigma),\Phi_\ptwo(\pi)}_{\sg'}$ has a unique pre-image $\rho = \Gamma^{-1}(\rho')$ obtained by removing all dummy states and $\anew$-actions. By the same step-by-step argument with $(\GammaOpone)^{-1}$, $(\GammaOptwo)^{-1}$ in place of $\GammaOpone$, $\GammaOptwo$, and using $\Psi_\star = (\Phi_\star)^{-1}$ in place of $\Phi_\star$, the play $\rho$ is consistent with $(\sigma, \pi)$, hence $\Plays^{\Phi_\pone(\sigma),\,\Phi_\ptwo(\pi)}_{\sg'} \subseteq \Gamma\bigl(\Plays^{\sigma,\pi}_{\sg}\bigr)$ .
 \end{proof}
\end{lemma}

\begin{lemma}[Sure Winning]
\label{lem:gg_sw}
Let $W \subseteq \Plays_\sg$ be a parity objective of $\sg$ and $W' = \Gamma(W)$ be its lifted objective in $\sg'$. Let $\sigma \in \Sigma_\sg$ and $\sigma' = \Phi_\pone(\sigma)$.  Then, \[ \text{$\sigma$ is sure winning for }W  \iff \quad \text{$\sigma'$ is sure winning for } W'\]
    \begin{proof}
    We will show that
    \[
     \forall \pi \in \Pi_\sg :\:\Plays^{\sigma,\pi}_\sg \subseteq W \iff  \forall \pi' \in \Pi_{\sg'}  :\: 
    \Plays^{\sigma',\pi'}_{\sg'} \subseteq W'
    \]
    For forward direction,      let $\sigma \in \Sigma_\sg$ be sure winning for $W$. Then, $\forall \pi \in \Pi_\sg :\:\Plays^{\sigma,\pi}_\sg \subseteq W$.  From lemma \ref{lemma:gamma-parity}, for any $\rho \in \Plays^{\sigma,\pi}_\sg, \rho \in W \iff \Gamma(\rho) \in W'$. Since $\rho$ is arbitrary, we have $\Gamma(\Plays^{\sigma,\pi}_\sg) \subseteq W'$.\\
    For any $\pi' \in \Pi_{\sg'}$: Due to bijection of $\Phi_\ptwo,\Psi_\ptwo, \exists \pi$ such that $\pi' = \Phi_\ptwo(\pi)$. Then, from lemma \ref{cor:ggsurewinning} and \ref{lemma:phi-psi-inverse}
    \[\Plays^{\sigma',\pi'}_{\sg'} \:=\: \Gamma(\Plays^{\Psi_\pone(\sigma'),\Psi_\ptwo(\pi')}_\sg) = \Gamma(\Plays^{\Psi_\pone(\Phi_\pone(\sigma)),\Psi_\ptwo(\Phi_\ptwo(\pi))}_\sg) = \Gamma(\Plays^{\sigma,\pi}_\sg)\:\subseteq\: W'\]
    
    The reverse direction can be proven similarly using inverse mappings $\Psi_\pone,\Psi_\ptwo$.
       
    \end{proof}
\end{lemma}

\begin{lemma}[Almost-sure Winning]
\label{lem:gg_as}
    Let $W \subseteq \Plays_\sg$ be a parity objective of $\sg$ and $W' = \Gamma(W)$ be lifted objective in $\sg'$. Let $\sigma \in \Sigma_\sg$ and $\sigma' = \Phi_\pone(\sigma)$. Then, \[ \text{$\sigma$ is almost-sure winning for }W  \iff \quad \text{$\sigma'$ is almost-sure winning for } W'\]
    \begin{proof}
        The fixed-strategy equality established in the proof of Lemma~\ref{thm:prenorm} is
    \[
        \inf_{\pi\in\Pi_\sg}\probm_\sg^{\sigma,\pi}(W)
            =\inf_{\pi'\in\Pi_{\sg'}}\probm_{\sg'}^{\sigma',\pi'}(W').
        \]
        Hence one side equals $1$ if and only if the other does.
    \end{proof}
\end{lemma}

\begin{lemma} [Limit-Sure Winning]
\label{lem:gg_ls}
    Let $W \subseteq \Plays_\sg$ be a parity objective of $\sg$ and $W' = \Gamma(W)$ is the corresponding objective in $\sg'$. Then,
    \[W \text{ is limit-sure} \iff W' \text{ is limit-sure} \]
    \begin{proof}
    Follows directly from lemma \ref{thm:prenorm}.
    \end{proof}
\end{lemma}

\begin{lemma}[Quantitative Analysis]
\label{lem: gg_qa}
    Let $W \subseteq \Plays_\sg$ be a parity objective of $\sg$ and $W' = \Gamma(W)$ is the corresponding objective in $\sg'$. Then, for $k \in [0,1]$
    \[\Val_\sg(W)\geq k \iff \Val_{\sg'}(W') \geq k \]
    \begin{proof}
    Follows directly from lemma \ref{thm:prenorm}.
    \end{proof}
\end{lemma}

\subsubsection{ Size and Memory Preservation of the reduction}
\label{app:premin-size-memory}

\begin{lemma}
\label{lem:premin-size-memory}
Let $\sg$ be an alternating-control POSG and $\sg'$ its
Pre-$\ptwo$ transformation. The maps
$\Gamma : \Plays_\sg \to \Plays_{\sg'}$,
$\Gamma_O^\pone, \Gamma_O^\ptwo$ on observation histories, and
$\Phi_\star, \Psi_\star = (\Phi_\star)^{-1}$
($\star \in \{\pone, \ptwo\}$) on strategies satisfy the following. Exact counts use displayed effective transitions, excluding the optional completion. Let $S'=\sone'\cup\stwo'$ and $A'=\aone'\cup\atwo'$.
\begin{enumerate}\itemsep2pt
  \item \textup{(size)}
        $|S'| = |S_\sg| + |\stwo|$,
        $|A'| = |A_\sg| + 1$,
        $|\obsone'| = 2|\obsone|$,
        $|\obstwo'| = 2|\obstwo|$, and
        $|\delta'| = |\delta| + |\stwo|$.
        In particular, $|\sg'| = \mathcal{O}(|\sg|)$ and the
        construction runs in linear time.
  \item \textup{(history length)} For every
        $oh \in \OH^\sg_\star$,
        $|\Gamma_O^\star(oh)| = \left\lceil \tfrac{3|oh|-1}{2} \right\rceil$  and for every
        $oh' \in \OH^{\sg'}_\star$,
        $|(\Gamma_O^\star)^{-1}(oh')| \le |oh'|$. Observation
        histories on the two sides are linear in each other.
  \item \textup{(memory)} The lifting maps
        $\Phi_\star, \Psi_\star$ preserve memory class exactly: 
        memoryless, finite-memory
        and infinite-memory regimes correspond exactly
        between $\sg$ and $\sg'$.
\end{enumerate}

\begin{proof}
\textbf{Clause (1).} Direct from the construction.
The state space gains one fresh dummy state per $\stwo$-state, so
$|S'| = |\sone| + 2|\stwo| = |S_\sg| + |\stwo|$. The action set
gains the single fresh action $\anew$, so $|A'|=|A_\sg|+1$.
The observation alphabets are
$\obsone' = \obsone \cup \{o^\dagger_x \mid x \in \obsone\}$
and
$\obstwo' = \obstwo \cup \{o^\#_x \mid x \in \obstwo\}$,
which doubles each cardinality. The transition function $\delta'$
agrees with $\delta$ for $s\in\stwo$, $a\in\atwo(s)$, and for $s\in\sone$, $a\in\aone(s)$ replaces the target $s_2\in
\stwo$ by $d_{s_2} \in D$ -- giving the same number of edges as
$\delta$ -- and additionally contributes the deterministic edge
$\delta'(d_s, \anew) = \dirac{s}$ for every $d_s \in D$, of which
there are $|\stwo|$. Hence $|\delta'| = |\delta| + |\stwo|$.
Each component is therefore a constant or unit-additive perturbation
of the corresponding component in $\sg$, so
$|\sg'| = \mathcal{O}(|\sg|)$ and $\sg'$ is computed from $\sg$ in
linear time.

\medskip
\textbf{Clause (2).} Let $\star = \pone$ and let
$oh = o_0, o_1, \dots, o_{2n} \in \OH^\sg_\star$, so that
$|oh| = 2n+1$. By construction,
\[
  \Gamma_O^\star(oh)
  = o_0, z_1, o_1,o_2, z_3, o_3, \dots, o_{2n-2}, z_{2n-1}, o_{2n-1}, o_{2n},
\]
where $z_{2k+1}$ is the appropriately tagged copy of $o_{2k+1}$
($z_{2k+1} = o^\dagger_{o_{2k+1}}$).
One tagged
observation is inserted before every $\min$-player original observation. Hence
\[
  |\Gamma_O^\star(oh)| = (2n+1) + n = 3n+1 = \tfrac{3|oh|-1}{2}
\]

Let $\star = \ptwo$ and let
$oh = o_0, o_1, \dots, o_{2n+1} \in \OH^\sg_\star$, so that
$|oh| = 2n+2$. By construction,
\[
  \Gamma_O^\star(oh)
  = o_0, z_1, o_1,o_2, z_3, o_3, \dots, o_{2n-2}, z_{2n-1}, o_{2n-1}, o_{2n},z_{2n+1},o_{2n+1}
\]
where $z_{2k+1}$ is the appropriately tagged copy of $o_{2k+1}$
($z_{2k+1} = \obsnewtwo_{o_{2k+1}}$)
one tagged
observation is inserted between every pair of consecutive original
observations. Hence
\[
  |\Gamma_O^\star(oh)| = (2n+2) + n+1 = 3n+3 = \tfrac{3|oh|}{2}.
\]

From above values, for $\star \in \{\pone,\ptwo\}$, $|\Gamma^\star_O(oh)| = \left\lceil \tfrac{3|oh|-1}{2} \right\rceil$

For the reverse direction, $(\Gamma_O^\star)^{-1}$ acts on
$oh' \in \OH^{\sg'}_\star$ by deleting observations in $\Obsnewone$
for $\star=\pone$ and in $\Obsnewtwo$ for $\star=\ptwo$. Deletion only shortens the
sequence, so $|(\Gamma_O^\star)^{-1}(oh')| \le |oh'|$. Combining the
two bounds, observation histories on the two sides are within a
factor of $\tfrac{3}{2}$ in length.

\medskip
\textbf{Clause (3).} 
We prove that $\Phi_\pone$ preserves
memory. The arguments for $\Phi_\ptwo$, $\Psi_\pone$,
$\Psi_\ptwo$ are symmetric: replace $\obsone^\dagger$ by
$\obstwo^\#$ for the $\ptwo$ versions, and swap $\Gamma_O^\star$
with $(\Gamma_O^\star)^{-1}$ for the $\Psi$ versions.

Let $M^\sg_\pone,M^\sg_\ptwo$ denote the mealy machine for $\pone,\ptwo$-player strategies in $\sg$. Then, for $M^\sg_\pone = \{Q\cup \{ q_0\}, q_0,\obsone \cup (\obsone \times \obsone), \tau, \lambda\}$: $Q\cup q_0$ are finite set of memory states where $q_0$ is initial state. The transition function $\tau$ is defined as:
\begin{align*}
    \tau: &\: q_0 \times \obsone \rightarrow Q \\
    &\: Q \times \obsone \times \obsone \rightarrow Q
\end{align*}

The output function $\lambda$ is defined as:
\begin{align*}
    \lambda: &\: q_0 \times \obsone \rightarrow \dist(\aone) \\
    &\: Q \times \obsone \times \obsone \rightarrow \dist(\aone)
\end{align*}

Let $\sigma \in \stratpone_\sg$ be realized by a Mealy machine $M^\sg_\pone =  \{Q\cup q_0, q_0,\obsone \cup \obsone \times \obsone, \tau, \lambda\} $ as defined above. We build a mealy machine $M^{\sg'}_\pone = \{Q'\cup q'_0 \cup Q_d, q'_0,\obsone' \cup \obsone' \times \obsone', \tau', \lambda'\}$ representing $\Phi_\pone(\sigma)$ where $Q' = Q, q'_0 = q_0, Q_d = \{q_a \mid a \in Q\} $, for transition $\tau(q_0,o) = q_1$ converts to transition $\tau'(q_0,o) = q_{q_1}$ and $\tau(q_1,o_1,o_2) = q_2$ converts to 2 transitions $\tau'(q_{q_1},o_{o_1}) = q_1, \tau'(q_1,o_1,o_2) = q_{q_2}$. The output $\lambda,\lambda'$ are mapped as follows: $\lambda'(q_0,o) = \lambda(q_0,o)$ and $\lambda(q_1,o_1,o_2) = q_2$ gives two outputs as $ \lambda'(q_{q_1},o_{o_1})$ gives dirac distribution over $\top$, $\lambda'(q_1,o_1,o_2) = \lambda(q_1,o_1,o_2)$.\\

The mealy machine $M^{\sg'}_\pone$ realizes $\Phi_\pone(\sigma)$ because it simulates the behavior of $M^\sg_\pone$ on observation histories after applying the inverse mapping $(\Gamma^\pone_O)^{-1}$. By construction, every transition in $M^\sg_\pone$ is expanded in $M^{\sg'}_\pone$ wherever required by introducing copied states $Q_d$ that correspond to the play structure induced by $\Gamma$. The copied states $Q_d$ handle the $\Obsnewone$ observations while producing $\dirac{\top}$ when required. Because every observation history $oh'$ in $\sg'$ is reduced via $(\Gamma^\pone_O)^{-1}$ before applying the original strategy, the output function $\lambda'$ mirrors $\lambda$ exactly on the corresponding histories of G. Hence, the constructed machine produces the same action distributions as $\sigma$, establishing that it correctly realizes $\Phi_\pone(\sigma)$.\par

\noindent\textit{Special cases.} A memoryless strategy corresponds
to $|Q| = 1$, which the construction trivially preserves. An
infinite-memory strategy places no bound on $|Q|$; the construction
uses the same $Q$, so the class is preserved.

\end{proof}

\end{lemma}

\subsection{Proof of Lemma \ref{thm:prenorm-main}}
\label{app:thm2}
\begin{proof}
    Because parity objective subsumes all $\omega$-regular objectives, from lemma \ref{thm:prenorm}, $\Val_\sg(W) = \Val_{\sg'}(W')$ for any $\omega$-regular objective $W$ in $\sg$ and its corresponding objective $W'$ in $\sg'$.\\
    The ``consequently'' statement follows from lemmas \ref{lem:gg_sw}, \ref{lem:gg_as}, \ref{lem:gg_ls}, and \ref{lem: gg_qa}. 
\end{proof}

\clearpage 

\section{Constructing an RPOMDP from a Pre-{$\ptwo$} Transformed POSG}
\label{app:sg2rmdp}

Let $\sg'$ be the Pre-$\ptwo$ transformation of an alternating control POSG $\sg$ as
constructed in definition \ref{def:maxtd-transform}.  In $\sg'$, every play has the periodic structure
\[
s_1 \;\xrightarrow{a_1}\; d_{s_2} \;\xrightarrow{\anew}\; s_2 \;\xrightarrow{a_2}\; s_3
\;\xrightarrow{a_3}\; d_{s_4} \;\xrightarrow{\anew}\; s_4 \;\cdots
\]
where each round visits an original max-state, its sampled dummy successor, and the corresponding min-state, in that order.
The key observation is that the $\ptwo$-player's choice of action $a \in \atwo(s)$ at a
$\stwo$-state $s$ determines the next transition distribution $\tran'(s,a) \in \dist(\sone)$.
This is exactly the role of the \emph{environment} in an RPOMDP: given a state and an action,
it selects a distribution from the uncertainty set.  We exploit this to define a polytopic
RPOMDP $\rsg$ whose states are the $\pone$-states and dummy states of $\sg'$, and whose
uncertainty set at each dummy state under action $\anew$ is the \emph{polytope} whose
vertices are the transition distributions corresponding to the $\ptwo$-player's available
actions.

\begin{definition}[RPOMDP reduction of a Pre-$\ptwo$ transformed POSG]
\label{def:sg2rmdp}
Let $\sg' = (\sone \cup D, \stwo, \aone \cup \{\anew\}, \atwo, \tran', s_0,
\obsone \cup \Obsnewone, \obstwo \cup \Obsnewtwo, \obfnone', \obfntwo')$ be a Pre-$\ptwo$
transformed POSG.  We define the associated RPOMDP as
\[
\rsg = (S_\rp,\; A_\rp,\; \uncert_\rp,\; \obsa^\rp,\; \obse^\rp,\;
         \obfna^\rp,\; \obfne^\rp,\; s_0^\rp),
\]
where:
\begin{itemize}
    \item \emph{(State space)} $S_\rp = \sone \cup D$.
    \item \emph{(Actions)} $A_\rp=\aone\cup\{\anew\}$, with effective alphabets $A_\rp(s)=\aone(s)$ for $s\in\sone$ and $A_\rp(d_s)=\{\anew\}$. Every agent strategy is one distribution over $A_\rp$, supported on the effective alphabet of its current observation class.
        \item \emph{(Uncertainty set)} $\uncert_\rp = \prod_{(s,a)\in S_\rp\times A_\rp}\uncertsa(s,a)$, where the non-default polytopes are:
          \begin{itemize}
              \item For $s\in\sone$, $a\in\aone(s)$, or $s=d_{s'}\in D$, $a=\anew$, define
                    \[
                    V_{s, a} =
                    \begin{cases}
                        \{\tran'(s, a)\} & \text{if } s \in \sone,\ a\in\aone(s),\\
                        \bigl\{\,\tran'(s', a') \;\big|\; a' \in \atwo(s')\,\bigr\} & \text{if } s = d_{s'},\ a=\anew,\ s'\in \stwo.
                    \end{cases}
                    \]
                    For these pairs,
                    \[
                    \uncertsa(s,a)=\conv(V_{s,a}),
                    \]
                    where $\conv(V_{s,a})$ denotes the convex hull (polytope) of $V_{s,a}$.
                    Remaining, ineffective pairs may use fixed singleton distributions; their components are never evaluated by an agent strategy. The list $V_{s,a}$ may contain redundant convex-hull generators. No claim that distinct min-actions induce distinct extreme points is needed.
          \end{itemize}
    \item \emph{(Agent observations)} $\obsa^\rp = \obsone \cup \Obsnewone$.
    \item \emph{(Environment observations)} $\obse^\rp = \obstwo \cup \Obsnewtwo$.
    \item \emph{(Agent observation function)} $\obfna^\rp: S_\rp \to \obsa^\rp$ with
          $\obfna^\rp(s') = \obfnone'(s')$ for $s' \in S_\rp$.
    \item \emph{(Environment observation function)} $\obfne^\rp: S_\rp \to \obse^\rp$ with
          $\obfne^\rp(s') = \obfntwo'(s')$ for $s' \in S_\rp$.
    \item \emph{(Initial state)} $s_0^\rp = s_0$.
\end{itemize}
\end{definition}

\begin{remark}
The generating set $V_{d_s,\anew} = \{\tran'(s,a) \mid a \in \atwo(s)\}$ of the polytope
$\uncertsa(d_s, \anew)$ encodes all effective $\ptwo$-actions at $s \in \stwo$:
each listed distribution is induced by an effective min-action. The environment selects a convex combination of these distributions, exactly matching a randomized min-action at $s$.
At non-dummy states $s' \in \sone$, the uncertainty set is a singleton, so the environment has no real choice
there.  This makes $\rsg$ a polytopic $(s,a)$-rectangular RPOMDP. 
\end{remark}
\subsection{Mapping plays, observation histories, strategies from $\sg'$ to $\rsg$}
\label{app:map3}
\smallskip\noindent\textbf{Mapping on plays.}
Because the uncertainty set at dummy states is polytopic (not a singleton), different plays
in $\sg'$ that agree on all $\sone \cup D$ states and $\aone \cup \{\anew\}$ actions but
differ in the $\ptwo$-action chosen at $\stwo$-states will map to the \emph{same} play in
$\rsg$.  We therefore define $\Lambda: \Plays_{\sg'} \to \Plays_{\rsg}$ and its
set-valued inverse $\Lambda^{-1}: \Plays_{\rsg} \to 2^{\Plays_{\sg'}}$.

Let $\rho' \in \Plays_{\sg'}$ be of the form
$\rho' = s_1, a_1, s_2, a_2, s_3, a_3, \ldots$
where by construction of $\sg'$, for all $k \ge 0$:
\[
s_{3k+1} \in \sone,\quad s_{3k+2} \in D,\quad s_{3k+3} \in \stwo,
\quad a_{3k+1} \in \aone,\quad a_{3k+2} = \anew,\quad a_{3k+3} \in \atwo.
\]
We define $\Lambda(\rho') = \rho^\rp$ where $\rho^\rp = t_1, b_1, t_2, b_2, \ldots \in \Plays_{\rsg}$ is given index-wise, for all $k \ge 0$, by:
\begin{align*}
t_{2k+1} &= s_{3k+1} \in \sone, \\
b_{2k+1} &= a_{3k+1} \in \aone, \\
t_{2k+2} &= s_{3k+2} \in D, \\
b_{2k+2} &= \anew.
\end{align*}
That is, we remove all occurrences of $\stwo$-states and $\atwo$-actions from $\rho'$.

\noindent\textbf{Set-valued inverse.}
Define $\Lambda^{-1}: \Plays_{\rsg} \to 2^{\Plays_{\sg'}}$ as follows.
Let $\rho^\rp = t_1, b_1, t_2, b_2, \ldots \in \Plays_{\rsg}$ where for all $k \ge 0$:
$t_{2k+1} \in \sone$, $t_{2k+2} \in D$, $b_{2k+1} \in \aone$, $b_{2k+2} = \anew$.
Then $\Lambda^{-1}(\rho^\rp)$ is the set of all plays
$\rho' = s_1, a_1, s_2, a_2, s_3, a_3, \ldots \in \Plays_{\sg'}$
such that for all $k \ge 0$:
\begin{align*}
s_{3k+1} &= t_{2k+1}, \\
a_{3k+1} &= b_{2k+1}, \\
s_{3k+2} &= t_{2k+2}, \\
a_{3k+2} &= b_{2k+2} = \anew, \\
s_{3k+3} &= s \text{ such that } t_{2k+2} = d_s \in D, \\
a_{3k+3} &\in \atwo \text{ (arbitrary action subject to } \tran'(s_{3k+3}, a_{3k+3})(s_{3k+4}) > 0\text{)}.
\end{align*}
Thus $\Lambda^{-1}(\rho^\rp)$ consists of all plays in $\sg'$ obtained by resolving each
transition at a dummy state using an arbitrary $\ptwo$-action.  In particular,
$\Lambda(\rho') = \rho^\rp$ for every $\rho' \in \Lambda^{-1}(\rho^\rp)$, i.e.\
$\Lambda$ is many-to-one.

\smallskip\noindent\textbf{Mapping on observation histories.}
We extend $\Lambda$ to observation histories exactly as before. Let
$oh = o_1, o_2, o_3, \ldots, o_n \in \OH_\pone^{\sg'}$.

We define $\LambdaOpone: \OH_\pone^{\sg'} \to \OH_\pa^{\rsg}$ by
$\LambdaOpone(oh) = oh^\rp$ where $oh^\rp = \tilde{o}_1, \tilde{o}_2, \tilde{o}_3, \ldots$
is given index-wise, for all $k \ge 0$ such that indices exist, by:
\begin{align*}
\tilde{o}_{2k+1} &= o_{3k+1}, \\
\tilde{o}_{2k+2} &= o_{3k+2}.
\end{align*}
Observations $o_{3k+3}$ of $\stwo$-states are removed.
The map $\LambdaOpone$ is a bijection with inverse
$(\LambdaOpone)^{-1}: \OH_\pa^{\rsg} \to \OH_\pone^{\sg'}$ defined, for
$oh^\rp = \tilde{o}_1, \tilde{o}_2, \ldots, \tilde{o}_m \in \OH_\pa^{\rsg}$
($\tilde{o}_{2k+1} \in \obsone,\tilde{o}_{2k+2} \in \Obsnewone$ for $k\ge 0$), by
$(\LambdaOpone)^{-1}(oh^\rp) = oh = o_1, o_2, \ldots, o_n$. Insert the reconstructed min-observation only after a dummy observation that is followed by another retained observation, not after a terminal dummy observation. For all indices within this length:
\begin{align*}
o_{3k+1} &= \tilde{o}_{2k+1}, \\
o_{3k+2} &= \tilde{o}_{2k+2}, \\
o_{3k+3} &= x \text{ such that } \tilde{o}_{2k+2} = \obsnewone_x \quad(2k+2<m), \\
o_n &= \tilde{o}_m.
\end{align*}

where $n = 3k+1$ when $m=2k+1$ and $n=3k+2$ when $m = 2k+2$.

We extend this to the $\ptwo$-player. We define
$\LambdaOptwo: \OH_\ptwo^{\sg'} \to \widehat{\OH}_\pe^{\rsg}$, where
$\widehat{\OH}_\pe^{\rsg} \subseteq \OH_\pe^{\rsg}$ is the set of environment observation
histories of $\rsg$ ending in an observation of a state in $D$, i.e, $\widehat{\OH}_\pe^{\rsg} = \{oh^\rp \in \OH_\pe^{\rsg} \mid \Last(oh^\rp) \in \Obsnewtwo\}$. For
$oh = o_1, o_2, \ldots, o_n \in \OH_\ptwo^{\sg'}$,
$\LambdaOptwo(oh) = \tilde{o}_1, \tilde{o}_2, \ldots$ is defined index-wise for all
$k \ge 0$ by:
\begin{align*}
\tilde{o}_{2k+1} &= o_{3k+1}, \\
\tilde{o}_{2k+2} &= o_{3k+2}.
\end{align*}
Since $\Last(oh) \in \obstwo$, the mapped history ends in $\Obsnewtwo$ and hence lies in
$\widehat{\OH}_\pe^{\rsg}$. Its inverse
$(\LambdaOptwo)^{-1}: \widehat{\OH}_\pe^{\rsg} \to \OH_\ptwo^{\sg'}$ is defined for
$oh^\rp = \tilde{o}_1, \tilde{o}_2, \ldots, \tilde{o}_m \in \widehat{\OH}_\pe^{\rsg}$
by $(\LambdaOptwo)^{-1}(oh^\rp) = o_1, o_2, \ldots, o_{m + m/2}$ where
for all $k \ge 0$:
\begin{align*}
o_{3k+1} &= \tilde{o}_{2k+1}, \\
o_{3k+2} &= \tilde{o}_{2k+2}, \\
o_{3k+3} &= x \text{ such that } \tilde{o}_{2k+2} = \obsnewtwo_x,\\
o_{m + m/2} &= x \text{ such that } \tilde{o}_m = \obsnewtwo_x.
\end{align*}

\begin{lemma}
\label{lemma:lam-laminv-max}
For all $oh^\rp \in \OH_\pa^{\rsg}$,
$\LambdaOpone\!\left((\LambdaOpone)^{-1}(oh^\rp)\right) = oh^\rp$; and for all
$oh' \in \OH_\pone^{\sg'}$,
$(\LambdaOpone)^{-1}\!\left(\LambdaOpone(oh')\right) = oh'$.
\end{lemma}
\begin{proof}
Let
\[
oh^\rp=\tilde{o}_1,\tilde{o}_2,\ldots,\tilde{o}_m
\in \OH_\pa^{\rsg}.
\]
We consider the two possible forms of $m$.

If $m=2k+1$, then by definition of $(\LambdaOpone)^{-1}$,
\[
o_{3i+1}=\tilde{o}_{2i+1},
\qquad
o_{3i+2}=\tilde{o}_{2i+2}
\qquad
o_{3i+3}=x \text{ where }
\tilde{o}_{2i+2}=\obsnewone_x
\]
for all $0\leq i<k$, and $o_{3k+1}=\tilde{o}_{2k+1}$. Hence, applying
$\LambdaOpone$  removes precisely the inserted observations
$o_{3i+3}$ and retains
\[
\LambdaOpone\!\left((\LambdaOpone)^{-1}(oh^\rp)\right) = \LambdaOpone\!\left(o_1,o_2,\cdots,o_{3k+1}\right)
=
\tilde{o}_1,\tilde{o}_2,\ldots,\tilde{o}_{2k+1}
=
oh^\rp.
\]

If $m=2k+2$, then by definition of $(\LambdaOpone)^{-1}$,
\[
o_{3i+1}=\tilde{o}_{2i+1},
\qquad
o_{3i+2}=\tilde{o}_{2i+2},
\qquad
o_{3i+3}=x \text{ where }
\tilde{o}_{2i+2}=\obsnewone_x
\]
for $0\leq i<k$. Also,
\[
o_{3k+1}=\tilde{o}_{2k+1},
\qquad
o_{3k+2}=\tilde{o}_{2k+2}.
\]
Therefore, applying $\LambdaOpone$ removes precisely the inserted observations
$o_{3i+3}$ and retains
\[
o_{3i+1}=\tilde{o}_{2i+1},
\qquad
o_{3i+2}=\tilde{o}_{2i+2}.
\]
Thus
\[
\LambdaOpone\!\left((\LambdaOpone)^{-1}(oh^\rp)\right) = \LambdaOpone\!\left(o_1,o_2,\cdots,o_{3k+1},o_{3k+2}\right)
=
\tilde{o}_1,\tilde{o}_2,\ldots,\tilde{o}_{2k+2}
=
oh^\rp.
\]

Conversely, let
\[
oh'=o_1,o_2,\ldots,o_n\in\OH_\pone^{\sg'}.
\]
By the definition of $\LambdaOpone$, its image is obtained by retaining
\[
o_{3k+1},o_{3k+2}
\]
and deleting every $o_{3k+3}$. Applying $(\LambdaOpone)^{-1}$ then reconstructs
each deleted observation $o_{3k+3}$ from $\obsnewone_{o_{3k+3}}$
whenever a subsequent retained observation exists. If the history ends at
$3k+1$, no observation is inserted after the terminal observation; if it ends
at $3k+2$, the reconstructed history ends at $3k+2$. Hence every observation
of $oh'$ is restored at its original position, and
\[
(\LambdaOpone)^{-1}\!\left(\LambdaOpone(oh')\right)
=
o_1,o_2,\ldots,o_n
=
oh'.
\]
Therefore both compositions are the identity.
\end{proof}

\begin{lemma}
\label{lemma:lam-laminv-min}
For all $oh^\rp \in \widehat{\OH}_\pe^{\rsg}$,
$\LambdaOptwo\!\left((\LambdaOptwo)^{-1}(oh^\rp)\right) = oh^\rp$; and for all
$oh' \in \OH_\ptwo^{\sg'}$,
$(\LambdaOptwo)^{-1}\!\left(\LambdaOptwo(oh')\right) = oh'$.
\end{lemma}
\begin{proof}
Let
\[
oh^\rp=\tilde{o}_1,\tilde{o}_2,\ldots,\tilde{o}_m
\in\widehat{\OH}_\pe^{\rsg}.
\]
Since $oh^\rp$ ends in $\Obsnewtwo$, $m$ is even. Write
$m=2k+2$ for some $k\geq0$. By definition of $(\LambdaOptwo)^{-1}$,
\[
o_{3i+1}=\tilde{o}_{2i+1},
\qquad
o_{3i+2}=\tilde{o}_{2i+2},
\qquad
o_{3i+3}=x \text{ where }
\tilde{o}_{2i+2}=\obsnewtwo_x,
\]
for all $0\leq i\leq k$. Hence, applying $\LambdaOptwo$ removes exactly the
inserted observations $o_{3i+3}$ and retains
\[
o_{3i+1}=\tilde{o}_{2i+1},
\qquad
o_{3i+2}=\tilde{o}_{2i+2}.
\]
Consequently,
\[
\LambdaOptwo\!\left((\LambdaOptwo)^{-1}(oh^\rp)\right) = \LambdaOptwo\!\left(o_1,o_2,\cdots,o_{3k+3}\right)
=
\tilde{o}_1,\tilde{o}_2,\ldots,\tilde{o}_{2k+2}
=
oh^\rp.
\]

Conversely, let
\[
oh'=o_1,o_2,\ldots,o_n\in\OH_\ptwo^{\sg'}.
\]
Since $oh'$ ends in an observation in $\obstwo$, its length is of the form
$n=3k+3$ for some $k\geq0$. By definition of $\LambdaOptwo$,
\[
\tilde{o}_{2i+1}=o_{3i+1},
\qquad
\tilde{o}_{2i+2}=o_{3i+2},
\]
for all $0\leq i\leq k$. Moreover,
\[
\tilde{o}_{2i+2}=\obsnewtwo_{o_{3i+3}},
\]
so $(\LambdaOptwo)^{-1}$ reconstructs $o_{3i+3}$
as the unique $x$ satisfying $\tilde{o}_{2i+2}=\obsnewtwo_x.$
Thus every deleted observation $o_{3i+3}$ is reconstructed at its original
position, while the retained observations are restored by
\[
o_{3i+1}=\tilde{o}_{2i+1},
\qquad
o_{3i+2}=\tilde{o}_{2i+2}.
\]
In particular, the final observation is reconstructed as
\[
o_{3k+3}=x
\quad\text{such that}\quad
\tilde{o}_{2k+2}=\obsnewtwo_x.
\]
Therefore,
\[
(\LambdaOptwo)^{-1}\!\left(\LambdaOptwo(oh')\right) = (\LambdaOptwo)^{-1}\!\left(\tilde{o}_1,\tilde{o}_2,\cdots,\tilde{o}_{2k+2}\right) 
=
o_1,o_2,\ldots,o_n
=
oh'.
\]
Hence both compositions are the identity.
\end{proof}

\smallskip\noindent\textbf{Mapping of strategies.}
Recall from section \ref{sec:prelim} and the definition of RPOMDP that:
\begin{itemize}
    \item A max-strategy in $\sg'$ has type $\sigma':\OH_\pone^{\sg'}\to\dist(\aone')$; a min-strategy retains its state-indexed prescription type.
    \item An RPOMDP agent strategy has type $\sigma^\rp:\OH_\pa^{\rsg}\to\dist(A_\rp)$.
    \item An environment strategy $\pi^\rp\in\Pi_{\rsg}$ has type
        $\pi^\rp:\OH_\pe^{\rsg}\times A_\rp\to\uncert_\rp$ and hence selects a complete transition assignment.
\end{itemize}

Since $\uncert_\rp$ is $(s,a)$-rectangular and polytopic, the environment strategy
$\pi^\rp \in \Pi_{\rsg}$ at a dummy state $d_s \in D$ under action $\anew$ picks a
distribution $\pi^\rp(oh^\rp,\anew)(d_s, \anew) \in \uncertsa(d_s, \anew)$, which by
\Cref{def:sg2rmdp} is a convex combination of the vertices
$V_{d_s,\anew} = \{\tran'(s,a) \mid a \in \atwo(s)\}$. We write this as
\[
\pi^\rp(oh^\rp,\anew)(d_s, \anew) = \sum_{a \in \atwo(s)} \alpha_a \cdot \tran'(s, a),
\quad \alpha_a \geq 0,\; \sum_{a \in \atwo(s)} \alpha_a = 1.
\]
For every $s\in\stwo$, define
\[
\mathsf{bar}^{\sg'}_s:\dist(\atwo(s))\to\uncertsa(d_s,\anew) 
\text{ and for $\lambda \in \dist(\atwo(s))$, }
\mathsf{bar}^{\sg'}_s(\lambda)
=\sum_{a\in\atwo(s)}\lambda(a)\tran'(s,a).
\]
For every $P\in\uncertsa(d_s,\anew)$, fix a coefficient vector $\beta_s^P\in\dist(\atwo(s))$ with $\mathsf{bar}^{\sg'}_s(\beta_s^P)=P$. Such vectors exist by the convex-hull definition and need not be unique; the fixed choice makes the reverse map well-defined.
Fix a reference element $P^0_{x,b}\in\uncertsa(x,b)$ for every $(x,b)\in S_\rp\times A_\rp$.
We define the following four maps lifting strategies between $\sg'$ and $\rsg$:
\begin{enumerate}
    \item $\Omega_\pone : \stratpone_{\sg'} \to \Sigma_{\rsg}$:
            for each $\sigma' \in \stratpone_{\sg'}$ and $oh^\rp \in \OH_\pa^{\rsg}$,
          \[
            \Omega_\pone(\sigma')(oh^\rp) = \sigma'\!\left((\LambdaOpone)^{-1}(oh^\rp)\right).
          \]

    At a dummy observation both sides equal $\dirac{\anew}$, since every max-strategy in $\sg'$ uses the sole effective dummy action.

    \item $\Omega_\ptwo : \stratptwo_{\sg'} \to \Pi_{\rsg}$:
          for each $\pi'\in\stratptwo_{\sg'}$, $oh^\rp\in\OH_\pe^{\rsg}$, and $b\in A_\rp$, define the complete assignment $\Omega_\ptwo(\pi')(oh^\rp,b)\in\uncert_\rp$ by
          \[
          \Omega_\ptwo(\pi')(oh^\rp,b)(x,c)=
          \begin{cases}
          \mathsf{bar}^{\sg'}_s\!\left(
              \pi'\!\left((\LambdaOptwo)^{-1}(oh^\rp)\right)(s)
          \right)
              & \text{if }x=d_s,\ b=c=\anew \text{ and }
                \obfne^\rp(x)=\Last(oh^\rp)\\
          \tran'(x,b)
              &\text{if }{x}{\in}{\sone},\ b{=}c\in\aone(x) \text{ and }
                \obfne^\rp(x)=\Last(oh^\rp)\\
          P^0_{x,c} & \text{otherwise}
          \end{cases}
          \]

    \item $\Xi_\pone : \Sigma_{\rsg} \to \stratpone_{\sg'}$:
            for each $\sigma^\rp \in \Sigma_{\rsg}$ and $oh' \in \OH_\pone^{\sg'}$,
          \[
            \Xi_\pone(\sigma^\rp)(oh') = \sigma^\rp\!\left(\LambdaOpone(oh')\right).
          \]

    \item $\Xi_\ptwo : \Pi_{\rsg} \to \stratptwo_{\sg'}$:
          for each $\pi^\rp \in \Pi_{\rsg}$, $oh' \in \OH_\ptwo^{\sg'}$ and $\forall s' \in \stwo$ with
                    $\obfntwo'(s') = \Last(oh')$, define
          \[
                    \Xi_\ptwo(\pi^\rp)(oh')(s')
                    =\beta_{s'}^{\pi^\rp(\LambdaOptwo(oh'),\anew)(d_{s'},\anew)}.
          \]
\end{enumerate}

\begin{remark}
\label{rem:omegaptwo-sg2rmdp}
By definition of $\Xi_\ptwo$, for all $\pi^\rp \in \Pi_{\rsg}$,
$oh' \in \OH_\ptwo^{\sg'}$, and $d_{s'} \in D$,
\[
\pi^\rp(\LambdaOptwo(oh'),\anew)(d_{s'},\anew)
=\mathsf{bar}^{\sg'}_{s'}\!\left(
    \Xi_\ptwo(\pi^\rp)(oh')(s')
\right).
\]
\end{remark}

\begin{lemma}
\label{lemma:omega-xi-inverse}
The agent maps $\Omega_\pone$ and $\Xi_\pone$ are mutual inverses.  For every compatible dummy history and state, the environment maps satisfy
\begin{align*}
\Omega_\ptwo(\Xi_\ptwo(\pi^\rp))(oh^\rp,\anew)(d_s,\anew)
  &=\pi^\rp(oh^\rp,\anew)(d_s,\anew),\\
\mathsf{bar}^{\sg'}_s\!\left(
  \Xi_\ptwo(\Omega_\ptwo(\pi'))(oh')(s)
\right)
  &=\mathsf{bar}^{\sg'}_s\!\left(\pi'(oh')(s)\right).
\end{align*}
Thus the environment compositions preserve transition distributions.

\begin{proof}
The agent identities follow from Lemma~\ref{lemma:lam-laminv-max} by composition, as equalities of single action distributions. The history maps retain the current observation and effective agent/max alphabet, so both maps preserve admissibility, including the forced dummy action. The environment maps preserve local support because $\beta_s^P\in\dist(\atwo(s))$ and every barycenter belongs to $\uncertsa(d_s,\anew)$. The environment identities follow from
$\mathsf{bar}^{\sg'}_s(\beta_s^P)=P$ and Lemma~\ref{lemma:lam-laminv-min}. At non-dummy states the uncertainty set is a singleton.
\end{proof}
\end{lemma}
\subsection{Lifting objectives from $\sg'$ to $\rsg$}
\label{app:lift3}
Define $\prfn^\rp: S_\rp \to \{0,1,\dots,d\}, d \in \nat$ by
          $\prfn^\rp(s') = \prfn'(s')$ for $s' \in S_\rp$ where $\prfn'$ is the priority function for $\sg'$.

\begin{lemma}
\label{lemma:lambda-parity}
Let $\sg'$, $\rsg$, $\prfn'$, and $\prfn^\rp$ be as above, and let $\rho' \in \Plays_{\sg'}$.
Then
\[
\rho' \in \parity{\sg'}{\prfn'} \iff \Lambda(\rho') \in \parity{\rsg}{\prfn^\rp}.
\]
\begin{proof}
By construction of $\prfn^\rp$, every state $s' \in \sone$ in $\Lambda(\rho')$ carries
priority $\prfn^\rp(s') = \prfn'(s')$, and every dummy state $d_s \in D$ carries priority
$\prfn^\rp(d_s) = \prfn'(d_s)$.  The $\stwo$-states of $\rho'$ are removed by $\Lambda$;
since by definition of $\prfn'$ 
the priority of each
$s \in \stwo$ equals that of the preceding dummy state $d_s$, no $\liminf$-relevant priority
value is lost.  Hence the parity condition is preserved.
\end{proof}
\end{lemma}

\subsection{Value Preservation between $\sg'$ and $\rsg$}
\label{app:value3}

\begin{lemma}
        The strategy maps preserve objective probabilities in both directions:
        \begin{align*}
        \probm_{\sg'}^{\sigma',\pi'}(W')
            &=\probm_{\rsg}^{\Omega_\pone(\sigma'),\Omega_\ptwo(\pi')}(W^\rp),\\
        \probm_{\rsg}^{\sigma^\rp,\pi^\rp}(W^\rp)
            &=\probm_{\sg'}^{\Xi_\pone(\sigma^\rp),\Xi_\ptwo(\pi^\rp)}(W'),
        \end{align*}
        for all strategies of the indicated types,
    where $W' \subseteq \Plays_{\sg'}$ is a parity objective in $\sg'$ and $W^\rp = \Lambda(W')$ is the corresponding objective in $\rsg$.
    \label{lemma:omega-equiv}
\end{lemma}

\begin{proof}
Since the probability measure over plays is uniquely determined by its values on cylinder
sets, it suffices to show that for every finite history $h'$ of a every play in $W'$
ending at a state in $\sone \cup D$, there exists a corresponding history $h^\rp$ in for some play in $W^\rp$ such that
\[
\probm_{\sg'}^{\sigma', \pi'} (\cyl{h'})
= \probm_{\rsg}^{\Omega_\pone(\sigma'), \Omega_\ptwo(\pi')} (\cyl{h^\rp}).
\]
We map each such $h'$ to $h^\rp$ similar to $\Lambda$ maps the corresponding plays
(dropping $\stwo$-states).  The proof proceeds by induction on the
length of $h'$.

\smallskip\noindent\textbf{Base case.}  $h' = s_0$.  Then $h^\rp = s_0^\rp = s_0$ and
both cylinders have probability $1$.

\smallskip\noindent\textbf{Inductive step.} We distinguish two cases based on how $h'$
is extended.

\smallskip
\noindent\textbf{Case 1: $h'$ is of the form $h'_1, d_s$ where
$d_s \in D$.} Then $\Last(h'_1) \in \sone$ and $h^\rp = h^\rp_1, d_s$ where $h^\rp_1$
corresponds to $h'_1$ by induction.  Let $oh^\rp_a = \obfna^\rp(h^\rp_1)$.  Since
$\Last(h^\rp_1) = s_1 \in \sone$, the uncertainty set for $a \in \aone(s_1)$, $\uncertsa(s_1, a)$ is a
singleton $\{\tran'(s_1, a)\}$.  Then
\begin{align*}
\probm_{\rsg}^{\Omega_\pone(\sigma'), \Omega_\ptwo(\pi')} (\cyl{h^\rp})
&= \probm_{\rsg}^{\Omega_\pone(\sigma'), \Omega_\ptwo(\pi')} (\cyl{h^\rp_1})
    \times \sum_{a \in \aone(s_1)} \left(\Omega_\pone(\sigma')(oh^\rp_a)(a)
   \times \tran'(s_1, a)(d_s) \right).
\end{align*}
By the induction hypothesis, 
$\probm_{\rsg}^{\Omega_\pone(\sigma'), \Omega_\ptwo(\pi')} (\cyl{h^\rp_1}) = \probm_{\sg'}^{\sigma', \pi'} (\cyl{h'_1})$.  Since $\Last(oh^\rp_a) \in \obsone$ and
$(\LambdaOpone)^{-1}(oh^\rp_a) = \obfnone'(h'_1)$, we have
$\Omega_\pone(\sigma')(oh^\rp_a)(a) = \sigma'(\obfnone'(h'_1))(a)$. The singleton
uncertainty set means the environment contributes the unique element
$\tran'(s_1, a)$ evaluated for transition to $d_s$.  Hence
\begin{align*}
    \probm_{\rsg}^{\Omega_\pone(\sigma'), \Omega_\ptwo(\pi')} (\cyl{h^\rp}) &= \probm_{\sg'}^{\sigma', \pi'} (\cyl{h'_1}) \times \sum_{a \in \aone(s_1)} \left(\sigma'(\obfnone'(h'_1))(a)\times\tran'(s_1,a)(d_s)\right)\\
& = \probm_{\sg'}^{\sigma', \pi'} (\cyl{h'})
\end{align*}

\smallskip

\smallskip
\noindent\textbf{Case 2: $h'$ is of the form $h'_1, s, s'$ where
$\Last(h'_1) = d_s \in D$, $s \in \stwo$, and $s' \in \sone$.}
Then $h^\rp = h^\rp_1, s'$ where $h^\rp_1$ corresponds to $h'_1$.

The sole effective action from $d_s$ is $\anew$ and the transition $\tran'(d_s, \anew) = \dirac{s}$ is
deterministic.  The corresponding prefix in $\rsg$ for $h'_1,s$ is $h^\rp_1$ (since
$\stwo$-state $s$ is removed by $\Lambda$). Since the $\anew$ action at $d_s$ is
deterministic (probability $1$) in both $\sg'$ and $\rsg$, and the $\stwo$-state $s$ is
not present in $\rsg$, we have immediately
\[
\probm_{\rsg}^{\Omega_\pone(\sigma'), \Omega_\ptwo(\pi')} (\cyl{h^\rp_1}) = \probm_{\sg'}^{\sigma', \pi'} (\cyl{h'_1,s}) = \probm_{\sg'}^{\sigma', \pi'} (\cyl{h'_1})
\]

Let $oh^\rp_a=\obfna^\rp(h^\rp_1)$ and $oh^\rp_e = \obfne^\rp(h^\rp_1)$.  Since $\Last(h^\rp_1) = d_s \in D$, we have
$\Last(oh^\rp_e) \in \Obsnewtwo$.  The environment in $\rsg$ selects a distribution from
$\uncertsa(d_s, \anew)$; the mapped agent plays the sole effective action $\anew$ with probability $1$.
\begin{align*}
\probm_{\rsg}^{\Omega_\pone(\sigma'), \Omega_\ptwo(\pi')} (\cyl{h^\rp})
&= \probm_{\rsg}^{\Omega_\pone(\sigma'), \Omega_\ptwo(\pi')} (\cyl{h^\rp_1})
    \times \Omega_\pone(\sigma')(oh^\rp_a)(\anew) \\
&\quad \times
   \Omega_\ptwo(\pi')(oh^\rp_e,\anew)(d_s, \anew)(s')
\end{align*}
By definition of $\Omega_\ptwo$,
\[
\Omega_\ptwo(\pi')(oh^\rp_e,\anew)(d_s, \anew)
= \sum_{a'' \in \atwo(s)}
  \pi'\!\left((\LambdaOptwo)^{-1}(oh^\rp_e)\right)(s)(a'') \cdot \tran'(s, a'').
\]
Evaluating at $s'$:
\[
\Omega_\ptwo(\pi')(oh^\rp_e,\anew)(d_s, \anew)(s')
= \sum_{a'' \in \atwo(s)}
  \pi'\!\left((\LambdaOptwo)^{-1}(oh^\rp_e)\right)(s)(a'') \cdot \tran'(s, a'')(s').
\]
Since $(\LambdaOptwo)^{-1}(oh^\rp_e) = \obfntwo'(h'_1,s)$, this equals
$\sum_{a'' \in \atwo(s)} \pi'(\obfntwo'(h'_1,s))(s)(a'') \cdot \tran'(s,a'')(s')$.
Also, $\Omega_\pone(\sigma')(oh^\rp_a)(\anew) = 1$ since $\anew$ is the sole effective action
at dummy states. By the induction hypothesis applied to $h'_1$,
\begin{align*}
\probm_{\rsg}^{\Omega_\pone(\sigma'), \Omega_\ptwo(\pi')} (\cyl{h^\rp})
&= \probm_{\sg'}^{\sigma', \pi'} (\cyl{h'_1})
    \times \sum_{a'' \in \atwo(s)}
     \pi'(\obfntwo'(h'_1,s))(s)(a'') \cdot \tran'(s, a'')(s').
\end{align*}
On the other hand, in $\sg'$, from history $h'_1$ ending at $d_s$, the $\pone$-player plays
$\anew$ (probability $1$), reaching $s$ deterministically, after which the $\ptwo$-player
at $s$ chooses action $a'' \in \atwo(s)$ with probability $\pi'(\obfntwo'(h'_1,s))(s)(a'')$,
reaching $s'$ with probability $\tran'(s, a'')(s')$.  Therefore, summing over all $\ptwo$
actions in $\sg'$:
\begin{align*}
\probm_{\sg'}^{\sigma', \pi'} (\cyl{h'})
&= \probm_{\sg'}^{\sigma', \pi'} (\cyl{h'_1})
   \times 1 \times
    \sum_{a'' \in \atwo(s)} \pi'(\obfntwo'(h'_1,s))(s)(a'') \cdot \tran'(s, a'')(s').
\end{align*}
These two expressions agree, so
$\probm_{\rsg}^{\Omega_\pone(\sigma'), \Omega_\ptwo(\pi')} (\cyl{h^\rp}) = \probm_{\sg'}^{\sigma', \pi'} (\cyl{h'}) $.

For the converse mapping, repeat the cylinder induction with $\Xi_\pone$ and $\Xi_\ptwo$.  At a dummy state, the first environment identity in Lemma~\ref{lemma:omega-xi-inverse} ensures that the selected action mixture has barycenter exactly $\pi^\rp(oh^\rp,\anew)(d_s,\anew)$.  Thus every cylinder probability, and consequently the probability of the parity objective, is preserved in the reverse direction as well.
\end{proof}

We now state the main equivalence theorem for the RPOMDP reduction.

\begin{lemma}
\label{thm:sg2rmdp} 
Let $\sg'$ be a Pre-$\ptwo$ transformed POSG and $\rsg$ its RPOMDP reduction.  Let
$W' \subseteq \Plays_{\sg'}$ be a parity objective and $W^\rp = \Lambda(W')$ its
corresponding objective in $\rsg$.  Then
\[
\val_{\sg'}(W') = \val_{\rsg}(W^\rp).
\]
\begin{proof}
For fixed $\sigma'\in\Sigma_{\sg'}$, write
\[
q_{\sg'}(\sigma')=\inf_{\pi'\in\Pi_{\sg'}}
    \probm_{\sg'}^{\sigma',\pi'}(W'),
\qquad
q_{\rsg}(\sigma^\rp)=\inf_{\pi^\rp\in\Pi_{\rsg}}
    \probm_{\rsg}^{\sigma^\rp,\pi^\rp}(W^\rp).
\]
The first equality in Lemma~\ref{lemma:omega-equiv} gives
$q_{\rsg}(\Omega_\pone(\sigma'))\le q_{\sg'}(\sigma')$.  
\[
q_{\sg'}(\sigma') = \inf_{\pi'\in\Pi_{\sg'}}\probm_{\sg'}^{\sigma',\pi'}(W') =  \inf_{\Omega_\ptwo(\pi')\in\Pi_{\rsg}}\probm_{\rsg}^{\Omega_\pone(\sigma'),\Omega_\ptwo(\pi')}(W^\rp) \geq \inf_{\pi^{\rp}\in\Pi_{\rsg}}\probm_{\rsg}^{\Omega_\pone(\sigma'),\pi^{\rp}}(W^\rp) = q_{\rsg}(\Omega_\pone(\sigma'))
\]

Conversely, for every $\pi^\rp\in\Pi_{\rsg}$, the second equality there and
$\Xi_\pone(\Omega_\pone(\sigma'))=\sigma'$ give
\[
\probm_{\rsg}^{\Omega_\pone(\sigma'),\pi^\rp}(W^\rp)
    =\probm_{\sg'}^{\sigma',\Xi_\ptwo(\pi^\rp)}(W'),
\]
so
\[
q_{\rsg}(\Omega_\pone(\sigma')) = \inf_{\pi^{\rp}\in\Pi_{\rsg}}\probm_{\rsg}^{\Omega_\pone(\sigma'),\pi^{\rp}}(W^\rp) = \inf_{\Xi_\ptwo(\pi^\rp)\in\Pi_{\sg'}}\probm_{\sg'}^{\sigma',\Xi_\ptwo(\pi^\rp)}(W') \geq \inf_{\pi'\in\Pi_{\sg'}}\probm_{\sg'}^{\sigma',\pi'}(W') = q_{\sg'}(\sigma')
\]
Hence, $q_{\sg'}(\sigma') = q_{\rsg}(\Omega_\pone(\sigma'))$. Since $\Omega_\pone$ is a bijection by Lemma~\ref{lemma:omega-xi-inverse},
\begin{align*}
\val_{\sg'}(W')
    &=\sup_{\sigma'\in\Sigma_{\sg'}}q_{\sg'}(\sigma')
     =\sup_{\sigma'\in\Sigma_{\sg'}}q_{\rsg}(\Omega_\pone(\sigma'))\\
    &=\sup_{\sigma^\rp\in\Sigma_{\rsg}}q_{\rsg}(\sigma^\rp)
     =\val_{\rsg}(W^\rp).
\end{align*}

\end{proof}
\end{lemma}

\begin{corollary}
\label{cor:sg2rmdp}
Let $\sg'$ be a Pre-$\ptwo$ transformed POSG and $\rsg$ its RPOMDP reduction.  Let $\prfn'$ a priority function for $\sg'$.
Then
\[
\val_{\sg'}\!\left(\parity{\sg'}{\prfn'}\right) =
\val_{\rsg}\!\left(\parity{\rsg}{\prfn^\rp}\right).
\]
\begin{proof}
Follows directly from lemma \ref{thm:sg2rmdp} together with
lemma \ref{lemma:lambda-parity}.
\end{proof}
\end{corollary}

\begin{lemma}
\label{cor:grsurewinning}
Projected supported plays are preserved in both mapping directions:
\begin{align*}
\Lambda\!\left(\Plays^{\sigma',\pi'}_{\sg'}\right)
    &=\Plays^{\Omega_\pone(\sigma'),\Omega_\ptwo(\pi')}_{\rsg},\\
\Lambda\!\left(\Plays^{\Xi_\pone(\sigma^\rp),\Xi_\ptwo(\pi^\rp)}_{\sg'}\right)
    &=\Plays^{\sigma^\rp,\pi^\rp}_{\rsg}.
\end{align*}
\begin{proof}
We prove both equalities by induction on corresponding finite prefixes. The initial prefixes consist of $s_0$ on both sides. Suppose first that corresponding prefixes under $(\sigma',\pi')$ end at an original max-state $s\in\sone$. Their agent observation histories correspond under $\LambdaOpone$, and the definition of $\Omega_\pone$ gives
\[
\Omega_\pone(\sigma')(\LambdaOpone(oh'))=\sigma'(oh').
\]
Thus the same max-actions have positive probability. For every such action $a$, the RPOMDP uncertainty set at $(s,a)$ is the singleton $\{\tran'(s,a)\}$, so a dummy state $d_t$ is a supported successor in $\rsg$ exactly when it is a supported successor in $\sg'$.

Now suppose the corresponding prefixes end at $d_t$. In $\sg'$, the effective action $\anew$ leads deterministically to $t\in\stwo$, after which the min-player uses
$\lambda=\pi'(oh'_t)(t)\in\dist(\atwo(t))$. In $\rsg$, the mapped environment selects $\mathsf{bar}^{\sg'}_t(\lambda)$ at $(d_t,\anew)$. Hence, for every successor $s'\in\sone$,
\[
\mathsf{bar}^{\sg'}_t(\lambda)(s')>0
\quad\Longleftrightarrow\quad
\lambda(a)\tran'(t,a)(s')>0
\text{ for some }a\in\atwo(t),
\]
because all summands are nonnegative. Therefore an extension from $d_t$ to $s'$ is supported in $\rsg$ exactly when an extension from $d_t$ through $t$ and some supported min-action $a$ to $s'$ is supported in $\sg'$. This proves the induction step for the first equality.

For the reverse direction, agent-action supports agree by the definition of $\Xi_\pone$. At a dummy state use $\lambda=\Xi_\ptwo(\pi^\rp)(oh')(t)$. The first identity in Lemma~\ref{lemma:omega-xi-inverse} gives
$\mathsf{bar}^{\sg'}_t(\lambda)=\pi^\rp(\LambdaOptwo(oh'),\anew)(d_t,\anew)$, so the same nonnegativity argument proves the reverse induction step. Finally, an infinite play is supported exactly when all its finite prefixes are supported. Applying $\Lambda$ to the resulting prefix correspondences yields both displayed equalities. 
\end{proof}
\end{lemma}

\begin{lemma}
\label{lem:gr_sw}
Let $W' \subseteq \Plays_{\sg'}$ be a parity objective of $\sg'$ and $W^\rp = \Lambda(W')$ its lifted objective in $\rsg$. Let $\sigma' \in \Sigma_{\sg'}$ and $\sigma^\rp=\Omega_\pone(\sigma')$. Then
\[\text{$\sigma'$ is sure winning for $W'$}
\iff \text{$\sigma^\rp$ is sure winning for $W^\rp$}.\]
    \begin{proof}
        Suppose $\sigma'$ is sure winning and take arbitrary $\pi^\rp\in\Pi_\rsg$.  The second equality of Lemma~\ref{cor:grsurewinning} and  lemma~\ref{lemma:lambda-parity} gives
        \[
        \Plays^{\sigma^\rp,\pi^\rp}_{\rsg}
            =\Lambda\!\left(\Plays^{\sigma',\Xi_\ptwo(\pi^\rp)}_{\sg'}\right)
            \subseteq W^\rp.
        \]
        Conversely, suppose $\sigma^\rp$ is sure winning and take arbitrary $\pi'\in\Pi_{\sg'}$.  The first equality gives
        $\Lambda(\Plays^{\sigma',\pi'}_{\sg'})=\Plays^{\sigma^\rp,\Omega_\ptwo(\pi')}_{\rsg}\subseteq W^\rp$.  Lemma~\ref{lemma:lambda-parity} then implies $\Plays^{\sigma',\pi'}_{\sg'}\subseteq W'$.
    \end{proof}
\end{lemma}

\begin{lemma}[Almost-sure Winning]
\label{lem:gr_as}
        Let $W' \subseteq \Plays_{\sg'}$ be a parity objective of $\sg'$ and $W^\rp = \Lambda(W')$ its lifted objective in $\rsg$. Let $\sigma' \in \Sigma_{\sg'}$ and $\sigma^\rp = \Omega_\pone(\sigma')$. Then,
        \[ \text{$\sigma'$ is almost-sure winning for $W'$}
        \iff \text{$\sigma^\rp$ is almost-sure winning for $W^\rp$}.\]
    \begin{proof}
        The fixed-strategy equality proved in Lemma~\ref{thm:sg2rmdp} is
        \[
        \inf_{\pi'\in\Pi_{\sg'}}\probm_{\sg'}^{\sigma',\pi'}(W')
            =\inf_{\pi^\rp\in\Pi_{\rsg}}\probm_{\rsg}^{\sigma^\rp,\pi^\rp}(W^\rp).
        \]
        Hence one side equals $1$ if and only if the other does.
    \end{proof}
\end{lemma}

\begin{lemma} [Limit-Sure Winning]
\label{lem:gr_ls}
    Let $W' \subseteq \Plays_{\sg'}$ be a parity objective of $\sg'$ and $W^\rp = \Lambda(W')$ is the corresponding objective in $\rsg$. Then,
    \[W' \text{ is limit-sure} \iff W^\rp \text{ is limit-sure} \]
    \begin{proof}
    Follows directly from lemma \ref{thm:sg2rmdp}.
    \end{proof}
\end{lemma}

\begin{lemma}[Quantitative Analysis]
\label{lem: gr_qa}
    Let $W' \subseteq \Plays_{\sg'}$ be a parity objective of $\sg'$ and $W^\rp = \Lambda(W')$ is the corresponding objective in $\rsg$. Then, for $k \in [0,1]$
    \[\Val_{\sg'}(W')\geq k \iff \Val_{\rsg}(W^\rp) \geq k \]
    \begin{proof}
    Follows directly from lemma \ref{thm:sg2rmdp}.
    \end{proof}
\end{lemma}

\subsubsection{Size and memory preservation between $\sg'$ and $\rsg$}
\label{app:last-size-memory}

\begin{lemma}[Size and memory preservation under RPOMDP reduction]
\label{lem:rpomdp-size-memory}
Let $\sg'$ be a Pre-min-transformed POSG and $\rsg$ its RPOMDP
reduction. The maps $\Lambda : \Plays_{\sg'} \to \Plays_{\rsg}$,
$\LambdaOpone, \LambdaOptwo$ on observation histories, and
$\Omega_\star, \Xi_\star$
($\star {\in} \{\pone, \ptwo\}$) on strategies satisfy the following.
\begin{enumerate}\itemsep2pt
  \item \textup{(size)} $|S_\rp| = |\sone| + |\stwo|$,
        $|A_\rp| = |\aone| + 1$, the number of vertices in the uncertainity sets is at most
         $\sum_{t\in\stwo}|\atwo(t)|+\sum_{s\in\sone}|\aone(s)|$, and observation alphabets are
        inherited unchanged from $\sg'$. In particular,
        $|\rsg| = \mathcal{O}(|\sg'|)$ and the construction runs in
        linear time.
  \item \textup{(history length)} For every $oh' \in
        \OH_\pone^{\sg'}$, $|\LambdaOpone(oh')| =
        \left\lceil \tfrac{2|oh'|+1}{3} \right \rceil$, and for every $oh^\rp \in
        \OH_\pa^{\rsg}$, $|(\LambdaOpone)^{-1}(oh^\rp)| \le
        |oh^\rp| + \tfrac{|oh^\rp|-1}{2} $. 
        For every $oh' \in \OH_\ptwo^{\sg'}$, $|\LambdaOptwo(oh')| = \frac23 |oh'|$ and for every $oh^\rp \in
        \widehat{\OH}_\pe^{\rsg}$, $|(\LambdaOptwo)^{-1}(oh^\rp)| = \frac32|oh^\rp|$
  \item \textup{(memory)} The lifting maps
        $\Omega_\star, \Xi_\star$ preserve memory class exactly:
        memoryless, finite-memory
        and infinite-memory regimes correspond exactly
        between $\sg'$ and $\rsg$.
\end{enumerate}

\begin{proof}

\textbf{Clause (1).} Direct from the construction. 
The state space deletes the $\ptwo$-states $\stwo$ retaining only $\sone \cup D$, so
$|S_\rp| = |\sone| + |\stwo|$. The action set
deletes the $\ptwo$-actions retaining only $\aone \cup \{\top\}$, so $|A_\rp| = |\aone| + 1$.
The observation alphabets are
$\obsa^\rp = \obsone \cup \{o^\dagger_x \mid x \in \obsone\} = \obsone'$
and
$\obse^\rp = \obstwo \cup \{o^\#_x \mid x \in \obstwo\} = \obstwo'$,
which have the same cardinalities as in $\sg'$. For $s\in\sone$ and $a\in\aone(s)$, the vertex set $V_{s,a}$ is the singleton $\{\delta'(s,a)\}$. For each $(d_s,\anew)$, it consists of $\delta'(s,a)$ for $a\in\atwo(s)$. Thus each source effective transition distribution is copied once, giving the stated sum of local alphabet sizes.

Each component is therefore a constant or unit-additive perturbation
of the corresponding component in $\sg'$, so
$|\rsg| = \mathcal{O}(|\sg'|)$ and $\rsg$ is computed from $\sg'$ in
linear time.

\medskip
\textbf{Clause (2).} Let
$oh' = o_0, z_1, o_1,o_2, z_3, o_3, \cdots, o_{2n-2}, z_{2n-1}, o_{2n-1}, o_{2n} \in \OH^{\sg'}_\pone$ ending in $\pone$ state, so that
$|oh'| = 3n+1$. By construction,
\[
  \Lambda_O^\pone(oh')
  = o_0, z_1, \cdots, o_{2n},
\]
where $o_{2k+1}$ are removed for every 3 states.
Hence
\[
  |\Lambda_O^\pone(oh')| =  2n+1 = \tfrac{2|oh'|+1}{3}.
\]

Let
$oh' = o_0, z_1, o_1,o_2, z_3, o_3, \cdots, o_{2n-2}, z_{2n-1}, o_{2n-1}, o_{2n}, z_{2n+1} \in \OH^{\sg'}_\pone$ ending in dummy state, so that
$|oh'| = 3n+2$. By construction,
\[
  \Lambda_O^\pone(oh')
  = o_0, z_1, \cdots, o_{2n},z_{2n+1}
\]
where $o_{2k+1}$ are removed for every 3 states.
Hence
\[
  |\Lambda_O^\pone(oh')| =  2n+2 = \tfrac{2|oh'|+2}{3}.
\]

From above values, for every $oh' \in         \OH_\pone^{\sg'}$, $|\LambdaOpone(oh')| =        \left\lceil \tfrac{2|oh'|+1}{3} \right \rceil$

 Let
$oh = o_0, z_1, \cdots, o_{2n} \in \OH^{\rsg}_a$ ending in $\sone$ state, so that
$|oh| = 2n+1$. By construction,
\[
  (\Lambda_O^\pone)^{-1}(oh)
    =  o_0, z_1, o_1,o_2, z_3, o_3, \cdots, o_{2n-2}, z_{2n-1}, o_{2n-1}, o_{2n}
\]
where $o_{2k+1}$ are added deterministically from $z_{2k+1}$ for every 2 states. 
 
Hence
\[
    |(\Lambda_O^\pone)^{-1}(oh)| =  3n+1 = \tfrac{3|oh|-1}{2} = |oh|+\tfrac{|oh|-1}{2}.
\]

Let
$oh = o_0, z_1, \cdots, o_{2n},z_{2n+1} \in \OH^{\rsg}_a$ ending in dummy state, so that
$|oh| = 2n+2$. By construction,
\[
  (\Lambda_O^\pone)^{-1}(oh)
  = o_0, z_1, o_1,o_2, z_3, o_3, \cdots, o_{2n-2}, z_{2n-1}, o_{2n-1}, o_{2n}, z_{2n+1}
\]
where $o_{2k+1}$ are added from $z_{2k+1}$ for every 2 states.
Hence
\[
  |(\Lambda_O^\pone)^{-1}(oh)| =  3n+2 = \tfrac{3|oh|}{2} - 1 = |oh| + \tfrac{|oh|-2}{2} < |oh| + \tfrac{|oh|-1}{2}
\]

Thus, for every $oh \in \OH_\pa^{\rsg}$,
$|(\LambdaOpone)^{-1}(oh)|\le|oh|+\tfrac{|oh|-1}{2}$. 

For environment history bound, $\LambdaOptwo$ removes a $\ptwo$-observation history for every 3 consecutive observations while its inverse adds $\ptwo$-observation for every 2 consective observations which creates the required bounds.

\medskip
\textbf{Clause (3).} We prove that $\Omega_\pone$ preserves
memory. The arguments for $\Omega_\ptwo$, $\Xi_\pone$,
$\Xi_\ptwo$ are symmetric: replace $\obsone^\dagger$ by
$\obstwo^\#$ for the $\ptwo$ versions, and swap $\Lambda_O^\star$
with $(\Lambda_O^\star)^{-1}$ for the $\Xi$ versions and replace output of transitions by a prescriptions instead of single distribution.

Let $M^{\sg'}_\pone,M^{\sg'}_\ptwo$ denote the mealy machine for $\pone,\ptwo$-player strategies in $\sg'$. Then, for $M^{\sg'}_\pone = \{Q\cup q_0 \cup Q_d, q_0,\obsone' \cup \obsone' \times \obsone', \tau', \lambda'\}$: $Q\cup q_0 \cup Q_d$ are finite set of memory states where $q_0$ is initial state and $Q_d = \{q_a \mid a \in Q\}$. The transition function $\tau'$ is defined as:
\begin{align*}
    \tau': &\: q_0 \times \obsone \rightarrow Q_d \\
    &\: Q_d \times \Obsnewone \rightarrow Q\\
    &\: Q \times \obsone \times \obsone \rightarrow Q_d
\end{align*}
such that for $q_a \in Q_d,a\in Q,o \in \Obsnewone,\tau'(q_a,o) = a$

The output function $\lambda'$ is defined as:
\begin{align*}
    \lambda': &\: q_0 \times \obsone \rightarrow  \dist(\aone) \\
    &\: Q_d \times \Obsnewone \rightarrow \{\dirac{\top}\}\\
    &\: Q \times \obsone \times \obsone \rightarrow \dist(\aone)
\end{align*}

Let $\sigma' \in \Sigma_{\sg'}$ be a strategy realized by mealy machine $M^{\sg'}_\pone = \{Q'\cup q'_0 \cup Q'_d, q'_0,\obsone' \cup \obsone' \times \obsone', \tau', \lambda'\}$ as defined above. We build a mealy machine $M^{\rsg}_a = (Q\cup q_0,q_0,\obsone',\tau,\lambda)$ representing $\Omega_\pone(\sigma')$ where $Q = (Q'\times \obsone') \cup Q'_d, q_0 = q'_0$. The transition function is given by: $\tau(q_0,o) = \tau'(q_0,o)$, for $q \in  Q_d':$ $\tau(q,o_a) = (\tau'(q,o_a),a)$ and for $q \in Q':$ $\tau((q,o_1),o_2) = \tau'(q,o_1,o_2)$. The output function is given by: $\lambda(q_0,o) = \lambda'(q_0,o)$, for $q \in Q_d:$ $\lambda(q,o_1) = \lambda'(q,o_1)$ and for $(q,o_1) \in (Q'\times \obsone):$ $\lambda((q,o_1),o_2) = \lambda'(q,o_1,o_2)$. Mealy machine $M^{\rsg}_a$ realizes $\Omega_\pone(\sigma')$.

\medskip
\noindent\textit{Special cases.} For a Mealy transducer $M=(Q,q_0,u,g)$, a memoryless strategy corresponds
to $|Q| = 1$, which the construction trivially preserves. An
infinite-memory strategy places no bound on $|Q|$; the construction
uses the same $Q$, so the class is preserved.

\end{proof}

\end{lemma}

\subsection{Proof of Lemma \ref{thm:sg2rmdp-main}}
\label{app:thm3}

\begin{proof}
    Because parity objective subsumes all $\omega$-regular objectives, from lemma \ref{thm:sg2rmdp}, $\Val_{\sg'}(W') = \Val_{\rsg}(W^\rp)$ for any $\omega$-regular objective $W'$ in $\sg'$ and its corresponding objective $W^\rp$ in $\rsg$.\\
    The ``consequently'' statement follows from lemmas \ref{lem:gr_sw}, \ref{lem:gr_as}, \ref{lem:gr_ls}, and \ref{lem: gr_qa}. 
\end{proof}

\clearpage

\end{document}